%% file: main.tex
\documentclass{article}

\PassOptionsToPackage{numbers, sort}{natbib}

\usepackage[preprint]{neurips_2024}

\usepackage[utf8]{inputenc} %
\usepackage{mystyle}

\usepackage[T1]{fontenc}    %
\usepackage{hyperref}       %
\usepackage{url}            %
\usepackage{booktabs}       %
\usepackage{amsfonts}       %
\usepackage{nicefrac}       %
\usepackage{xcolor}         %

\usepackage[ruled,vlined]{algorithm2e}
\usepackage{tabularx}%
\usepackage{tikz}
\usepackage{url}
\usepackage{amsmath,amsthm, amssymb}
\usepackage{cleveref}
\usepackage{enumitem}
\usepackage{color}
\definecolor{green}{rgb}{0.09, 0.45, 0.27}

\def\Oracle{$\chi(\varepsilon)$-Learning Procedure}
\def\oracle{$\chi(\varepsilon)$-learning procedure}
\def\oraclearg[#1]{#1-learning procedure}
\def\Oraclearg[#1]{#1-Learning Procedure}

\SetKwInOut{Input}{Input}
\SetKwInOut{Output}{Output}

\usepackage{pifont}%
\newcommand{\cmark}{\ding{51}}%
\newcommand{\xmark}{\ding{55}}%

\def\sA{{\mathscr{A}}}

\input{macros}

\title{Contractual Reinforcement Learning: \\Pulling Arms with Invisible Hands\thanks{This work is supported in part by  Army Research
Office Award W911NF-23-1-0030, ONR Award N00014-23-1-2802 and NSF Award CCF-2303372. }}

\author{Jibang Wu\thanks{Department of Computer Science, University of Chicago; Corresponding author emails: \texttt{\{wujibang, haifengxu\}@uchicago.edu}.} 
\And Siyu Chen\thanks{Department of Statistics and Data Science, Yale University.}
\And Mengdi Wang\thanks{Department of Electrical and Computer Engineering, Princeton University.}
\And  Huazheng Wang\thanks{School of Electrical Engineering and Computer Science, Oregon State University.}
\And Haifeng Xu\footnotemark[2] }

\begin{document}

\maketitle

\begin{abstract}
The agency problem emerges in today's large scale machine learning tasks, where the learners are unable to direct content creation or enforce data collection. In this work, we propose a theoretical framework for aligning economic interests of different stakeholders in the online learning problems through contract design. 
The problem, termed \emph{contractual reinforcement learning}, naturally arises from the classic model of Markov decision processes, where a learning principal seeks to optimally influence the agent's action policy for their common interests through a set of payment rules contingent on the realization of next state. 
For the planning problem, we design an efficient dynamic programming algorithm to determine the optimal contracts against the far-sighted agent. For the learning problem, we introduce a generic design of no-regret learning algorithms to untangle the challenges from robust design of contracts to the balance of exploration and exploitation, reducing the complexity analysis to the construction of efficient search algorithms.  For several natural classes of problems, we design tailored search algorithms that provably achieve $\tilde{O}(\sqrt{T})$ regret. We also present an algorithm with $\tilde{O}(T^{2/3})$ for the general problem that improves the existing analysis in online contract design with mild technical assumptions.
\end{abstract}

\input{intro}

\input{model}

\input{bandit}

\input{rl}

\section{Conclusion}
In this paper, we propose the study of contractual reinforcement learning problems in which the principal learns to influence the agent's policy by adaptively designing contracts that are contingent on the state realization. The principal must not only balance the tradeoff between her  payments and rewards from the agent's policy, but also incentivize the agent's exploration for her learning in an unknown environment.
Our primary approach is to decouple this general problem into a standard online learning problem and a hyperplane search problem. This enables a clean analysis of the no-regret learning guarantee under several variants of technical assumptions. Meanwhile, several technical gaps remain for future work, including a tighter analysis under relaxed assumptions and the general setup where the agent adaptively improves his policy.
We believe this model forms a natural theoretical basis for the agency problem in today's large scale machine learning tasks where economic incentives of users, creators, service providers stand in conflict with the Internet platform's long-term objective. More generally, it sheds light on the emergent problems of AI alignment from the perspective of steering AI behaviors through reward-shaping in its training environment. We hope this work would motivate new avenues for developing robust, incentive-compatible frameworks that align diverse stakeholder interests in complex digital ecosystems.

\bibliographystyle{plainnat}
\bibliography{main} 

\newpage

\input{appendix}

\end{document}

%% file: macros.tex
\RequirePackage{bm} 

\numberwithin{equation}{section}
\newtheorem{theorem}{Theorem}

\newtheorem*{restatethm}{Theorem}
\newtheorem{corollary}{Corollary}[theorem]
\newtheorem{lemma}{Lemma}

\newtheorem{proposition}{Proposition}

\newtheorem{definition}{Definition}

\newtheorem{assumption}{Assumption}

\newtheorem{example}{Example}

\newtheoremstyle{nonindented}{1ex}{1ex}{}{}{\bfseries}{.}{.5em}{}
\newtheoremstyle{indented}{1ex}{1ex}{\itshape\addtolength{\leftskip}{0.6cm}\addtolength{\rightskip}{0.6cm}}{}{\bfseries}{.}{.5em}{}
\theoremstyle{nonindented}
\theoremstyle{indented}
\theoremstyle{plain}

\newcommand{\abs}[1]{\left| #1 \right|}

\renewcommand{\hat}{\widehat}
\renewcommand{\tilde}{\widetilde}
\renewcommand{\bar}{\overline}

\newcommand{\norm}[1]{\left\lVert#1\right\rVert}
\newcommand{\bignorm}[1]{\bigg|\bigg|#1\bigg|\bigg|}

\DeclareMathOperator{\Reg}{Reg}
\DeclareMathOperator{\Vol}{Vol}

\def\min{\qopname\relax n{min}}
\def\max{\qopname\relax n{max}}

\def\argmin{\qopname\relax n{argmin}}
\def\argmax{\qopname\relax n{argmax}}

\def\minarg{\qopname\relax n{minarg}}
\def\maxarg{\qopname\relax n{maxarg}}

\def\Ex{\qopname\relax n{\mathbf{E}}}

\newcommand{\bx}{\bm{x}}

\newcommand{\cA}{\mathcal{A}}
\newcommand{\cB}{\mathcal{B}}

\newcommand{\cE}{\mathcal{E}}

\newcommand{\cG}{\mathcal{G}}

\newcommand{\cI}{\mathcal{I}}

\newcommand{\cL}{\mathcal{L}}
\newcommand{\cM}{\mathcal{M}}

\newcommand{\cO}{\mathcal{O}}
\newcommand{\cP}{\mathcal{P}}
\newcommand{\cQ}{\mathcal{Q}}

\newcommand{\cS}{{\mathcal{S}}}
\newcommand{\cT}{{\mathcal{T}}}

\newcommand{\cV}{\mathcal{V}}

\newcommand{\cX}{\mathcal{X}}
\newcommand{\cY}{\mathcal{Y}}
\newcommand{\cZ}{\mathcal{Z}}

\newcommand{\BB}{\mathbb{B}}
\newcommand{\EE}{\mathbb{E}}

\newcommand{\PP}{\mathbb{P}}

\newcommand{\RR}{\mathbb{R}}
\newcommand{\SSS}{\mathbb{S}}

\newcommand{\bpi}{\bm{\pi}}

\newcommand{\bSigma}{\bm{\Sigma}}

\newcommand{\one}{{\bf 1}}

\newcommand{\mini}[1]{\mbox{minimize} & {#1} &\\}

\newcommand{\st}{\mbox{subject to} }

\newcommand{\qcon}[2]{&#1, & \mbox{for } #2.  \\}
\newenvironment{lp}{\begin{equation}  \begin{array}{lll}}{\end{array}\end{equation}}
\newenvironment{lp*}{\begin{equation*}  \begin{array}{lll}}{\end{array}\end{equation*}}

%% file: intro.tex
\section{Introduction}\label{sec:intro}
\begin{quote}
    \textit{``Every individual...  intends only his own gain, and is led by an invisible hand to promote an end which was no part of his intention.''}  
    
    \hspace{35mm} --- Adam Smith, \emph{The Theory of Moral Sentiments}, 1759. 
\end{quote}

The ``invisible hand'' metaphor by Adam Smith illustrates how properly designed incentive structures can guide self-interested individuals to inadvertently promote the greater social good. 
This concept is increasingly relevant in the realm of machine learning, as the scale of applications expands and the conflict of economic interests intensifies.
For example, an Internet platform wants to estimate the ad revenues from serving different types of content, but it is up to the creators to decide what content to produce. While the platform seeks high-quality content to  boost its long-term growth, creators may opt to minimize their production costs.
This misalignment has prompted platforms to implement revenue-sharing models, fueling the growth of the \emph{creator economy}, projected to exceed half a trillion by 2027~\cite{bhargava2022creator, creatoreconomy, florida2022rise, goldman}. 
However, current incentive models are inadequate, especially in light of their roles in exacerbating the proliferation of clickbait and misinformation online~\cite{yao2023bad,yao2024rethinking,immorlica2024clickbait}.
Moreover, this issue of misalignment extends well beyond content platforms. E-commerce sites rely on sufficient consumers experimenting with new products for accurate preference assessments. Gig platforms depend on freelance workers accepting tasks to gather essential operational data. Even recommender systems are paying users for their engagement in order to effectively optimize their algorithms~\cite{tiktok}.
In these cases, the learner's hands are tied, and decision-makers interacting with the environment have their own objectives, dooming the system to under-exploration regardless of the learner's objective.
Hence, there is a pressing need to pursue formal treatments of incentive alignment problems between the learners and decision-makers and to design principled learning algorithms with statistical and computational efficiency guarantees.

\noindent\textbf{Contributions. }
On the conceptual side, the presence of self-interested decision-makers challenges our common assumption in online learning, where a single learner controls all the interactions with the environment.
This paper introduces the contractual reinforcement learning (RL) problem in the principal-agent Markov decision process (PAMDP), where we adopt the principal-agent model from contract theory~\cite{grossman1992analysis,dutting2019simple} to capture strategic interactions between the learner and decision-maker.
As illustrated in Figure~\ref{fig:pamdp}, the learner (henceforth, principal/she) collects the rewards from the actions of decision-maker (henceforth, agent/he). Without any incentive design, the agent simply optimizes his policy in a standard Markov decision process (MDP) based on his cost function. However, since the agent's optimal policy is not necessarily in the principal's best interest, the principal is motivated to properly incentivize the agent to act in her favor by designing contracts that specify the payment rules contingent on the realization of the next state.
The core challenge in this design problem is the information asymmetry at two levels: (1) the principal cannot observe the agent's action a priori and has to condition her payment on the probabilistic outcome of the action --- a phenomenon known as the \emph{moral hazard} in economics; (2) the agent is far-sighted that he is willing to take suboptimal actions at one step in order to reach a more favorable state in future steps --- a major barrier for theoretical analysis in multi-agent learning problems.

\vspace{-5pt}
\begin{figure}[tbh]
    \centering
    \includegraphics[width=0.6\linewidth]{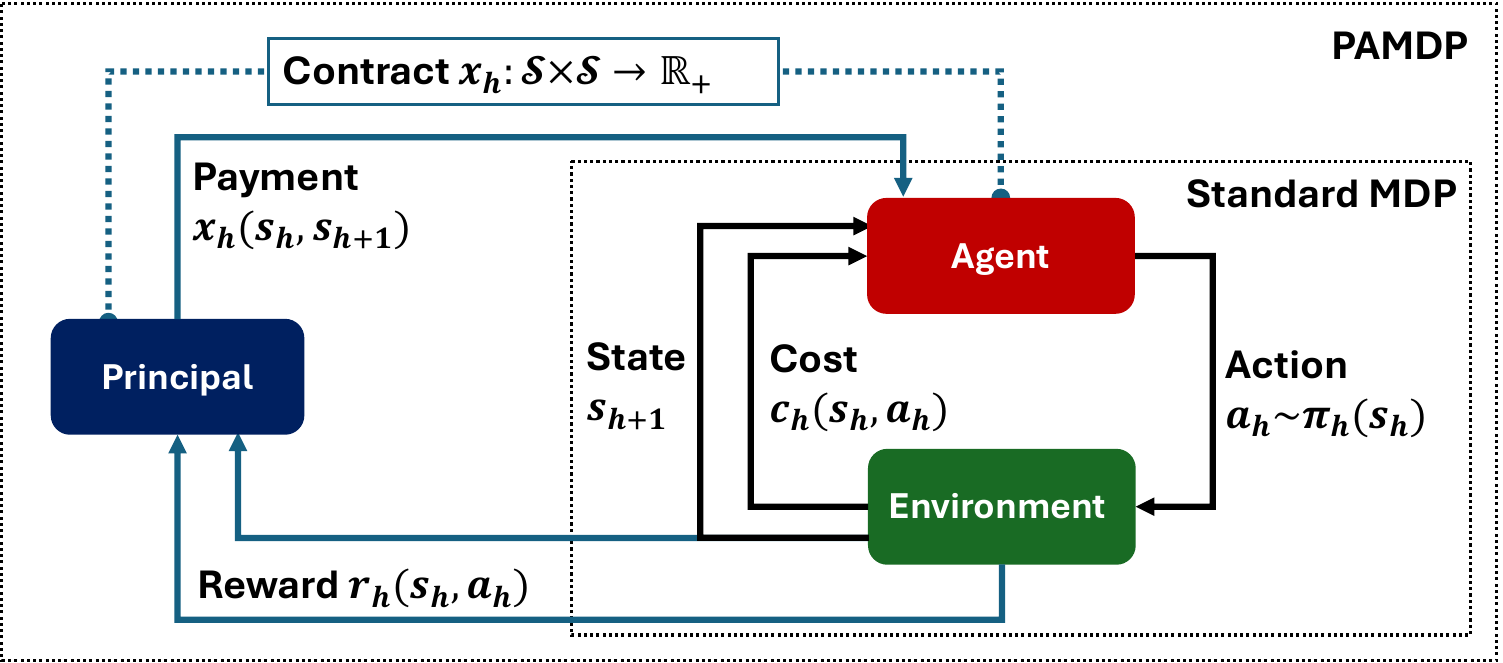}
    \caption{An illustration of contractual RL in the PAMDP.}
    \label{fig:pamdp}
\end{figure}
\vspace{-5pt}

On the technical side, this paper provides a comprehensive solution framework to address the unique learning and computational challenges when  moral hazard meets far-sighted agency in contractual RL problems. 
In Section~\ref{sec:formal-model}, we define state value functions for both the agent and principal, from which we derive a new class of Bellman equations to characterize the intricate correspondence between the principal and agent's optimal policy. This leads to our Theorem~\ref{prop:bellman-opt}, which shows that the principal's optimal planning problem can be solved by a clean formulation of dynamic programming in polynomial time.
The learning problem is more involved, so we begin with the contractual bandit learning problem (episode length $H=1$) in Section~\ref{sec:bandits} to focus on the challenges from moral hazard. In particular, to achieve low regret, the principal's learning algorithm must balance exploration and exploitation while continuously improving its estimation of the agent's preferences to determine cost-efficient contracts. In Theorem~\ref{thm:decoupling}, we construct a generic algorithm that reduces the learning problem into a standard online learning problem and an efficient search problem for the agent's decision boundary. As a consequence, we are able to obtain sublinear regret guarantee under different setups, summarized in Table~\ref{tab:complexity}. The efficient search algorithm we designed for learning the outcome distribution difference in the simplex may be of interest for general use.
With these insights, we delve into the full contractual RL problem in Section~\ref{sec:mdp} and show a provably efficient learning algorithm under several technical assumptions in Theorem~\ref{thm:contractual-rl}. Meanwhile, the general result highlights a trade-off between statistical and computational tractability, leaving an intriguing open question on the existence of the best-of-both-worlds solution. The complexity of search  is in logarithmic order yet with a large constant in the Markovian setup, and we expect an improved analysis by organically combining the search and exploration in the algorithm design.

\begin{table}[tbh]
    \caption{Regrets in Contractual RL, $\tilde{O}$ omits logarithmic terms and other problem-specific constants}
    \label{tab:complexity}
    \centering
    \begin{tabular}{cccl} \toprule
         Moral Hazard & Far-sighted Agency & Known Cost & Regret \\ \midrule
         \xmark & \xmark & \xmark & $\tilde{O}(T^{1/2})$, Corollary~\ref{coro:mab}\\ \midrule
       \cmark & \xmark & \xmark & $\tilde{O}({T}^{2/3})$, Corollary~\ref{coro:general-contractual} \\ \midrule
        \cmark & \xmark & \cmark  & $\tilde{O}({T}^{1/2})$, Corollary~\ref{coro:general-contractual-known-cost} \\ \midrule
      \cmark & \cmark & \cmark & $\tilde{O}({T}^{1/2})$, Theorem~\ref{thm:contractual-rl} \\ \bottomrule
    \end{tabular}
\end{table}

\noindent\textbf{Related Work. }
Our problem is built upon the principal-agent model in contract theory, a crucial branch of economics \cite{grossman1992analysis, smith2004contract, laffont2009theory}. 
Driven by an accelerating trend of contract-based markets deployed to Internet-based applications, the contract design problem recently started to receive a surging interest especially from the computer science community \cite{dutting2019simple, guruganesh2021contracts, alon2021contracts, castiglioni2022designing}. 
The principal-agent model has been also applied for the delegation of online search problems~\cite{bechtel2022delegated, kleinberg2018delegated} and machine learning tasks~\cite{saig2024delegated}.
While these works focus on the computational aspects of contract design, we consider the adaptive design problem of the contract between learners and decision-makers in an initially unknown environment. 
For our learning problem, the dynamic (contextual) pricing problem \cite{kleinberg2003value, mao2018contextual, shah2019semi, lobel2018multidimensional, leme2018contextual} can be viewed as one of its special cases, where the contract is contingent on the agent's binary action and the principal already knows her reward function. As we will see in Section~\ref{sec:bandits}, our algorithm is able to borrow some design insights from these pricing problems.
Meanwhile, the online contract design problem begins as a variant of dynamic pricing~\cite{kleinberg2003value} where the agent's cost is stochastic (or adversarially) chosen, and regret bound is $\Theta(\sqrt{T})$ (or $\Theta(T^{2/3})$ in adversarial setup). 
\citet{ho2016adaptive, zhu2022sample} consider a generalized model where the agent has multiple actions, both the cost and reward of his actions are determined by the agent's Bayesian type that are unknown to the learner. This problem relates to the continuum-armed bandit problem~\cite{agrawal1995continuum}, except the principal's utility is not continuous, and \citet{zhu2022sample} shows an almost tight linear regret bound $\tilde{\Theta}(T^{1-K/|\cS|})$ for some constant $K$ and the number of outcomes $|\cS|$. 
In comparison, our learning problem is closer to the standard contract design model, in which the agent type is observable by the principal (captured by the initial state or context), as many platforms hold a good amount of data on their users and content creators. More importantly, this modeling choice allows us to focus on solving the key challenges of learning and planning the optimal contract under moral hazard, where we are able to achieve $\tilde{O}(\sqrt{T})$ regret for a large class of problems and $\tilde{O}(T^{2/3})$ in general under mild assumptions.  
Lastly, several recent works~\cite{dogan2023repeated, dogan2023estimating, scheid2024incentivized} consider the simple special case of our problem, where there is no Markov state transition and principal can directly incentivize the agent to take certain action without the barrier of moral hazard. 
We defer further discussion of the related work to Appendix~\ref{sec:all-related-work}.

%% file: model.tex
\section{Problem Formulation}
\label{sec:formal-model}

\subsection{The Principal-Agent Markov Decision Process}
Let us first recall the standard reinforcement learning problem in a (finite-horizon) Markov decision process $(\cA, \cS, \{P_h, r_h\}_{h=1}^{H}, P_0)$, where we have the agent's action space $\cA$, the environment's state space $\cS$, the transition kernel $P_h: \cS\times \cA \to \Delta(\cS)$, the expected reward function $r_h: \cS \times \cA \to [0,1]$, the initial state distribution $P_0 \in \Delta(\cS)$ and the horizon length $H$. The contractual reinforcement learning problem simply extends the MDP to a principal-agent Markov decision process $(\cA, \cS, \{P_h, r_h, c_h\}_{h=1}^{H}, P_0)$ with the additional cost function $c_h: \cS \times \cA \to [0,1]$.~\footnote{The $[0,1]$ scale of the cost and reward function range is without loss of generality, due to constant shifting and rescaling, thereby covers existing models~\cite{dogan2023repeated, dogan2023estimating, scheid2024incentivized} that assume a positive reward function for the agent.} In this process, the agent interacts with the environment by taking actions and bearing the costs, whereas the principal receives the reward from the environment. Unable to directly interact with the environment, the principal has to instead design and implement contracts to incentivize the agent to take actions in her interest. Below, we formalize the design of their policies.

Following from a standard MDP, the agent's \emph{action policy} $\bpi = \{ \pi_h: \cS\to \Delta(\cA) \}_{h=1}^{H}$ specifies that at each step $h$, given the state $s$, the agent would take the action $a\sim\pi_h^{\bx}(s)$. 
In the following subsection, we will discuss how the agent chooses his action policy and that it suffices to only consider deterministic action policies.
Meanwhile, the principal's \emph{contract policy} $\bx = \{ x_h: \cS \times \cS \to \RR_{+} \}_{h=1}^{H} $ is a sequence of non-liable payment rules $x_h$, where $x_h(s_h, s_{h+1})$ specifies the payment to the agent if the next state $s_{h+1}$ is realized, given the current state $s_h$ at the $h$-th step. The non-liability constraint ensures that the principal's payment in the contract for any realization of the next state must be non-negative; the problem would otherwise degenerate with an trivially optimal solution for the principal (see e.g., \cite{dutting2019simple}). Denote $\Pi, \cX$ as the agent and principal's policy space, respectively. Let $|\cS|=S, |\cA|=A$ and thus $|\Pi|=(SA)^{H}$.

The typical setting of the PAMDP problems can be summarized by the following steps. 
In the beginning of each episode, the initial state $s_1 \sim P_0$ is realized and observed by both the principal and the agent. Afterwards, the principal commits to a contract policy $\bx$ and the agent accordingly chooses an action policy $\bpi$. Their interactions then proceed as follows at each step $h \in [H]$,
\vspace{-15pt}
\begin{center}
\fcolorbox{black}{gray!20!white}{
\parbox{0.98\linewidth}{
\begin{enumerate}[leftmargin=*]
    \item The agent takes an action $a_h \sim \pi_h(s_h)$ and bears the cost $c_h(s_h, a_h)$.
    \item The next state $s_{h+1} \sim P_h(s_h, a_h)$ is realized and observed by both the principal and agent.
    \item The principal receives a noisy reward $\iota_h( s_{h}, s_{h+1})$ and pays the agent $x_h( s_{h}, s_{h+1} )$.
    \item The principal observes the agent's action $a_h$.
\end{enumerate}
}}
\end{center}
\vspace{-5pt}
In this step, the principal's utility is $\iota_h( s_{h},  s_{h+1}) - x_h(s_h, s_{h+1})$, her reward minus the payment to agent, whereas the agent's utility is $x_h( s_{h}, s_{h+1} ) - c_h(s_h, a_h)$, the payment from principal minus his cost. The reward noise has zero mean such that $ r_h(s, a) = \Ex_{s'\sim P_h(s, a)} \iota_h(s, s'), \forall s\in\cS,a\in\cA$.
We refer the readers to Appendix~\ref{sec:rl-sketches} for a summary of notations and Appendix~\ref{sec:modeling-choice} for a full discussion of our modeling choices.

\subsection{The Optimal Contract Policy} 

Without any contract design, the model reduces to a standard MDP $(\cA, \cS, \{P_h, c_h\}_{h=1}^{H}, P_0)$ for the agent and the principal passively collects the reward from the agent's policy. This outcome could be suboptimal for both the principal and agent. Instead, by reshaping the agent's reward environment through the design of contract policy, the principal could induce the agent adopt some action policy with higher social surplus. This motivates the problem of designing the optimal contract policy. We focus on a realistic yet challenging setup in the face of a long-lived, far-sighted and Bayesian rational agent who is also planning optimally for his cumulative reward --- we expect the case of myopic agents can be worked out with simpler approach. In particular, since the agent's utility is not necessarily $0$ under the principal's optimal contract at any state due to moral hazard, a far-sighted agent could take certain actions that are sub-optimal in the current step, yet secure him toward certain future states where he can obtain higher cumulative utility.

We extend  notions of value functions and optimal policies from MDP to PAMDP. 
Under any action policy $\bpi$ and contract policy $\bx$,
we define the principal's state value function  at the $h$-th step as, 
$$
 V_h^{\bx, \bpi}(s) := \Ex \big[ \sum_{\tau=h}^{H} r_\tau(s_\tau, a_\tau )   - x_\tau(s_\tau, s_{\tau+1}) \big| \{ \pi_\tau \}_{\tau=h}^{H}, s_\tau = s \big],
$$
and the agent's state value function at the $h$-th step as, 
$$
 U_h^{\bx, \bpi}(s) := \Ex \big[ \sum_{\tau=h}^{H} x_\tau(s_\tau, s_{\tau+1}) - c_\tau(s_\tau, a_\tau ) \big| \{ \pi_\tau \}_{\tau=h}^{H}, s_\tau = s \big],
$$
where the expectation in both $V, U$ are with respect to the randomness of the trajectory (due to the stochasticity of state transitions and action policy). 
Let $V^{\bx, \bpi} := \Ex_{s \sim P_0} V_1^{\bx, \bpi}(s)$ and $U^{\bx, \bpi} := \Ex_{s \sim P_0} U_1^{\bx, \bpi}(s)$.
The principal's goal is to maximize her value $V^{\bx, \bpi}$, given the agent's optimal response $\bpi$, which equivalently maximizes $V_1^{\bx, \bpi}(s)$ at any initial state $s$ with $P_0(s)>0$. 
Hence, we define the principal's optimal contract policy $\bx^* = \{ x^*_h \}_{h=1}^{H}$ and the corresponding optimal value function $V^*$ as the optimal solution and value of the following bi-level optimization problem,~\footnote{Throughout this paper, we assume the agent breaks tie in favor of the principal. This is without loss of generality in generic games, since the principal can force the tie-breaking by making an infinitesimally small additional payment to the action of her interest.}
\begin{equation}\label{eq:principal-optimal-contract}
    V^*, \bx^* := \maxarg_{\bx \in \cX} V^{\bx, \bpi^{\bx}}\quad\text{s.t.}\quad \bpi^{\bx} = \argmax_{\bpi \in \Pi} U^{\bx, \bpi},
\end{equation}
where  ``$\maxarg $'' is a convenient operator notation on an optimization problem that returns the optimal objective value followed by its optimal solution. 
For notational convenience, we will denote the agent's optimal action policy in response to contract policy $\bpi$ as $\bpi^{\bx} = \argmax_{\bpi \in \Pi} U^{\bx, \bpi}$, and use shorthands $ V_h^{\bx} := V_h^{\bx, \bpi^{\bx}}, U_h^{\bx} := U_h^{\bx, \bpi^{\bx}} $  for the principal's and agent's value function under contract policy $\bx$ at the $h$-th step given that the agent responds optimally.
Meanwhile, we denote $\bx^{\bpi} = \argmax_{\bx \in \cX} V^{\bx, \bpi} \text{ s.t. } \bpi  = \argmax_{\bpi \in \Pi} U^{\bx, \bpi}$ as the principal's optimal contract policy to induce the agent's action policy $\bpi$. We use similar shorthands $ V_h^{\bpi} := V_h^{\bx^{\bpi}, \bpi}, U_h^{\bx} := U_h^{\bx^{\bpi}, \bpi} $  for the principal's and agent's value function under contract policy $\bx^{\bpi}$ at the $h$-th step given that the agent responds optimally.
Notably, since the optimization problem~\eqref{eq:principal-optimal-contract} hinges on the intricate correspondence between $\bx$ and $\bpi$, it is unclear for now if the principal can efficiently plan his optimal policy adopting the standard approach in MDP.

\noindent\textbf{Solving for the Agent's Optimal Policy. }
One key observation is that the correspondence between $\bpi$ and $\bx$ has a clean characterization through the Bellman equation. 
Specifically, both functions $\{\pi_h^{\bx}, U_h^{\bx}\}_{h=1}^{H}$ can be solved through backward induction with $U^{\bx}_{H+1}(s) = 0$: 
\begin{equation}\label{eq:agent-optimal-bellman}
 \textup{given $U^{\bx}_{h+1}$,}\qquad  U^{\bx}_h(s), \pi^{\bx}_h(s)  = \maxarg_{a \in \cA} P_h(s, a) \cdot [x_h(s) + U^{\bx}_{h+1}] - c_h(s,a).
\end{equation}
Notice that since $\pi^{\bx}_h(s)$ is a maximizer of a linear function, the agent's best responding policy $\pi_h$ is deterministic without loss of generality. 
With $\bpi^{\bx}$, the principal's value function under $\bx$ can also be computed iteratively from $V^{\bx}_{H+1}(s)=0$: 
\begin{equation}\label{eq:principal-expected-value}
 \textup{given $V^{\bx}_{h+1}$,}\qquad   V_h^{\bx}(s) = r_h(s, \pi_h^{\bx}(s) ) + P_h(s,  \pi_h^{\bx}(s) ) \cdot [   V^{\bx}_{h+1} - x_h(s)].
\end{equation}
Due to the space limit, we only solve the agent's best response $\bpi^{\bx}$ at any given $\bx$. We refer the reader to Appendix~\ref{sec:value-iterations} for the more involved formulation to solve the optimal policy $\bx^{\bpi}$ for any given $\bpi$.

\noindent\textbf{Value Decomposition. }
Another key observation is that the value functions can  decomposed into  parts that are only depends on the agent's action policy. 
This is analogous to the standard contract design where principal's and agent's utility sums up to the social surplus, i.e., the difference between the reward and cost of the agent's  action.
Here, let the principal's expected reward and agent's expected cost function in the $h$-th step be 
$$ R_h^{\bpi}(s) := \Ex \big[\sum_{\tau=h}^{H} r_\tau(s_\tau, a_\tau ) \big| \{ \pi_\tau \}_{\tau=h}^{H}, s_\tau = s \big],\  
 C_h^{\bpi}(s) :=  \Ex \big[\sum_{\tau=h}^{H} c_\tau(s_\tau, a_\tau )\big| \{ \pi_\tau \}_{\tau=h}^{H}, s_h = s \big].
$$
By linearity of expectation, for any policy $\bx\in \cX, \bpi\in\Pi$, at any state $s$ of any step $h$, we have
$$
V_h^{\bx, \bpi}(s) = R_h^{\bpi}(s) - C_h^{\bpi}(s) - U_h^{\bx, \bpi}(s). 
$$
Both functions $R, C$ are fixed to the agent's action policy $\bpi$, regardless of the contract policy $\bx$. In addition, $C_h^{\bpi}(s) - U_h^{\bx, \bpi}(s)$ captures the total amount of expected payment transferred from the principal to the agent since $h$-th step at state $s$.
Since the total reward is fixed under any given $\bpi$, the principal's value is maximized under the minimal total payment, $\zeta_h^{\bpi}(s) := C_h^{\bpi}(s) + U_h^{\bpi}(s)$. The function $\zeta_h^{\bpi}$ thus serves as the equivalent optimization objective in the least-payment Bellman equation in Appendix~\ref{sec:value-iterations}.

\noindent\textbf{Solving for the Optimal Contract Policy. }
With the two observations above, it is clear that the principal's the optimal value and policy $V^* = \max_{\bpi\in \Pi} V^{\bpi} $ can be determined by computing $V^{\bpi}$ for every $\bpi$, according to the least-payment bellman equation in Appendix~\ref{sec:value-iterations}. 
However, this maximization problem is still intractable, as there are exponentially many possible $\bpi$. 
Instead, we have to interleave the process of solving for the optimal policy with least payment and maximum reward. This enables the following construction of a bi-level backward induction that iteratively solves for the optimal contract policy $\bx^*$. 

\begin{theorem}[Bellman Equations of PAMDP]\label{prop:bellman-opt}
The optimal contract policy can solved by dynamic programming in polynomial time, from $h=H$ to $1$ with $U^{\bx}_{H+1}(s), V^{\bx}_{H+1}(s) = 0, \forall s\in \cS, a\in \cA$,
\begin{equation}
    \label{eq:principal-optimal-bellman}
\begin{aligned}
    W^*_h(s, a; x) &= P_h(s, a) \cdot [x 
+ U^*_{h+1} ] - c_h(s,a),   \\
    x^{*}_h(s;a) &= \argmin_{x: \cS \to \RR_{+} }  \left\{ P_h(s, a) \cdot x \mid  W^*_h(s, a; x)  \geq W^*_h(s, a'; x), \forall a'\in \cA  \right\},   \\
    Q^*_h(s, a) &=  r_h(s, a) + P_h(s, a) \cdot [ V^*_{h+1} - x^*_h(s; a)  ],  \\
    V_h^*(s),  \pi^*_h(s) &= \maxarg_{a \in \cA} Q^*_h(s, a), \ x^*_h(s) =  x^*_h(s; \pi^*_h(s)), \
    U_h^*(s) = W^*_h(s, \pi^*_h(s); x^*_h(s)  ),
\end{aligned}
\end{equation}
\end{theorem}

To interpret the Bellman equation above, $ x^*_h(s; a)$ denotes the contract with the least payment to induce the agent to take action $a$ at state $s$ in step $h$. Given that $\pi^*_h(s)$ is the best agent action for the principal to induce, the optimal contract at state $s$ in step $h$ can be determined as $x_h(s) = x_h(s; \pi^*_h(s) ) $. $Q_h^*(s, a), W_h^*(s, a;x)$ are respectively the principal's and agent's total expected utility from $h$-th step under policy $\{ x^*_\tau \}_{\tau=h+1}^{H}$ and $\{ \pi^*_\tau \}_{\tau=h+1}^{H}$, which can be viewed as their optimal state-action value function at $h$-th step, serving as the intermediate variable for the computation. See Appendix~\ref{sec:bellman-optimality} for the proof of correctness.

\subsection{The Contractual Reinforcement Learning Problem}
\label{sec:model-contractual-rl}
We now introduce the reinforcement learning problem in PAMDP, where the principal acts as the learner and seeks to adaptively improve its contract policy by interacting with the agent.
Following the online learning convention, we use the expected regret to evaluate the learning performance in $T$ episodes, 
$
\Reg(T) := \sum_{t=1}^{T} V^* - V^{\bx^t},
$
where $\bx^t$ is the principal's contract policy in the $t$-th episode.

This paper makes a few assumptions for the analysis of reinforcement learning problems. First, the far-sighted agent has perfect knowledge of his cost function and the state transition kernel $\{ P_h, c_h \}_{h\in [H]}$ such that he can always chooses the best response. This is realistic because in applications of our interest, agents are the experts (e.g., content creators, freelance workers, ride-sharing drivers) in the fields who has learnt about the environment sufficiently well  whereas principal as the system designer does not know. 
Second, the agent at time $t$ is assumed to best respond to $\bx^t$. This can be equivalently interpreted as the agent at each time $t$ showing up only once. This is motivated by the reality of Internet applications where each individual agent's participation only accounts for a negligible portion of the system's traffic hence has little influence over the entire system's learning policy, so the best response (regardless of the learning policy) is optimal for each individual.
Thirdly, we assume that the design space of contract is restricted to $\{ x_h: \cS\times \cS \to [0, \eta] \}$ at any step $h\in [H]$. This reflects the practical concern of  contract design under randomness: while contract with bounded payment may sacrifice the optimality, it regularizes the variance in the payment transfer and reduces the risk for both the principal and agents. Moreover, as long as the environment parameters have finite precision,  the parameter $\eta$ can be matched to the finite bit complexity of the optimal contract. Though this assumption is without loss of generality from a modeling perspective, we expect future work to develop tighter analysis techniques to relax the dependency on $\eta$. 
For other regularity assumptions necessary to obtain tractable complexity results, we defer to the technical sections.

%% file: bandit.tex
\section{Warm-up: Solving the Contractual Bandit Learning Problem}
\label{sec:bandits}
In this section, we consider an important special case of the contractual reinforcement problem with $H=1$, which allows us to first focus on the learning challenge from moral hazard without the concern of far-sight agency.
We refer to this problem as the contractual bandit learning problem. Below, we first describe the contractual bandit learning problem with much simplified notations, since it suffices to omit the current state and the time step in the subscripts given that $H=1$.
We then showcase a generic analysis of the statistical complexity of contractual bandit learning problem.

\subsection{The Contractual Bandit Learning Problem}
In this setup, the agent's policy space is simply its action space $\cA$, i.e., the set of bandit arms. $P: \cA \to \Delta(\cS)$ specifies an outcome distribution for each action, where the outcome space $\cS$ could naturally capture the reward stochasticity of each arm in bandit learning problems.
The principal designs the contract $x: \cS \to \RR_{+}$, contingent on the outcome space $\cS$, to influence the agent's choice of action. The principal's reward and agent's cost are both function of the agent's action, $r, c: \cA \to [0,1]$.
We consider a contractual bandit learning problem with $T$ rounds.
In the beginning of each round $t$, the principal commits to a contract $x_t$ and interacts with the agent as follows:
\vspace{-15pt}
\begin{center}
\fcolorbox{black}{gray!20!white}{
\parbox{0.98\linewidth}{
\begin{enumerate}[leftmargin=*]
    \item The agent takes the action $a_t$. 
    \item The outcome $s_t \sim P(a_t)$ is realized and observed by both the principal and agent.
    \item The principal receives the noisy reward $\iota(s_t)$ and pays the agent $x_t(s_t)$. 
    \item The principal observes the agent's action $a_t$.  
\end{enumerate}
}}
\end{center}
\vspace{-5pt}
Here, the noisy reward function satisfies $\Ex_{s \sim P(a_t)} \iota(s)=  r(a_t)$, and we assume the agent's action always maximizes his expected utility, i.e., $a_t \in \argmax_{a\in \cA} \{ P(a)\cdot x_t - c(a) \}$.
To determine the principal's optimal contract, let us recall the notion of \emph{least payment} function from the general setup. We similarly define $\zeta: \cA \to \RR_{+}$ such that for any given action $a\in \cA$, it outputs the least amount of the expected payment necessary to induce $a$,
$\zeta(a)  := \min_{x \in \cX^a}   P(a) \cdot x,$
where $\cX^a = \{ x\in \cX : [P(a) - P(a')] \cdot x  \geq c(a) - c(a'), \forall a' \neq a \}$ denotes the set of all contracts under which the agent would respond with action $a$. Hence, the principal can determine the optimal action to induce, $a^* = \max_{a\in \cA}  \sum_{t=1}^{T} [ r_t(a) - \zeta(a) ] $ with the optimal contract $x^* = \argmin_{x \in \cX^{a^*}}   P(a^*) \cdot x$. 
With the benchmark of the optimal contract $x^*$ that induces $a^*$ with the least payment $\zeta(a^*)$, we can measure the learning performance in $T$ rounds with the expected regret as follows,
$ \Reg(T) 
 = \max_{a^*\in \cA}  \sum_{t=1}^{T} \big[ r(a^*) - \zeta(a^*) \big] - \sum_{t=1}^{T} \big[ r(a_t) - P(a_t) \cdot x_t \big].$
This problem is a strict generalization of standard online learning, as it degenerates to the standard notion of regret when $\zeta(a)=0, \forall a\in \cA$. 
However, with the additional $\zeta$ function, the no-regret learner must not only obtain good estimation of both $r$ and $\zeta$ towards the optimal action, but also implement the contracts that induce the optimal action and have expected payment approaching towards $\zeta$.

\noindent\textbf{A Simpler Case with Direct Incentives. } 
We remark that a special case of the contractual bandit learning problem assumes the principal is able to design her contract contingent on the agent's action. This enables the principal to implement any payment rule $x: \cA \to \RR_{+}$, and the agent responds with his optimal action $a^* = \argmax_{a\in \cA} x(a) - c(a)  $.
With this relaxation, $\zeta = c$, since the optimal $x$ to induce any action $a$ is to set a direct incentive with $x(a)=c(a), x(a')=0, \forall a'\neq a$. The expected regret reduces to
$ \Reg(T) = \max_{a^*\in \cA} \sum_{t=1}^{T} \big[ r(a^*) - c(a^*) \big] - \sum_{t=1}^{T} \big[ r(a_t) - x_t(a_t) \big].$
As we will see in this paper, the learning problem becomes more tractable in this setup, since the principal can directly learn the cost function $c$ to determine the least payment to induce each action. 
In Appendix~\ref{sec:mab-proof} and~\ref{sec:linear-bandit-proof}  we showcases the multi-armed bandits and linear bandits under direct incentives, both of which have been recently studied by \citet{scheid2024incentivized}.

\subsection{A Generic Approach to Contractual Bandit Learning}
\label{sec:generic}

We begin with a natural assumption that enable us to simply employ existing techniques in online learning to obtain tractable complexity results for a large class of contractual bandit learning problems.

\begin{assumption}[$\lambda$-Inducibility] \label{assum:inducibility} 
For any action $a \in \cA$, there exists an event $e \in \{0,1\}^{S}$ as a distribution of outcomes such that $ [ P(a) - P(a') ] \cdot e \geq \lambda, \forall a' \neq a$.
\end{assumption}

This assumption ensures the regularity of the problem instance in the  sense that each action is dominantly capable of inducing a set of outcomes over others such that for any cost function $c$ and any action $a$, there exists a contract $x$ to induce $a$. To see this, one can explicitly construct such contract as $x = e \max_{a'} \frac{c(a)-c(a')}{\lambda}$, where $e$ is the event such that $ [ P(a) - P(a') ] \cdot e \geq \lambda, \forall a' \neq a$.  Otherwise, if $\lambda \leq 0$, then there could be some action that is never the agent's best response under any contract.

We now propose a generic approach to design statistically efficient algorithm for contractual bandit learning problem. 
The key idea of our approach is to decouple the learning of the contract from the learning of the optimal action. In particular, let us first assume an oracle in Definition~\ref{def:contract-oracle} that is able to construct a robust contract set for each action $a \in \cA$, despite the uncertainty in parameter estimation. We use the robust contract set to determine the optimistic action and eventually learn the optimal action with no regret. This enables us to decouple the sample complexity result  into the estimation errors from optimal contract and the optimal action, according to Theorem \ref{thm:decoupling}.

\begin{definition}[$\varepsilon$-margin Contract Set]\label{def:contract-oracle}
    We define the $\varepsilon$-margin contract set for each action $a\in \cA$ as 
    $$ {\cX}^{a}(\varepsilon) = \{x\in \cX: [P(a)-P(a')]\cdot x \geq c(a)-c(a') + \varepsilon, \forall a \neq a'\}. $$
\end{definition}

\begin{theorem}
\label{thm:decoupling}
Under Assumption \ref{assum:inducibility}, with a $\varepsilon$-margin contract set for every action $a\in \cA$, there is a generic algorithm with regret $\tilde{O}(\eta\sqrt{T}+ T\varepsilon/\lambda)$ for the contractual bandit learning problems. 
\end{theorem} 

The key step of the proof is Lemma~\ref{lm:robust-contract-loss}, which shows the contracts solved from LP \eqref{eq:robust-contract-empirical} have bounded suboptimality from the least payment contract (both in estimation and in execution) depending on parameter estimation error $\epsilon$ and the robustness margin $\varepsilon$. This allows us to simply adopt an upper confidence bound argument to bound the regret. See Appendix~\ref{sec:generic-algo-proof} for the full proof and the construction of the generic algorithm. 
The rationale behind Theorem~\ref{thm:decoupling} is to separate the learning of the contract sets from the learning of the optimal action. In particular, the learning and construction procedure of such contract sets has been a well-established problem in variants of Stackelberg games~\cite{letchford2009learning, peng2019learning}. We abstract this problem into the design of a $\chi(\varepsilon)$-learning procedure defined below.

\begin{definition}[$\chi(\varepsilon)$-Learning Procedure] \label{def:chi-learning-procedure}
    For a $\chi(\varepsilon)$-learning procedure, after any $\chi(\varepsilon)$ number of rounds, it can construct a robust contract set $\hat{\cX}^{a}, \forall a \in \cA$ such that ${\cX}^{a}(\varepsilon) \subseteq \hat{{\cX}}^{a}  \subseteq {\cX}^{a}$.
\end{definition}

Based on the concept in Definition \ref{def:chi-learning-procedure}, an immediate implication of Theorem~\ref{thm:decoupling} is that if there is an ${O}(1/\varepsilon)$-learning procedure, a simple ``prepare-then-commit'' style algorithm can achieve $\tilde{O}(\sqrt{T})$ regret in the contractual bandit problem. %
That is, it first \emph{prepares} for a warm start by running the learning procedure for $T^{1/2}$ rounds to obtain the $O(T^{-1/2})$-margin contract sets, then \emph{commits} to follow Algorithm~\ref{algo:black-box-ucb} for the remaining $T-T^{1/2}$ rounds. 
Futhermore, using the standard doubling trick~\cite{besson2018doubling}, we can convert ``prepare-then-commit'' style algorithm into an anytime algorithm with the same $\tilde{O}(\sqrt{T})$ regret guaruntee that is agnostic to the time horizon $T$ during its construction.
Therefore, the difficulty of solving the contractual bandit learning problem hinges on the statistical efficiency of the learning procedure, which heavily depends on the problem structure.

\noindent\textbf{Solving Bandit Problems under Direct Incentives. }
As a direct application of Theorem~\ref{thm:decoupling}, we show that the \oraclearg[$O(1/\varepsilon)$] can be constructed for the two bandit problems under direct incentives and thus admits $O(\sqrt{T})$ regret online learning algorithm. 
The construction of the efficient search algorithm essentially relies on the binary search for the cost of each arm. 
In addition, the binary search algorithm can be generalized to cases with infinitely many arms. Such problem is known as the contextual search, and recent work~\cite{liu2021optimal} have established clean solutions with nearly optimal performance. We defer their detailed construction and proofs to Appendix~\ref{sec:mab-proof} and~\ref{sec:linear-bandit-proof}.

\begin{corollary}\label{coro:mab}
Multi-armed bandits and linear bandits under direct incentives have $\tilde{\Theta}(\sqrt{T})$ regret.
\end{corollary}

\noindent\textbf{Solving Contractual Bandit Problems under Moral Hazard. }
The construction of efficient learning procedure is difficult in general contractual bandit learning.
We instead start with sufficient knowledge of $P$ to construct an \oraclearg[$O(1/\varepsilon)$] under the following assumption. This assumption is motivated by the practice, where the principal would ask the agent to provide a listing of desired conditions for him to perform different level of services. The search problem is otherwise known to have exponential sample complexity lower bound in Stackelberg games~\cite{peng2019learning}.

\begin{assumption}[Preliminary Contracts] \label{assum:prior-contract}
For any $a\in \cA$, the principal has the preliminary knowledge to construct an non-liable contract $x$ that induces the agent's action $a$ with constant payment.  
\end{assumption}

We defer the construction of this learning procedure and its proof to   Appendix~\ref{sec:generic-cost-search}.
As a result, we can construct an explore-then-commit style algorithm $O(T^{2/3})$ regret for general contractual bandit learning, as. Specifically, this algorithm induces the agent to take each action uniformly random for $T^{2/3}$ rounds under the Assumption \ref{assum:prior-contract}. Then, given that the outcome distribution is estimated with error up to $T^{-1/3}$, it can efficiently estimate the difference of cost up to error $T^{-1/3}$ and thus construct an $T^{-1/3}$-optimal contract to induce the optimal action $a^*$ in the remaining rounds.

\begin{corollary}\label{coro:general-contractual}
Under Assumption~\ref{assum:inducibility} and \ref{assum:prior-contract}, $\tilde{O}(T^{2/3})$ regret can be achieved for contractual bandit learning problems.    
\end{corollary}

This result reveals the core challenge of learning the optimal contract under  moral hazard. That is, constructing the contract to induce the optimal action, $[P(a^*) - P(a')] \cdot x  \geq c(a) - c(a'), \forall a' \neq a^*$ already requires a sufficiently good estimate of $P$ for all actions (including the suboptimal ones). 
This observation raises the question on whether it is possible to learn $P(a')$ without playing the costly sub-optimal action $a'$ --- the barrier to achieve $o(T^{2/3})$ regret. The answer turns out to be ``Yes'' but with some catches.
The solution is to implement a binary search procedure for contract $x$ near the hyperplane formed by the linear system $[P(a) - P(a')] \cdot x  = c(a) - c(a'), \forall a' \neq a$. We want to solve the parameters $c(a) - c(a')$ and $ P(a) - P(a'), \forall a'\neq a$ in the linear system with bounded errors using a number of contracts $x$ that almost satisfy the linear system.
This is however impossible unless knowing at least one set of parameters in the linear system to ensure it has full rank. 

\begin{corollary}\label{coro:general-contractual-known-cost}
Under Assumption~\ref{assum:inducibility} and with the knowledge of agent's cost, $\tilde{O}(T^{1/2})$ regret can be achieved for contractual bandit learning problems.    
\end{corollary}

In Appendix~\ref{sec:search-distribution-diff}, we formally show that, knowing the agent's cost, there is an efficient learning procedure for the unknown parameters $ P(a) - P(a'), \forall a'\neq a$ with small errors under mild assumptions. This allows us to attain $\tilde{O}(\sqrt{T})$ for the general contractual bandit problem, and we showcase its application in designing contractual RL algorithms in the next section.
Since the design and analysis of the learning procedure is highly technical, we also demonstrate the high-level idea on a simplified instance in Example~\ref{ex:motivation for hyperplane search} of Appendix~\ref{sec:search-distribution-diff}. More generally, we expect similar learning procedure exists if we alternatively assume some predictive state $s$ in $P$ such that the principal knows $P(s_0 | a), \forall a\in\cA$, since it would also eliminate one extra degree of freedom in the linear system above.

%% file: rl.tex
\section{The Complexity of Contractual Reinforcement Learning}
\label{sec:mdp}

If we treat each stationary policy in contractual RL as an arm and its induced visitation measure (see its formal definition in Appendix~\ref{sec:basic-rl-proofs}) as an outcome in the contractual bandit problem, the generic algorithm from Section~\ref{sec:generic} already provides a $\tilde{O}(T^{2/3})$ regret bound. However, the computational and statistical complexity of both Algorithm~\ref{algo:black-box-ucb} and \ref{algo:bandit-search} has polynomial dependence on the size of action space, which has become exponential as $|\Pi|= (SA)^H$. Moreover, as pointed out above, it requires a uniformly good knowledge over the transition kernel $P$ to constructing the near-optimal contract policy under the moral hazard.
In this section, we provide an improved analysis for the complexity of contractual reinforcement learning, given that the agent's cost function $\{ c_h \}_{h=1}^{H}$ is known initially. This assumption allows us to leverage the learning procedure designed in the last section to efficiently learn the parameters $\mu_h(s, a, a') := P_h(s, a)- P_h(s, a')$ for all $h\in [H], s\in \cS, a,a'\in \cA$.

\begin{algorithm}[tbh]
    \caption{Contractual RL with Warm Start }
    \label{algo:ucb-mdp}
        \KwIn{State, action set $\cS, \cA$, number of steps $H$, episodes $T$, solver $\sA$. }
            Run \oracle{} in Algorithm \ref{algo:farsighted mdp learning oracle} for $T_1$ rounds and obtain estimates $\{ \hat{\mu}_h\}_{h\in [H]}$.\\
            Initialize empirical estimate of parameters $\{ \hat{P}^1_h, \hat{r}^1_h, b^1_h\}_{h\in [H]}$ \\
         \For{$t = 1 \dots T-T_1$}{    
                Solve $\bx^t, \bpi^t$ from the subroutine $\sA$ using the parameters $\{ \hat{P}^t_h, \hat{r}^t_h, b^t_h, \hat{\mu}_h\}_{h\in [H]}$. \\
                Execute the policy $\bx^t$ and observe the trajectory $\{ (s_h^t, a_h^t, r_h^t) \}_{h\in [H]} $. \\
                Update the empirical estimate of parameters $\{ \hat{P}^t_h, \hat{r}^t_h, b^t_h, \hat{\mu}_h\}_{h\in [H]}$
            }
\end{algorithm}

We sketch the no-regret learning algorithm in contractual RL in Algorithm \ref{algo:ucb-mdp}, which cuts the number of episodes $T$ into two phases and can be improved to be agnostic to $T$ with the doubling trick. It begins by running the \oracle{} to efficiently obtain the estimated parameter $\hat{\mu}$ for the construction of robust contract policy.
Then, the algorithm use a solver to determine the robust contract policy $\bx^t$ that induces an optimistic action policy $\bpi^t$ with almost optimal payment. In Theorem~\ref{thm:contractual-rl}, we state the complexity results under two different solvers that work under different technical assumption and provides different trade-offs in statistical and computational complexity. Here, $\kappa,\lambda_s$ in the regret bound are constants in the regularity assumptions, and omit $\log T$ terms from learning $\mu_h$, though the effect of these constants can be canceled out only for sufficiently large $T$; we defer the details to Appendix~\ref{sec:rl-proofs}. Below we zoom into the construction of each component.

\begin{theorem}\label{thm:contractual-rl} With high probability, Algorithm~\ref{algo:ucb-mdp} has $\tilde{O}\left((SA^{-1/2} +  \kappa^{-1/2} )H^2 \sqrt{T} \right)$ regret using the solver in Algorithm~\ref{algo:lp-solver} and $\tilde{O}\left((H^2SA^{-1/2} +  \eta \lambda^{1-H}_s \kappa^{-1/2} ) \sqrt{T} \right)$ regret using the solver in Algorithm~\ref{algo:value-iteration} in contractual RL under mild assumptions.
\end{theorem}

\noindent\textbf{\Oracle{} in Contractual RL. } 
One challenge in the construction is the need to separate the stepwise interference among $\{x_h\}_{h=1}^{H}$. Otherwise, the actual response space for the agent is $(SA)^H$, which is unacceptable even for doing binary search. 
Our solution is due to the observation that if we fix $x_{h+1}, \dots, x_{H}$ and tune $x_h$ only, the agent's expected profits $U_{h+1}^{\bx}, \dots, U_{H}^{\bx}$ remain unchanged. This allows us to set $x_h$ without influencing the agent's action policy for step $h+1, \dots, H$.
Another key challenge in constructing the oracle in the MDP setting is to guarantee visitation measure over each state at step $h$. To maximize the visitation measure of a particular state $s$ at step $h$, we let $x_h(s, \cdot)$ have nonzero values such that the agent has a strong incentive to maximize her visitation measure over state $s$ at step $h$. To simplify our analysis, we assume that the maximal visitation measure at each state $s\in\cS$ and at each step $h\in[H]$ is bounded below, though we expect  it to be relaxed via a more careful analysis since those states rarely visited contributes little to the estimation of the cumulative utility. 
Lastly, the task of setting $x_h(s, \cdot)$ is solved in the bandit learning setup under the techniques and assumptions specified in Appendix~\ref{sec:search-distribution-diff}. 
See \Cref{sec:learning oracle for farsighted MDP} for the formal proof and detailed construction of the learning procedure. 

\noindent\textbf{Solving for Optimistic and Robust Contract Policies. } 
We show two different solvers for the optimistic contract with bounded suboptimality using the estimated parameters. Their basic idea is the same, which is to include additional bonus for optimism and margin for robustness.
However, it turns out that they can either ensure statistical or computational efficiency, leaving an intriguing open question on the existence of the best-of-both-world solver.
For the solver in Algorithm~\ref{algo:lp-solver}, we directly solve for the optimal contract policy according to LP~\eqref{eq:principal-optimal-contract} with additional bonus and margin step for the entire policy. 
For the solver in Algorithm~\ref{algo:value-iteration}, we employ the value iteration from the Bellman equation~\eqref{eq:principal-optimal-bellman} with bonus and margin set at every step. Both solvers require the inducibility assumption similar to Assumption~\ref{assum:inducibility} in contractual bandit learning problem. However, the computationally efficient solver requires the inducibility assumption to hold at every step, whereas the statistically efficient solver only requires the inducibility assumption to hold at the trajectory level. We defer their detailed construction and their proofs to Appendix~\ref{sec:solver-LP} and~\ref{sec:solver-VI}.

%% file: appendix.tex
\appendix

\renewcommand*\contentsname{Table of Contents}

\tableofcontents

\newpage 
\section{Further Discussion on Related Work}
\label{sec:all-related-work}
\input{related-work}

\newpage 
\section{Omitted Content in Section~\ref{sec:formal-model}}

\subsection{Notations and Illustrations}

We use the notation of $[n]$ for the set $\{1, 2, \dots, n\}$. We use $\Delta(\cS)$ to denote the simplex space on discrete set $\cS$.
For probability distribution $P\in \Delta(\cS)$, we will use $P(s)$ to denote the measure of $s\in \cS$ in $P$.
We use the notation of $\maxarg, \minarg$ as an operator on an optimization problem that returns the optimal objective value followed by its optimal solution, e.g., $0, a =\minarg \norm{x-a}^2 $.

We will interchangeably treat a function $f: \cX \to \cY$ as a vector from $\cY^{|\cX|}$. As such, we denote the inner product $f\cdot g := \sum_{x\in \cX} f(x)g(x)$ for $f,g: \cX \to \cY$ or  $\inp{f}{g}_{\cX\times \cY} := \sum_{x\in \cX, y\in\cY} f(x,y)g(x,y)$ for $f,g: \cX\times \cY \to \RR$. Denote their outer product as $f\otimes g$.
Denote $\norm{f}_{\ell, \infty} := \sup_{x\in \cX}\norm{f(x)}_{\ell} $ for $f: \cX \to \cY$. 
In addition, for function $f: \cX \times \cY \to \cZ$, we use $f(x) \in \cZ^{\cY}$. 
For conditional probability $P: \cX \to \Delta(\cY)$, we denote $P(y|x)$ as the measure of $y\in \cY$ given $x\in \cX$.

\label{sec:rl-sketches}
\input{notation-table}

\input{full-paper/rl-plot}

\subsection{Discussion on the Modeling Choices.}
\label{sec:modeling-choice}
\input{modeling-choice}

\subsection{Least-Payment Bellman Equations in PAMDP}
\label{sec:value-iterations}

\input{value-iterations}

\subsection{Bellman Optimality Equations in PAMDP}
\label{sec:bellman-optimality}
\input{proofs/bellman-optimality}

\newpage 

\section{Proofs in Section~\ref{sec:bandits}}
\subsection{The Regret Analysis of the Generic Algorithm}
\label{sec:generic-algo-proof}
\input{proofs/generic-algo}

\subsection{Solving Multi-armed Bandits under Direct Incentives}
\label{sec:mab-proof}
\input{proofs/mab-search}

\subsection{Solving  Linear Bandits under Direct Incentives}
\label{sec:linear-bandit-proof}
\input{proofs/lin-bandit-search}

\subsection{Solving General Contractual Bandit Problems}
\label{sec:generic-cost-search}
\input{proofs/known-distribution-search}

\newpage 

\input{proofs/hyperplane-search}

\newpage 

\section{Proofs in Section~\ref{sec:mdp}}
\label{sec:rl-proofs}

\subsection{Preliminaries for Regret Analysis}
\label{sec:basic-rl-proofs}

\input{proofs/mdp-proofs}

\subsection{\Oracle{} in Contractual Reinforcement Learning} 
\label{sec:learning oracle for farsighted MDP}

\input{proofs/mdp-search}

\subsection{Proofs for the Statistically Efficient Solver} 
\label{sec:solver-LP}
\input{proofs/solver-LP}

\subsection{Proofs for the Computationally Efficient Solver} 
\label{sec:solver-VI}
\input{proofs/solver-VI}

%% file: related-work.tex
\paragraph{Contract Design.} The contract theory has been a crucial branch of economics \cite{grossman1992analysis, smith2004contract, laffont2009theory}. 
Driven by an accelerating trend of contract-based markets deployed to Internet-based applications, the contract design problem recently started to receive a surging interest especially from the computer science community \cite{dutting2019simple, guruganesh2021contracts, alon2021contracts, castiglioni2022designing}. 
The principal-agent model has been also applied for the delegation of online search problems~\cite{bechtel2022delegated, kleinberg2018delegated} and machine learning tasks~\cite{saig2024delegated}.
While these works focus on the computational aspects of contract design, our work is to adaptively design the optimal contract between learners and decision makers in an initially unknown environment.

\paragraph{Dynamic Pricing.}
Our model is related to the dynamic (contextual) pricing problems \cite{kleinberg2003value, mao2018contextual, shah2019semi, lobel2018multidimensional, leme2018contextual}, where a seller learns to post a price on a single item for a sequence of buyers with a fixed cost (possibly under different context). In particular, they can be viewed as special cases of contractual reinforcement learning, where the contract is contingent on the agent's binary action and the principal already knows her reward function. As we will see in Section~\ref{sec:bandits}, our algorithm is able to borrow some design insights from these pricing problems. Nonetheless, our learning algorithm deals with the more involved situations, where the agent has multiple actions (e.g., a list of items to buy) of which the principal's rewards are unknown, and the contract is not necessarily contingent on the agent's actions but their outcomes. 
As such, it is possible to achieve constant regret in these pricing problems, whereas the regret lower bound of contractual reinforcement learning is $\Omega(\sqrt{T})$.  

\paragraph{Online Contract Design.}
The problem begins as a variant of dynamic pricing in \citet{kleinberg2003value} where the agent's cost is stochastic (or adversarially) chosen, and regret bound is $\Theta(\sqrt{T})$ (or $\Theta(T^{2/3})$ in adversarial setup). 
\citet{ho2016adaptive, zhu2022sample} consider a generalized model where the agent has multiple (instead of binary) actions, both the cost and reward of his actions are determined by the agent's Bayesian type that are unknown to the learner. These problems can be viewed as a continuum-armed bandit problem~\cite{agrawal1995continuum}, except the principal's utility is not continuous. \citet{zhu2022sample} shows an almost tight linear regret bound of this problem $\tilde{\Theta}(T^{1-K/|\cS|})$ for some constant $K$ and the number of outcomes $|\cS|$. On top of this model, \citet{zhu2023online} considers the joint online optimization problem of contract and recommendation policy in the context of creator economy.
\citet{zuo2024new} assumes a smoothness condition and presents a direct reduction to the standard Lipschitz bandits problem.
In comparison, our learning problem is closer to the standard contract design model, in which the agent type is observable by the principal (captured by the initial state or context), as many platforms hold a good amount of data on their users and content creators. More importantly, this modeling choice allows us to focus on solving the key challenges of learning and planning the optimal contract under moral hazard, where we are able to achieve $\tilde{O}(\sqrt{T})$ regret for a large class of problems and $\tilde{O}(T^{2/3})$ in general under mild assumptions.  
Meanwhile, several recent works~\cite{dogan2023repeated, dogan2023estimating, scheid2024incentivized} consider the simple special case of our problem, where there is no Markov state transition and principal can directly incentivize the agent to take certain action without the barrier of moral hazard.

\paragraph{Online Learning with Incentive Constraints.}
The incentive design problems have been studied in online learning in several different ways. 
One line of works, known as the incentivized exploration \cite{frazier2014incentivizing, mansour2015bayesian}, consider the situations where the principal recommends the agents to pull different arms and the recommendation policy must be incentive compatible to the agents in a Bayesian sense w.r.t. each agent's prior of arm rewards. \citet{bahar2020fiduciary} consider the fiduciary bandits problem, where a slightly stronger constraint of individual rationality is introduced. Our model is different from these works in that the principal use monetary incentives (contracts) instead of information advantage to influence the agents' decisions. 
Another line of work, known as the budgeted bandits \cite{tran2010epsilon, tran2012knapsack, xia2015thompson, xia2016budgeted}, and more generally, bandits with knapsacks \cite{badanidiyuru2018bandits, agrawal2014bandits, agrawal2016linear,immorlica2022adversarial}, models the intrinsic cost of arm selection. The cost only affects the learner's choices due to the limited budget, whereas the learner (principal) in our multi-agent decision making process needs to properly reimburse the agent's (opportunity) cost in order to influence the agent's arm choices. 
\citet{ratliff2018incentives} consider the multi-armed bandit problem where the reward distribution (impacted by user types) shifts according to the history of arm selection. 
\citet{braverman2019multi} models each bandit arm as a self-interested agent that keeps part of the reward from the principal to strategically maximizes his long-term utility.
Besides the online contract design problem, there are also rich line of literature in the online learning problems under Stackelberg games, information design and auction design setups~\cite{blum2004online, balcan2015commitment, dudik2020oracle, wu2022sequential, zhao2023online, bernasconi2023optimal, cacciamani2023online}.

%% file: notation-table.tex
\begin{table}[tbh]
\caption{A table of notations in the contractual reinforcement learning problem}
\centering
\begin{tabular}{l|l}
\toprule  
Symbols & Interpretations \\
\midrule
  $\cS, \cA$ & state, action space  \\
  $P_h:  \cS \times \cA \to \Delta(\cS), P_0\in \Delta(\cS) $ & transition kernel, initial state distribution \\
  $\iota_h: \cS \times \cS \to \RR_+$ & noisy reward function at $h$-th step \\
  $r_h: \cS \times \cA \to [0,1]$ & expected reward function at $h$-th step \\
  $c_h: \cS \times \cA \to [0,1]$ & cost function at $h$-th step \\
$\bx = \{ x_h: \cS \times \cS \to \RR_+ \}_{h=1}^{H}$ & contract policy \\
$\bpi = \{ \pi_h: \cS \to \Delta(\cA)\}_{h=1}^{H}$ & action policy \\
  $\Pi$ & action policy space  \\
  $\cX$ & contract policy space \\
  $V_h^{\bx, \bpi}, V_h^{\bx}, V_h^{\bpi}, V_h^{*}: \cS \to \RR_+$ & principal's state value function at $h$-th step \\
  $U_h^{\bx, \bpi}, U_h^{\bx}, U_h^{\bpi}, V_h^{*}: \cS  \to \RR_+$ & agent's state value function at $h$-th step \\
  $\rho_h^{\bpi}: \Delta(\cS)$ & state visitation measure at $h$-th step \\
  $\zeta_h^{\bpi}: \cA \to \RR_+$ & least payment function at $h$-th step \\
  $Q_h^{\bpi}: \cS \times \cA \to \RR_+$ & principal's state-action value function at $h$-th step \\
  $W_h^{\bpi}: \cS \times \cA \to \RR_+$ & agent's state-action value function at $h$-th step \\
\bottomrule
\end{tabular}
\end{table}

\medskip

%% file: full-paper/rl-plot.tex
\begin{figure}[tbh]
\centering

\tikzset{every picture/.style={line width=0.75pt}} %

\begin{tikzpicture}[x=0.75pt,y=0.75pt,yscale=-1,xscale=1]

\draw    (237,200) -- (134.27,146.42) ;
\draw [shift={(132.5,145.5)}, rotate = 27.54] [color={rgb, 255:red, 0; green, 0; blue, 0 }  ][line width=0.75]    (10.93,-3.29) .. controls (6.95,-1.4) and (3.31,-0.3) .. (0,0) .. controls (3.31,0.3) and (6.95,1.4) .. (10.93,3.29)   ;
\draw    (456,143) -- (338.3,200.13) ;
\draw [shift={(336.5,201)}, rotate = 334.11] [color={rgb, 255:red, 0; green, 0; blue, 0 }  ][line width=0.75]    (10.93,-3.29) .. controls (6.95,-1.4) and (3.31,-0.3) .. (0,0) .. controls (3.31,0.3) and (6.95,1.4) .. (10.93,3.29)   ;
\draw    (170,115) -- (420,115) ;
\draw [shift={(422,115)}, rotate = 180] [color={rgb, 255:red, 0; green, 0; blue, 0 }  ][line width=0.75]    (10.93,-3.29) .. controls (6.95,-1.4) and (3.31,-0.3) .. (0,0) .. controls (3.31,0.3) and (6.95,1.4) .. (10.93,3.29)   ;
\draw   (100,104) -- (170,104) -- (170,144) -- (100,144) -- cycle ;
\draw   (422,102) -- (492,102) -- (492,142) -- (422,142) -- cycle ;
\draw   (238,178) -- (336,178) -- (336,218) -- (238,218) -- cycle ;
\draw    (169.5,135.5) -- (418,135) ;
\draw [shift={(420,135)}, rotate = 179.89] [color={rgb, 255:red, 0; green, 0; blue, 0 }  ][line width=0.75]    (10.93,-3.29) .. controls (6.95,-1.4) and (3.31,-0.3) .. (0,0) .. controls (3.31,0.3) and (6.95,1.4) .. (10.93,3.29)   ;

\draw (106,116) node [anchor=north west][inner sep=0.75pt]   [align=left] {Principal};
\draw (436,114) node [anchor=north west][inner sep=0.75pt]   [align=left] {Agent};
\draw (244,190) node [anchor=north west][inner sep=0.75pt]   [align=left] {Environment};
\draw (185.5,90) node [anchor=north west][inner sep=0.75pt]   [align=left] {Contract Policy $\displaystyle x_{h} :\mathcal{S} \times \mathcal{S}\rightarrow \mathbb{R}_{+}$};
\draw (375,186.5) node [anchor=north west][inner sep=0.75pt]   [align=left] {Action Policy $\displaystyle \pi _{h} :\mathcal{S}\rightarrow \mathcal{A}$};
\draw (79.5,173.5) node [anchor=north west][inner sep=0.75pt]   [align=left] {State Transition \\$\displaystyle s_{h+1} \sim P( s_{h} ,\pi _{h}( s_{h}))$};
\draw (218.5,139) node [anchor=north west][inner sep=0.75pt]   [align=left] {Payment $\displaystyle x_{h}( s_{h} ,s_{h+1})$};

\end{tikzpicture}
\caption{An illustration of the interaction procedure in the principal-agent Markov decision process.}

\end{figure}
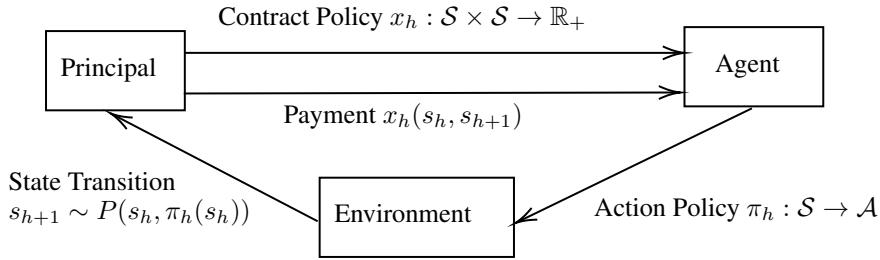

%% file: modeling-choice.tex
We make a few remarks on the procedure of the PAMDP.
\begin{enumerate}[leftmargin=*]
\item It is without loss of generality for the principal to commit his contract policy $\bx$ at the very beginning of each episode, since this MDP setup can be viewed as an extensive-form game as long as the principal as the first mover can  predict the agent's response and plan his follow-up move accordingly. Once the principal commits its contract policy, the agent can also determine his optimal action policy in response. 

\item We assume the Markovian state after its realization is publicly observable by both the agent and principal, serving as the natural conditions and contingencies for the contract design. Hence, our model directly use the state transition kernel, $P_h: \cS\times \cA \to \Delta(\cS)$, as the outcome distribution in contract design problems. Otherwise, if either agent or principal only partially observes the state, the planning problem is known to be intractable~\cite{kaelbling1998planning}, and we leave this open question for future work. A more subtle caveat here is that, different from standard episodic MDP, the transition kernel in the last step $P_H$ matters, as it influences the principal's design of contract.

\item  The principal's noise reward $\iota_h(s_h, s_{h+1})$ is set to be conditionally independent of the agent's action $a_h$, given the next state $s_{h+1}$. This is necessary for a subtle modeling reason: as we will see in the next section, the contract design problem would become much easier, if the principal can condition its payment directly based on the agent' action (i.e., without the concern of moral hazard). 
Note that since the reward itself can be modeled as a part of the state, the existence of such $\iota_h$ is without loss of generality, and there is no need to assume additional zero-mean noise on top of $\iota_h$. 
    
\item We assume the principal is able to observe the agent's action once the payment is transferred. This is well-motivated in practice. For example, a content platform may ask the creators to fill a survey on the amount of time they spent to create their content; the creators have no incentive to misreport this information, as long as their payment is independent of the answers.
The more general setup is that the principal is only able to observe a probabilistic signal of agent taking some action $a$ (e.g., from the realization of the next state and knowledge of the transition kernel). For the convenience of analysis, we save the additional steps for the principal to infer the agent's decision up to a sufficient level of confidence by repeating the same contract policy, though this could introduce additional factor of $H$ into the sample complexity, depending on the mixing ratio. We leave the tight analysis to future work.
\item It is without loss of generality to assume that the rational agent always has the incentive to participate in the PAMDP. This is because enforcing the additional constraint that the agent's utility must be non-negative under the principal's optimal contract is equivalent to adding an ``idle'' action $a_0$ to the existing action set $\cA$ with $r_h(s, a_0)=c_h(s,a_0)=0, \forall s\in\cS, h\in [H]$, which allows our analysis to ignore the agent's non-negative utility (individual rationality) constraint. 
\end{enumerate}

%% file: value-iterations.tex
With the correspondence between $\bx$ and $\bpi^{\bx}$, a natural next step is to fix $\bpi$ and find the contract policy 
$$ \bx^{\bpi} =  \argmax_{\bx \in \cX} V^{\bx, \bpi} \quad\text{s.t.}\quad \bpi = \argmax_{\bpi \in \Pi} U^{\bx, \bpi},$$ 
with the maximal value among all policy that the agent would optimally respond with action policy $\bpi$. Since the total expected reward of the principal is fixed under $\bpi$, the objective of the optimization problem can be equivalently rewritten as minimizing the principal's total payment,
$$ \bx^{\bpi}, \zeta^{\bpi} = \minarg_{\bx \in \cX}  \Ex \big[\sum_{h=1}^{H} x_h(s_h, a_h ) \big| \{ \pi_h \}_{h=1}^{H} \big] \quad\text{s.t.}\quad \bpi = \argmax_{\bpi \in \Pi} U^{\bx, \bpi}.$$
Recall that $\zeta_h^{\bpi}(s)$ denotes the least amount of expected payment to induce an action policy $\bpi$ at the state $s$ from the $h$-th step. 
Meanwhile, since $U^{\bpi}_h(s) =  P_h(s, \pi_h(s)) \cdot [x + U^{\bpi}_{h+1} ] - c_h(s,\pi_h(s))$, 
the above constraint can be equivalently rewritten as a set of constraints in an iterative form,
$$\pi_h = \argmax_{a \in \cA} P_h(s, a) \cdot [x + U^{\bpi}_{h+1} ] - c_h(s,a), \quad \forall h\in [H].$$
Therefore, such a contract policy $\bx^{\bpi}$ can be computed iteratively with backward induction from $h=H$ to $1$ with $U^{\bx}_{H+1}(s)$, $\forall s\in \cS$,
\begin{equation} \label{eq:least-payment-policy}
 \begin{aligned} 
    W^{\bpi}_h(s, a; x) &= P_h(s, a) \cdot [x + U^{\bpi}_{h+1} ] - c_h(s,a),   \\
    x^{\bpi}_h(s) &= \argmin_{x: \cS \to \RR_{+} }  \{ P_h(s, \pi_h(s)) \cdot x \ |\  W^{\bpi}_h(s, \pi_h(s); x)  \geq W^{\bpi}_h(s, a'; x), \forall a'\in \cA \},   \\
    U^{\bpi}_h(s) &= W^{\bpi}_h(s, \pi_h(s); x_h^{\bpi}(s)  ),\\ \zeta^{\bpi}_h(s) &= P_h(s, \pi_h(s)) \cdot [x_h^{\bpi}(s) + \zeta_{h+1}^{\bpi} ], \\
    V_h^{\bpi}(s) &= r_h(s, \pi_h(s)) + P_h(s, \pi_h(s)) \cdot [ V^{\bpi}_{h+1} - x^{\bpi}_h(s)  ],
\end{aligned}
\end{equation}
where the function $\zeta^{\bpi}_h(s),  V_h^{\bpi}(s)$ are computed as by-products of the value-iteration. We refer to Equation~\eqref{eq:least-payment-policy} as the least-payment Bellman equation.

%% file: proofs/bellman-optimality.tex
\begin{proof}[Proofs of Theorem~\ref{prop:bellman-opt}]
We begin by giving an interpretation for each variable in the Bellman equation. With slight abuse of notation, $ x^*_h(s; a)$ denotes the contract with the least payment to induce the agent to take action $a$ in each step $h$. Given that $\pi^*_h(s)$ is the best agent action for the principal to induce, the optimal contract at state $s$ in step $h$ can be determined as $x_h(s) = x_h(s; \pi^*_h(s) ) $. $Q_h^*(s, a), W_h^*(s, a;x)$ are respectively the principal's and agent's total expected utility from $h$-th step under policy $\{ x^*_\tau \}_{\tau=h+1}^{H}$ and $\{ \pi^*_\tau \}_{\tau=h+1}^{H}$, which can be interpreted as their optimal state-action value function at $h$-th step, serving as the intermediate variable for the computation. 
We now prove the optimality of its solution $\bx^*, V^*$ via induction:

For the base case, observe that planning for any state $s$ at the last step $H$ is reduced to a standard contract design problem and the optimal contract can be determined by solving the following linear program, $\forall a\in \cA$,
\begin{align*}
    Q_H^*(s, a), x^*_H(s;a) &= \argmin_{x: \cS \to \RR_{+} }  r_H(s,a) -  P_H(s, a) \cdot x \\
    \textup{s.t.} &  \quad P_H(s, a) \cdot x - c_H(s,a) \geq  P_H(s, a') \cdot x - c_H(s,a'), \forall a'\neq a,
\end{align*}
where $x^*_H(s;a)$ is the least payment contract to induce the agent to take action $a$ in the last step and $Q_H^*(s, a)$ is the principal's expected utility under $x^*_H(s;a)$. In Equation~\ref{eq:principal-optimal-bellman}, we save the term $r_H(s,a)$, since it is a constant once the action is fixed.
Hence,
$V_H^*(s),  a* = \maxarg_{a \in \cA} Q^*_H(s, a)$ determines the best action $\pi^*_H(s)=a^*$ for the principal to induce and $V_H^*(s)$ is the optimal state value function at $H$-th step. 
The principal's optimal contract can be determined as $x^*_H(s) =  x^*_H(s; a^*)$.
The agent's value is $U_H^*(s) = P_H(s, a^*) \cdot x^*_H(s) - c_H(s, a^*)$ based on his best response $a^*$ under $x^*_H(s)$.

For the inductive case, given that $\{ x^*_\tau \}_{\tau=h+1}^{H}$ is optimal, with the agent's best responding action policy $\{ \pi^*_\tau \}_{\tau=h+1}^{H}$, we show that $x^*_h$ solved from 
Equation~\ref{eq:principal-optimal-bellman} is optimal. 
Let us observe that 
$ W^*_h(s, a; x) = P_h(s, a) \cdot [x 
+ U^*_{h+1} ] - c_h(s,a)$ captures the agent's total utility of taking action $a$ under the contract $x$ at step $h$ state $s$ and then optimally following the action policy $\{ \pi^*_\tau \}_{\tau=h+1}^{H}$ under $\{ x^*_\tau \}_{\tau=h+1}^{H}$ from step $h+1$. Here, we can use $U^*_{h+1}$ computed from previous iteration, because the agent's value $U^*_{h+1}(s)$ is conditionally independent to the action in the current step given the realization of next state $s$ --- this enables efficient computation through dynamic programming.
Similar to the base case, for every action $a\in \cA$, the principal is to compute the least payment contract for the agent to take action $a$,
\begin{align*}
    Q_h^*(s, a), x^*_h(s;a) &= \minarg_{x: \cS \to \RR_{+} }  r_h(s, a) + P_h(s, a) \cdot [ V^*_{h+1} - x^*_h(s; a)  ] \\
    \textup{s.t.} &  \quad P_h(s, a) \cdot [x 
+ U^*_{h+1} ] - c_h(s,a)  \geq P_h(s, a') \cdot [x 
+ U^*_{h+1} ] - c_h(s,a'), \forall a'\neq a,
\end{align*}
where the objective is set as the principal's total utility if the agent takes action $a$ at current step $h$ and follows the policy $\{ \pi^*_\tau \}_{\tau=h+1}^{H}$ onward; the constraint is to reflect that it is (weakly) optimal for the agent to take action $a$ at current step. In Equation~\ref{eq:principal-optimal-bellman}, we save the term $r_h(s, a) + P_h(s, a) \cdot V^*_{h+1}$, since they are constant once the action is fixed.
With the least payment contract $x^*(s;a)$ and state value $Q^*_h(s, a)$ for each action,
$V_h^*(s),  a^* = \maxarg_{a \in \cA} Q^*_h(s, a)$ determines the best action $\pi^*_h(s)=a^*$ for the principal to induce and $V_h^*(s)$ is the optimal state value function at $h$-th step. 
The principal's optimal contract can be determined as $x^*_h(s) =  x^*_h(s; a^*)$.
The agent's value is $U_h^*(s) = P_h(s, a^*) \cdot x^*_h(s) - c_h(s, a^*)$ based on his best response $a^*$ under $x^*_h(s)$. Therefore, $x^*_h(s)$ is the optimal contract following from the optimal contract policy in previous steps $\{ x^*_\tau \}_{\tau=h+1}^{H}$, which concludes the induction.

Lastly, we note that this Bellman equation can solved efficiently using backward induction from the state-function of the $(H+1)$-step, $V^*_{H+1}, U^*_{H+1} = 0$. In each step, it solves $S\times A$ many linear programs for the optimal contract $ x^*_h(s; a)$, while each linear programs have $O(A^2)$ many constraints. Hence, the total time complexity to solve for the optimal policy $\bx^*$ is polynomial w.r.t. $A, S, H$.

\end{proof}

%% file: proofs/generic-algo.tex
We first describe the design of the generic algorithm in Algorithm~\ref{algo:black-box-ucb} and the technical lemmas.

\begin{algorithm}[h]
    \caption{Contractual bandit learning with $\varepsilon$-margin contract sets}
    \label{algo:black-box-ucb}
        \KwIn{$\{\cX^{a}(\varepsilon)\}_{a\in \cA}$, the $\varepsilon$-margin contract sets. }
        \For{$t = 1\dots T$}{ 
Estimate the least payment for each action under the empirical outcome distribution,
\begin{equation}\label{eq:robust-contract-empirical}
\hat{\zeta}_t(a) \gets \min_{x\in \cX^{a}(\varepsilon)}  \hat{P}_t(a)\cdot x.
\end{equation} \\

Determine the best action based on the optimistic estimation of profit,
$$
a_{t} \gets \argmax_{a\in \cA} \hat{r}_t(a) - \hat{\zeta}_t(a) + (1+\eta)\epsilon_t(a).
$$ \\
Solve for a robust contract to induce $a_t$,
$$
x_t \gets \argmin_{x\in \cX^{a_t}(\varepsilon)}  \hat{P}_t(a_t)\cdot x.
$$\\
Commit to contract $x_t$ and observe the agent's action $a'_{t}$, outcome $s_t$ and its reward $\iota_t(s_t)$. \\
Update the empirical estimation of the outcome distribution and reward function, $\hat{P}_t, \hat{r}_t$.\\
Set the confidence interval $\epsilon_t$ such that $\norm{P(a) - \hat{P}_t(a)}_1 \leq \epsilon_t(a), \forall a\in\cA$, with prob. $1-\delta$. \\
}
\end{algorithm}

\begin{lemma}\label{lm:robust-contract-loss}
    Under Assumption \ref{assum:inducibility},
    for each action $a \in \cA$, given a robust contract set $\cX^a(\varepsilon)$ with margin $\varepsilon$, and an empirical estimation of $\hat{P}(a)$ with $\norm{\hat{P}(a) - P(a)}_1 \leq \epsilon$, let $\hat{x}, \hat{\zeta}(a)$ be the minimizer and minimum objective value of LP \eqref{eq:robust-contract-empirical}. The following conditions are satisfied,
    \begin{enumerate}[leftmargin=*]
        \item 
        The expected payment of $\hat{x}$ is bounded as,
        $ 0 \leq P(a)\cdot \hat{x} - {\zeta}(a) \leq \lambda^{-1} \varepsilon. $
        \item 
        The estimated payment of $\hat{x}$ is bounded as,
        $
        - \eta\epsilon \leq \hat{\zeta}(a) - \zeta(a) \leq \lambda^{-1} \varepsilon + \eta\epsilon.
        $
    \end{enumerate}
    
\end{lemma}

\begin{lemma}[\citet{mcdiarmid1989method}]\label{lm:learning-distribution}
    With $t$ i.i.d. samples of an $m$-dimensional distribution $Q$, we can construct a confidence ball $ \cB = \{Q\in \Delta^m : \norm{\hat{Q}_t - Q}_1 \leq \sqrt{ \frac{m\log (1/\delta) }{ t } } \}$ such that $Q\in \cB$ with prob. at least $1-\delta$. 
\end{lemma}

\begin{proof}[Proof of Theorem~\ref{thm:decoupling}]
At a high level, Algorithm \ref{algo:black-box-ucb} proceeds by following the upper confidence bound of the expected ``profit'' of each action $z(a) := r(a) - \zeta(a)$, which shrinks at the rate of $t^{-1/2}$, based on Lemma \ref{lm:robust-contract-loss}. This enables us to apply the upper confidence bound analysis from online learning problem to the contractual bandit learning problem. 

That is, we construct a variable $\tilde{z}_t(a) := \tilde{r}_t(a) - \hat{\zeta}_t(a) + (1+\eta)\epsilon_t(a) + \lambda^{-1}\varepsilon$ as an optimistic estimation of $z(a)$.
First, notice that Algorithm \ref{algo:black-box-ucb} is equivalently to follow the action $a_t = \argmax_{a\in \cA}  \tilde{z}_t(a) $ at each round $t$, as $\lambda^{-1}\varepsilon $ is constant and does not affect the optimization. 
Second, we show that under the difference between $\tilde{z}_t(a)$ and $z(a)$ satisfies the following inequality with probability at least $1-\delta$, 
\begin{equation}\label{eq:optimistic-z}
    0 \leq  \tilde{z}_t(a) - z(a) \leq (2+2\eta)\epsilon_t(a) + \lambda^{-1}\varepsilon,\quad \forall a\in \cA, t\in [T].   
\end{equation}
On the event that $\norm{\hat{P}_t(a) - P(a)}_1 \leq \epsilon_t(a)$, we can derive that $\abs{ \hat{r}_t(a) - r(a) } = [\hat{P}_t(a) - P(a)]\cdot \iota \leq  \epsilon_t(a)$ and $ - \eta\epsilon_t(a) \leq \hat{\zeta}(a) - \zeta(a) \leq \lambda^{-1} \varepsilon + \eta\epsilon_t(a)$ by Lemma~\ref{lm:robust-contract-loss}. This implies that
$ -(1+\eta)\epsilon_t(a) - \lambda^{-1}\varepsilon \leq \hat{r}_t(a) - \hat{\zeta}_t(a) - r(a) + \zeta(a) \leq (1+\eta)\epsilon_t(a)
$, which leads to the Equation \eqref{eq:optimistic-z}.

Under the event that $\norm{\hat{P}_t(a) - P(a)}_1 \leq \epsilon_t(a), \forall a\in \cA$, we have $a_t = a'_t, \forall t\in [T]$ and the expected regret of Algorithm \ref{algo:black-box-ucb} in the $T$ rounds is as follows,
$$
\Reg(T) = \max_{a^*\in \cA} \sum_{t=1}^{T} [r(a^*) - \zeta(a^*) ] - \sum_{t=1}^{T} [ r(a_t) - P(a_t)x_t ].
$$

We decompose the regret into two cases on whether the optimal arm $a^*$ is played at round $t$:

When $a_t = a^*$, we have $[r(a^*) - \zeta(a^*) ] - [ r(a_t) - P(a_t)x_t ] =  P(a_t)x_t - \zeta(a^*) \leq \lambda^{-1}\varepsilon $ by Lemma~\ref{lm:robust-contract-loss}.

When $a_t \neq a^*$, we have 
\begin{align*}
     [r(a^*) - \zeta(a^*) ] - [ r(a_t) - P(a_t)\cdot x_t ]
    & \leq  \tilde{z}_t(a^*) - r(a_t) + P(a_t) \cdot x_t \\
    & \leq \tilde{z}_t(a_t)- r(a_t) + \zeta(a_t) + \lambda^{-1}\varepsilon \\
    & = \tilde{z}_t(a_t)- z(a_t) + \lambda^{-1}\varepsilon \\
    & \leq  (2+2\eta)\epsilon_t(a_t) + 2\lambda^{-1}\varepsilon,
\end{align*}
where the first inequality follows from Equation~\eqref{eq:optimistic-z} that $ z(a^*) \leq \tilde{z}_t(a^*) $; the second inequality uses the fact that $\tilde{z}_t(a^*) \leq  \tilde{z}_t(a_t)$ and $P(a_t)x_t \leq  \zeta(a_t) + \lambda^{-1}\varepsilon$ from Lemma~\ref{lm:robust-contract-loss}; the third inequality follows Equation~\eqref{eq:optimistic-z} that $\tilde{z}_t(a^*)- z(a^*) \leq (2+2\eta)\epsilon_t(a_t) + \lambda^{-1}\varepsilon$.

It remains to bound the total regret based on the exact choice of $\epsilon_t$ in different setup.

\paragraph{The Case of Finite Action Space.}
In the case where the action space $\cA$ is finite.
For any action $a\in \cA$, we denote $N_t(a)$ as the number of times action $a$ has been taken. By Lemma \ref{lm:learning-distribution}, we can set $\epsilon_t(a) = \sqrt{ \frac{|\cS| \log ( T|\cA|/\delta ) }{N_t(a)}  }$ such that the empirical estimation of the outcome distribution $\hat{P}_t$ satisfies $\norm{ \hat{P}_t(a) - P(a)}_1 \leq \epsilon_t(a) $ with probability at least $1-\frac{\delta}{T|\cA|}$.  Thus, by union bound, with probability $1-\delta$, the expected regret can be bounded as follows,
\begin{align*}
\Reg(T) 
& \leq \sum_{t=1}^{T} (2+2\eta)\sqrt{ \frac{|\cS| \log (T|\cA| /\delta) }{N_t(a_t)}} + 2\lambda^{-1}\varepsilon \\
& \leq (4+4\eta)\sqrt{  |\cS| \log (T|\cA| /\delta)  } \sum_{a\in \cA} \sqrt{N_T(a)}  + 2\lambda^{-1}\varepsilon T  \\
& \leq (4+4\eta)\sqrt{  |\cS| \log (T|\cA| /\delta)  } \sqrt{|\cA|T}  + 2\lambda^{-1}\varepsilon T  \\
& = O(\eta\sqrt{T|\cA||\cS|\log(T|\cA|/\delta)} + \varepsilon T/\lambda),
\end{align*}
where the first inequality uses the fact that the loss incur when $a_t \neq a^*$ is at least as much as the loss when $a_t = a^*$; the second inequality follows from the Cauchy-Schwarz inequality; the third inequality again applies Cauchy-Schwarz inequality and use the fact that $\sum_{a\in \cA} N_T(a) = T$.

\paragraph{The Case of Infinite Action Space with Linear Context.} 
In the case when the action space $\cA\subset \RR^d$ is infinite, the outcome distribution $P(a) = a^\top \theta$ for some unknown parameter $\theta \in \RR^{d\times m}$. Let $\hat{\theta}_t(a) =\Sigma^{-1}\sum_{\tau=1}^{t} s_\tau a_{\tau}  $ with $\Sigma_t = \lambda I + \sum_{\tau=1}^{t} a_\tau a_\tau^\top $. By Lemma 11 of \citet{abbasi2011improved}, with probability at least $1-\delta$, we have $\norm{\hat{\theta}_t - \theta^*}_{\Sigma_t} \leq \sqrt{\beta_t}$ , where $\beta_t = \sigma^2 (2+4d\log(T+1) + 8 \log(4/\delta)) d \log(1+\frac{T}{d\sigma^2})$.
We can set $\epsilon_t = \sqrt{\beta_t \norm{a}_{\Sigma_t^{-1}}}$ such that empirical estimation of the outcome distribution $\hat{P}_t = a^\top \hat{\theta}_t $ satisfies 
$$\norm{ \hat{P}_t(a) - P(a)}_1 = \norm{ (\Sigma_t^{-1/2}a)^\top \Sigma_t^{1/2}(\hat{\theta}_t - \theta^*) }_1  \leq \norm{\Sigma_t^{-1/2}a} \norm{ \Sigma_t^{1/2}(\hat{\theta}_t - \theta^*) }  \leq \sqrt{\beta_t \norm{a}_{\Sigma_t^{-1}}} = \epsilon_t(a) .$$ 
Thus, by union bound, with probability $1-\delta$, the expected regret can be bounded as follows,
\begin{align*}
    \Reg(T) & \leq \sum_{t=1}^{T} (2+2\eta)\sqrt{\beta_t \norm{a}_{\Sigma_t^{-1}}}  + 2\lambda^{-1}\varepsilon \\
    & \leq (2+2\eta)\sqrt{ T \sum_{t=1}^{T} \beta_t^2 \norm{a}_{\Sigma_t^{-1}}^2  }  + 2\lambda^{-1}\varepsilon T  \\
    & = O\bigg( (1+\eta)\sqrt{T} \big(d \log(T) + \log(1/\delta) \big) + \varepsilon T/\lambda \bigg),
\end{align*}
where the first inequality uses the fact that the loss incur when $a_t \neq a^*$ is at least as much as the loss when $a_t = a^*$; the second inequality follows from the Cauchy-Schwarz inequality; the third inequality again applies Cauchy-Schwarz inequality and use the fact that $\sum_{a\in \cA} N_T(a) = T$.

\end{proof}

\begin{proof}[Proof of Lemma~\ref{lm:robust-contract-loss}]
    Pick an arbitrary $a\in \cA$.
    We have ${\cX}^{a}(\varepsilon) = \{x: [P(a)-P(a')]\cdot x \geq c(a)-c(a') + \varepsilon, \forall a \neq a'\}$ and an empirical estimation of $\hat{P}(a)$ with $\norm{\hat{P}(a) - P(a)}_1 \leq \epsilon$, LP \eqref{eq:robust-contract-empirical}.
    With ${\cX}^a(\varepsilon)$ and $\hat{P}(a)$, LP \eqref{eq:robust-contract-empirical} solves for a robust contract $\hat{x}$. Since $\hat{x} \in \cX^a(\varepsilon) \subseteq \cX^a$, the agent's best response to $\hat{x}$ is to take action $a$.

    First, we derive a bound for an intermediate value $\min_{x \in {\cX}^a(\varepsilon) } P(a) \cdot x$. Recall that Assumption~\ref{assum:inducibility} guarantees that, for any $x\in \cX^a$, there exists $\bar{x} = x + \lambda^{-1} \varepsilon e$ such that $\bar{x} \in {\cX}^a(\varepsilon)$. Let $x^* = \argmin_{x \in \cX^a} P(a) \cdot x$ and $\bar{x}^* = \argmin_{x \in {\cX}^a(\varepsilon)} P(a) \cdot x $.
    We have
    $$
    \min_{x \in {\cX}^a(\varepsilon) } P(a) \cdot x
    = \min_{x \in \cX^a} P(a) \cdot x +  P(a)\cdot (\bar{x}^* - x^*).
    $$

    Since $\norm{P(a)}_{1}=1, \norm{\bar{x}^* - x^*}_{\infty}\leq \lambda^{-1} \varepsilon$, we have $ \abs{ P(a)\cdot (\bar{x}^* - x^*) } \leq \norm{P(a)}_{1} \norm{\bar{x}^* - x^*}_{\infty}$. In addition, as $ {\cX}^a(\varepsilon) \subseteq {\cX}^a$, $\min_{x \in {\cX}^a(\varepsilon) } P(a) \cdot x \geq \min_{x \in {\cX}^a } P(a) \cdot x = \zeta(a)$. Hence, we have
    \begin{equation}\label{eq:payment-upper-bound}
        0 \leq  \min_{x \in {\cX}^a(\varepsilon) } P(a) \cdot x - \zeta(a)  \leq \lambda^{-1} \varepsilon.
    \end{equation}

    Notice that $\hat{x} \in {\cX}^a(\varepsilon)$ and is not necessarily the minimizer of $P(a) \cdot x$ over ${\cX}^a$. We get the first condition of this lemma, by Equation \eqref{eq:payment-upper-bound}
    \begin{equation*}
        \zeta(a) = \min_{x \in {\cX}^a} P(a) \cdot x \leq P(a)\cdot \hat{x} 
        \leq \min_{x \in {\cX}^a(\varepsilon)} P(a) \cdot x 
        \leq \zeta(a) + \lambda^{-1} \varepsilon.
    \end{equation*}

    We now bound the estimate payment $\hat{\zeta}(a) = \hat{P}(a)\cdot \hat{x} $.
    Since $\norm{P(a) - \tilde{P}}_{1} \leq \epsilon, \norm{\tilde{x}^a}_{\infty} \leq \eta$ by the bounded contract space assumption, we have
    \begin{equation}\label{eq:payment-difference-bound}
     \abs{ \hat{\zeta}(a) - P(a) \cdot \hat{x} }  = \abs{ [P(a)-\hat{P}(a)] \cdot \hat{x} } \leq \norm{P(a)-\hat{P}(a)}_{1} \norm{\hat{x}}_{\infty} \leq \eta\epsilon. 
    \end{equation}
    Therefore, combining Equation \eqref{eq:payment-difference-bound} and the first condition, we get the second condition of this lemma,
    $$ 
    - \eta\epsilon \leq \hat{\zeta}(a) - \zeta(a) \leq \lambda^{-1} \varepsilon + \eta\epsilon.
    $$

\end{proof}

%% file: proofs/mab-search.tex
\paragraph{Multi-Armed Bandits under Direct Incentives}
This is perhaps the most simple yet natural class of contractual online learning problems. The principal is unable to directly pull arms but is able to receive the reward from arm pulled by the agent. 
In this problem, we have the action space $\cA = [N]$ and $r, c: [N] \to [0,1]$ specifying the principal's reward and agent's cost of pulling each arm $i\in [N]$. 
At the beginning of each round $t$, the principal sets a contract $x_t: [N] \to \RR_{+}$ and the agent accordingly decides its best response $i_t = \max_{i\in [n]} \big[ x_t(i) - c(i) \big]$.
At the end of each round $t$, the principal is able to observe the exact arm $i_t$ taken by the agent as well as the noisy bandit feedback on its corresponding reward $\tilde{r}_{t}(i) = r(i) + \epsilon_t$, where $\epsilon_t$ is zero-mean, i.i.d. $\sigma$-subGuassian noise. 
Finally, the learning goal of the principal is to minimize the regret, $\Reg(T) = T\cdot \max_{i\in [N]} \big[ r(i) - c(i) \big] - \sum_{t\in [T]} [r(i_t) - x_t(i_t)]  $.  

\begin{lemma}[Binary Search for Finite Arms] \label{lm:binary-search}
There exists an \oraclearg[$O(|\cA|\log(1/\varepsilon) )$] for multi-armed bandits under direct incentives.
\end{lemma}

\begin{proof}[Proof of Lemma~\ref{lm:binary-search}]
We show an explicit construction of \oracle{} in the problem. Observe that, if we can learn an estimation of $ | \hat{c}(a) - c(a) | \leq \varepsilon/2, \forall a\in \cA$, we can set the least payment contract $x$ as follows, $x(a)= \hat{c}(a) + \varepsilon/2, x(a')= 0, \forall a'\neq a$.  $x(a) - c(a) \geq 0 \geq x(a') - c(a')$, it is optimal for the agent to respond with action $a$. Moreover, the payment is minimized as $x(a) - \varepsilon \leq c(a) + \varepsilon/2 + \varepsilon/2 - \varepsilon = c(a) =  \zeta(a)$. 

So it only remains to learn the estimation of $ | \hat{c}(a) - c(a) | \leq \varepsilon / 2, \forall a\in \cA$. This can be achieved through binary search. For any action $a$, we set a cost lower bound $c^-(a)$ and upper bound $c^+(a)$. At each round, the algorithm sets the contract $x$ with $x(a) = \frac{ c^-(a) + c^+(a)}{2}, x(a')= 0, \forall a'\neq a$. If the agent takes the action $a$, then the algorithm updates $c^-(a) \gets \frac{ c^-(a) + c^+(a)}{2}$. Otherwise, it updates $c^+(a) \gets \frac{ c^-(a) + c^+(a)}{2}$. In $\log (1/\varepsilon) + 1$ rounds, the algorithm is guaranteed to have $ c^+(a) - c^-(a) \leq \varepsilon / 2$ and thus an estimation $ | \hat{c}(a) - c(a) | \leq \varepsilon / 2$. To conduct the binary search for every action, the total sample complexity is $O(|\cA| \log (1/\varepsilon) )$. 

\end{proof}

%% file: proofs/lin-bandit-search.tex
\paragraph{Linear Bandits under Direct Incentives}
In this problem, we have the action space $\cA \subset \RR^d$ (composed of the context vectors) and $r, c: \cA \to [0,1]$ specifying the principal's reward and agent's cost of choosing each context $a\in \cA$.
At the beginning of each round $t$, the principal observes a set of contexts $\cA_t \subset \cA $ and sets a contract $x_t: \cA_t \to \RR_{+}$. The agent accordingly decides its best response $a_t = \max_{a\in \cA_t} \big[ x_t(a) - c(a) \big]$, where $c(a) = a^\top\gamma$. 
At the end of each round $t$, the principal is able to observe the exact arm $a_t$ taken by the agent as well as the noisy bandit feedback on its corresponding reward $r(a_t) = a_t^\top \theta + \epsilon_t$, where $\epsilon_t$ is zero-mean, i.i.d. $\sigma$-subGuassian noise. $\theta, \gamma$ are fixed, unknown parameters to be learnt. Without loss of generality, we assume $\norm{\theta} \leq 1, \norm{\gamma} \leq 1, \norm{a_{t}} \leq \sqrt{d}$ by coordinate transformation.
Finally, the learning goal of the principal is to minimize the regret, 
$
\Reg(T) = \sum_{t=1}^{T} [ (a^*_t)^\top \theta^* -  (a^*_t)^\top \gamma^* - a_t^\top \theta^*  +  x_t(a_t)  ],
$
where $a^*_t = \argmax_{a_t \in \cA_t} \{  a_t^\top \theta^*  -  a_t^\top \gamma^* \} $ is the optimal arm at round $t$.

\begin{lemma}[Contextual Search for Infinite Arms] \label{lm:contextual-search}
There exists an \oraclearg[$O(d\log 1/\varepsilon)$] for linear bandits under direct incentives.
\end{lemma}

\begin{proof}[Proof of Lemma~\ref{lm:contextual-search}]
We show an explicit construction of \oracle{} in the problem with agent's best response function $h^*(x) = \argmax_{a\in\cA} \{x(a) -  a^\top \gamma^* \}$ for some parameter $\gamma^*\in \RR^d,  \norm{\gamma^*} \leq 1 $ and action set $\cA \subset \RR^d$. Observe that, if we can learn an estimation of $\gamma$ such that  $ \norm{ \gamma - \gamma^* } \leq \frac{1}{2t} $, we can set the least payment rule $x$ such that $x(a)= \gamma^\top a + \frac{1}{2t}, x(a')= 0, \forall a'\neq a$. Since $x(a) - c(a) \geq  \frac{1}{2t} - \norm{ \gamma - \gamma^* } \cdot \norm{a} \geq 0 \geq x(a') - c(a')$, we have $h(x)=a$. Moreover, $x(a) - c(a) \leq \frac{1}{2t} + \norm{ \gamma - \gamma^* } \cdot \norm{a}  \leq \frac{1}{t} $. 

To learn an estimation of $\gamma$ such that $ \norm{ \gamma - \gamma^* } \leq \frac{1}{2t} $, we adopt the contextual search algorithm under symmetric loss~\cite{liu2021optimal}. At a high level, we use the constant regret guarantee of contextual search algorithm again adversarially chosen context at every round, and we present a simple argument assuming $\cA_t = \cA$ that allows us to pick arbitrary context for the contextual search algorithm. 
Specifically, consider a contextual search problem with the unknown vector $\gamma^* \in \RR^d$ and $\norm{\gamma^*} \leq 1$.  Fix any unit vector $e\in \cA$, in $O(t)$ rounds, the contextual search algorithm can determine a knowledge set $\Gamma_t$ of all feasible $\gamma$ such that $ \max_{\gamma \in \Gamma_t} | \gamma^\top e - (\gamma^*)^\top e | \leq 2^{-t}$. 
Repeating this search procedure for all $d$ linearly independent direction in $\cA$, we obtain a knowledge set of all feasible $\gamma$ such that $ \max_{\gamma \in \Gamma_t} | \gamma^\top a - (\gamma^*)^\top a | \leq 2^{-t}, \forall a\in \cA$, since any action $a$ can be decomposed as a convex combination of the $d$ linearly independent unit vectors. The total sample complexity is $O(d \log(1/\varepsilon))$.
\end{proof}

%% file: proofs/known-distribution-search.tex
\begin{lemma}\label{lm:generic-cost-search}
Under Assumption \ref{assum:prior-contract} and given $\hat{P}$ that satisfies 
$\norm{\hat{P} - P}_{1,\infty} \leq \varepsilon/\eta$, 
we can construct an \oraclearg[$O\big( |\cA|^2 \log (|\cA|\eta/\varepsilon) \big)$] for general contractual bandit learning problems.
\end{lemma}

\begin{proof}[Proof of Lemma~\ref{lm:generic-cost-search}]
We denote $d(a,a'):=c(a)-c(a')$ and the learning procedure is to query certain contract in a binary search fashion in order to obtain estimation $\hat{d}$ with bounded error $|\hat{d}(a, a') - d(a, a')|, \forall a,a'$, from which we can construct $\varepsilon$-margin contract set $\hat{\cX}^{a}$ for any action $a \in \cA$ and thereby compute almost least payment contract according to Lemma \ref{lm:robust-contract-loss}. For precise analysis, let $\norm{\hat{P} - P}_{1,\infty} \leq \epsilon = \frac{\varepsilon}{10\eta}$.
We describe the full procedure in Algorithm~\ref{algo:bandit-search}.

\begin{algorithm}[tbh]
    \caption{\Oracle{} in Contractual Bandit Learning}
    \label{algo:bandit-search}
        \KwIn{Action set $\cA$, estimated parameters $\hat{P}$. }
        \Output{Robust contract sets $ \hat{\cX}^{a}, \forall a\in \cA$.}
        $\hat{d}(a, a') \gets \infty, \forall a, a' \in \cA $.\\
        Set binary search precision $\epsilon = \frac{\varepsilon}{10\eta|\cA|}$. \\
        \For{each $a \in \cA$}{ 
        Construct a contract $x^a$ that induces the action $a$. \\
        \For{each $a' \neq a \in \cA'$}{ 
        Construct a contract $x^{a'}$ that induces the action $a'$. \\
        Binary search for parameter $\alpha\in (0,1)$ such that $x = \alpha x^a + (1-\alpha)x^{a'}$ induces action $a$, while $x' = (\alpha + \epsilon) x^a + (1-\alpha - \epsilon)x^{a'}$ induces action $a''$. \\
        Use $x, x'$ to solve for $\hat{d}(a, a'') \gets \big[ \hat{P}(a)  - \hat{P}(a'') \big] \cdot x^{a} $.
            }
        }
        \For{each $a,a' \in \cA$}{ 
        \If{$\hat{d}(a,a') = \infty$}{
        $\hat{d}(a,a') \gets \min_{\cP} \sum_{(a_i, a_j)\in \cP} \hat{d}(a_i, a_j) $, where $\cP$ is a choice of path from $a$ to $a'$.
        }
        }
        \Return $ \hat{\cX}^{a} = \{ x\in \cX : \big[ \hat{P}(a)  - \hat{P}(a') \big]\cdot x \geq  \hat{d}(a, a') + \varepsilon/2, \forall a\neq a' \}$ for each $a\in \cA$
\end{algorithm}

To prove its correctness, we start from the observation that for any two action $a, a'$ with sufficiently small $\epsilon$, given two contracts $x^{a}, x^{a'}$ that respectively induces action $a, a'$ and $\norm{x^{a}-x^{a'}}_{\infty} \leq \epsilon$, we can obtain the estimation $\hat{d}(a, a')$ such that $|\hat{d}(a, a') - d(a, a')| < 3\epsilon\eta = 3\varepsilon/10$. 
To see this, we introduce a contract $x^0 = \alpha x^{a} + (1-\alpha)x^{a'}$ for some $\alpha\in(0,1)$ such that $\big[ P(a) - P(a') \big] x^{0} = d(a, a')$. Such $x^0$ must exist, since $\big[ P(a)  - P(a') \big] x^{a} > d(a, a')$ and $\big[ P(a) - P(a') \big] x^{a'} < d(a, a')$. 
Now let $\hat{d}(a, a') = \big[ \hat{P}(a)  - \hat{P}(a') \big] \cdot x^{a} $, we have 
\begin{align*}
  \quad \abs{ \hat{d}(a, a') -  d(a, a') }
 & = \abs{ \big[ P(a)  - P(a') \big] \cdot (x^{a} - x^{0}) + \big[ \hat{P}(a) - P(a)  - \hat{P}(a') + P(a')\big] \cdot x^{a} } \\
 & \leq (1-\alpha) \abs{ \big[ P(a)  - P(a') \big] \cdot (x^{a} - x^{a'}) }
 + 2 \epsilon \eta  \\
 & \leq (1-\alpha) \epsilon^2 + 2 \epsilon \eta  \\
 & < 3\epsilon \eta  = 3\varepsilon/10.
\end{align*}

To obtain that the contracts $x^{a}, x^{a'}$, it only requires to do a binary search based on two initial contracts $x^{a}, x^{a'}$ that induces action $a, a'$. Then, if the contract $x' = \frac{1}{2}x^{a} + \frac{1}{2}x^{a'} $ induces the action $a$, then we update $x^{a} \gets x'$. Otherwise, $x^{a'} \gets x'$. In $\log(1/\epsilon)$ rounds, the distance of $x^{a}$ and $x^{a'}$ is bounded by $\epsilon$. 
As is described in Algorithm \ref{algo:bandit-search}, we can do such binary search for every pair of actions $a,a'$.
While two actions may not share a decision boundary, we identify all action pairs that do share a decision boundary with each other. This means for pairs that do not share a decision boundary, we can find a path through their neighbours to determine their cost difference given by $d$ and the shortest path can find by the Dijkstra's algorithm in $O(|\cA|^2)$. In the worst case, such path can be as long as $|\cA|- 2$, this means we need to conduct binary search to the precision level of $\epsilon/|\cA| = \frac{\varepsilon}{10\eta |\cA|}$ for $O(\log(|\cA|\eta/\varepsilon)) $ rounds.

Finally, with the estimated parameters $\hat{P}$ and $\hat{d}(a, a')$, the algorithm construct the robust contract set 
$\hat{{\cX}}^{a} = \{ x\in \cX: \big[ \hat{P}(a)  - \hat{P}(a') \big]\cdot x \geq  \hat{d}(a, a') + \varepsilon/2 , \forall a\neq a' \}$, and
we claim that ${\cX}^{a}(\varepsilon) \subseteq \hat{{\cX}}^{a}  \subseteq {\cX}^{a}$. 
To verify that $\hat{{\cX}}^{a}  \subseteq {\cX}^{a} $, we can check that the following inequality must hold, $\forall \hat{x} \in \hat{{\cX}}^{a}, \forall a' \neq a$,
\begin{align*}
  [P(a) - P(a')] \cdot \hat{x} 
& \geq   [\hat{P}(a) - \hat{P}(a')] \cdot \hat{x}  +  [{P}(a) - \hat{P}(a) - {P}(a') + \hat{P}(a')] \cdot \hat{x}  \\
& \geq   \hat{d}(a, a') + \varepsilon/2  - 2\max_{a\in \cA}\norm{\hat{P}(a) - {P}(a)}_1 \norm{x}_{\infty} \\
& \geq  d(a,a') + \varepsilon/2  - \varepsilon/5 - 3\varepsilon/10 \geq d(a,a').   
\end{align*} 
Similarly, to verify ${\cX}^{a}(\varepsilon) \subseteq \hat{{\cX}}^{a} $, we can check that the following inequality must hold, $\forall x \in {\cX}^{a}(\varepsilon), \forall a' \neq a$,
\begin{align*}
  [\hat{P}(a) - \hat{P}(a')] \cdot x
& \geq   [{P}(a) - {P}(a')] \cdot {x}  +  [\hat{P}(a) - {P}(a) - \hat{P}(a') + {P}(a')] \cdot {x}  \\
& \geq   {d}(a, a') + \varepsilon  - 2\max_{a\in \cA}\norm{\hat{P}(a) - {P}(a)}_1 \norm{x}_{\infty}  \\
& \geq  \hat{d}(a,a') + \varepsilon   - \varepsilon/5 - 3\varepsilon/10 \geq \hat{d}(a,a') + \varepsilon/2.   
\end{align*}

\end{proof}

%% file: proofs/hyperplane-search.tex
\section{Searching on Probability Simplex}\label{sec:search-distribution-diff}
In this section, we discuss the details related to specifying the information structure through hyperplane searching. A motivation for doing hyperplane searching is given in Example \ref{ex:motivation for hyperplane search}, where learning the outcome distribution difference with $\cO(\log T)$ rounds potentially avoid pulling the non-optimal arm too many times and paves the way for constructing $T^{-1}$-optimal contract. In addition, the need to plan with the transition kernel $P_h(\cdot\given s, a)$ in the MDP environment with far-sighted agent prompts us to learn the difference in $P_h(\cdot\given s, a)-P_h(\cdot\given s, a')$ in order to fully exploit the information structure and as well reduce the cost of redundant explorations.

\begin{example}[$o(T^{2/3})$ regret with known cost]\label{ex:motivation for hyperplane search}
Consider a class of contractual bandit problem instances parameterized on $\mu \in (0, 1]$. For each instance, there are two outcomes $s_1, s_2$ with mean reward $\iota(s_1)=1, \iota(s_2)=0$, and two agent actions $a_1, a_2$ with cost $c(a_1) = 1/2, c(a_2) = 0$ and outcome distribution $P(a_1) = [1, 0], P(a_2) = [1-\mu, \mu]$. 
One can verify that the optimal contract $x^*$ here is to set $x^*(s_1) = \frac{1}{2\mu}, x^*(s_2) = 0$ and the principal gets the expected utility $1 - \frac{1}{2\mu}$. 
The naive learning method is to play $a_2$ for $T^{2/3}$ rounds and learn its outcome distribution parameterized by $\mu$ up to the bounded error $O(T^{-1/3})$. This is costly as $a_2$ is the sub-optimal arm, resulting in $\tilde{O}(T^{2/3})$ regret in Theorem~\ref{thm:decoupling}. 
However, an alternative method is to conduct a binary search for $\mu$. 
This would achieve $O(\log T)$ regret, since the algorithm can get estimation error of $\mu$ bounded by $T^{-1}$ in $O(\log T)$ rounds, and construct an $T^{-1}$-optimal contract. 

\end{example}

Here, we consider searching for the agent's best response section in a $D$-dimensional probability simplex. Let $\cS$ with $|\cS|=d$ be the outcome space and  $x:\cS\rightarrow \RR_+$ denote the contract the principal announces to the agent. Here, we restrict the contract $x$ to a subspace $x\in\cP^{d-1}$ where $\cP^{d-1}$ is the $(d-1)$-dimension probability simplex.
We remark that searching over a low dimensional simplex is without loss of generality, and we just consider the simplex  $\nbr{x}_1=\eta, x\in[0, \eta]^{D}$ for simplicity, where $\eta$ bounds the infinity norm of any contract we use. 
Let $\cA$ be the agent's action set with $\abr{\cA}=N$.
This is because we have the following proposition.
\begin{proposition}[Action inducibility]
    If an action $i\in[N]$ can be induced by contract $x$ with $\nbr{x}_\infty \le \eta$, then $i$ can also be induced by contract $y=x+((N-1)\eta-\nbr{x}_1)\ind/N$.
    \begin{proof}
        The inducibility condition implies
        \begin{align*}
            \inp[]{x}{p_i-p_j}\ge c_i-c_j, \forall j\neq i.
        \end{align*}
        Obviously, adding $(\eta-\nbr{x}_1)\ind/N$ to $x$ does not change the inequality. Moreover, 
        \begin{align*}
            y^{(l)}=x^{(l)}+((N-1)\eta-\nbr{x_1})/N=x^{(l)}\frac{N-1}{N} + \frac{(N-1)\eta}{N} - \frac 1 N \sum_{m\neq l} x^{(m)} \ge 0, 
        \end{align*}
        which implies that $y$ is a valid contract with $\nbr{y}_1 = (N-1)\eta$.
    \end{proof}
\end{proposition}
For simplicity, we ignore the scale $(N-1)\eta$ and just conduct our search on the probability simplex. 
Under this setting, the best response region for $a_i\in\cA$ is
\begin{align*}
    \cV_i = \cbr{x\in\cP^{d-1}\biggiven \inp[]{x}{p_i-p_j}\ge c_i-c_j, \quad\forall j\neq i}, 
\end{align*}
where $p_i\in\Delta(\cS)$ is the outcome distribution under action $i$ and $c_i$ is the action cost the agent has to pay for any $i\in[N]$.
Our target is to identify each $\cV_i$ by searching for the hyperplanes that separate these $\cV_i$ under weak assumptions. Specifically, we assume that the cost of each action is known. The algorithm is summarized in Algorithm \ref{algo:hyperplane search}.

\begin{algorithm}[t]
    \caption{Searching on Probability Simplex}\label{algo:hyperplane search}
        \KwIn{Number of actions $N$, number of samples $T$, binary search threshold $ \varepsilon$, parameters $c_d$. }
        Initial memory $\cM=\emptyset$\;
        \For{$t = 1,\dots, T$}{ 
            Randomly sample $z_1, z_2\in\cP^{d}$ and draw the line $\ell\subset \cP^d$ connecting $z_1, z_2$\;
            Binary search on $\ell$ for all the switching points up to precision $\varepsilon$ \footnotemark, and obtain all the segments containing a switching point: $\overbar{x_1 y_1}, \dots, \overbar{x_m y_m}$\;
            \For{$k=1, \dots, m$}{
                $\cM\leftarrow \cM\cup\cbr{\rbr{x_k, a^*(x_k)}}\cup \cbr{\rbr{y_k, a^*(y_k)}}$\;
                Randomly draw a $d$-dimensional simplex centered at $(x_k+y_k)/2$ with length $\sqrt 2 c_d$ and vertices $v_1, \dots, v_{d+1}$\;
                Play $v_1, \dots, v_{d+1}$ and obtain the best response $a^*_1, \dots, a^*_{d+1}$\;
                \For{each pair $(i,j)$ s.t. $1\le i<j\le d+1$ and $a^*_i\neq a^*_j$}{
                    Binary search on $\overbar{v_i v_j}$ for a switching point up to precision $\varepsilon$, and obtain the segment $\overbar{u w}$ containg the switching point\;
                    $\cM\leftarrow \cM\cup\cbr{\rbr{u, a^*(u)}}\cup \cbr{\rbr{w, a^*(w)}}$\;
                }
            }
        }
        Solve for $\cS$ with $\cM$\;
\end{algorithm}

\footnotetext{Precision $\varepsilon$ in this algorithm always means the $d$-dimensional infinity-norm $\nbr{x_k-y_k}_\infty\le \varepsilon$.}
Here, we show how to recover $p_1,\dots,p_N$ from the memory $\cM$. Suppose that when the algorithm terminates, we have $\cM=\cbr{(w_l, a^*(w_l))}_{l\in[L]}$. We just solve for $(p_1,\cdots,p_N)$ that satisfies the following constraints, 
\begin{align}
    \inp[]{w_l}{p_{a^*(w_l)}-p_{a'}}&\ge c_{a^*(w_l)}-c_{a'}, \quad \forall a'\neq a^*(w_l), \quad\forall l\in[L], \label{cond:boundary}\\
    \inp[]{\ind}{p_i - p_j} &= 0, \qquad\qquad\quad\forall (i, j)\in[N]^2. \label{cond:regularity}
\end{align}
In the sequel, we write $\cS$ as the set of $(p_1,\dots, p_N)$ that satisfy Conditions \eqref{cond:boundary} and \eqref{cond:regularity}.
For Algorithm \ref{algo:hyperplane search} to work, we introduce the following assumption on the volume of $\cV_i$.
\begin{assumption}[Minimal Volume Ratio]\label{assum:volume}
Let $\Vol^{d}(\cV)$ denote the $d$-dimensional volume of set $\cV\in\RR^d$. We assume that there exists $\varsigma\in(0, 1]$ such that $\Vol^{d-1}(\cV_i)\ge \varsigma \cdot \Vol^{d-1}(\cP^{d-1})$ for any $i\in[N]$.
\end{assumption}
The minimal volume ratio assumption guarantees that all the sections $\cV_i$ are detectable via random sampling with high probability. We also make the following assumption on the cost difference.
\begin{assumption}[Minimal Cost difference]\label{assum:cost}
    We assume that $\inf_{1\le i< j\le N}\abr{c_i-c_j}\ge \theta$.
\end{assumption}
Specifically, we use the following definition of surface detection probability function.
\begin{definition}[Surface Detection Probability Function]\label{def:surface detect func}
Let $\Conv^{d-1}$ be the set of convex regions on some $(d-1)$-dimensional hyperplane such that $e\subset \cP^{d}$ for any $e\in\Conv^{d-1}$. Define function $\sigma_d:[0, 1]\rightarrow [0,1]$ as the pointwise maximum such that,
\begin{align*}
    \PP(\ell\cap e\neq \emptyset)\ge \sigma_d\rbr{\frac{\Vol^{d-1}(e)}{\Vol^{d-1}(\cP^{d-1})}}, \quad \forall e\in\Conv^{d-1}.
\end{align*}
\end{definition}
Note that $\sigma_d$ is a property inherent to the $d$-dimensional probability simplex. We argue that $\sigma_d$ can be roughly viewed as a linear function for small $e$. 
To characterize the searching result $\cS$ of Algorithm \ref{algo:hyperplane search}, we present the following Lemma.
\begin{lemma}\label{lem:hyperplane search}
    Under Assumptions \ref{assum:volume}, \ref{assum:cost}, suppose that $\varepsilon, c_d$ is chosen to satisfy
    \begin{align*}
        \xi_{d}^2&\defeq {\frac{c_d^2}{d^2}-\frac{d\varepsilon^2}{8}} >0, \nend
        \tau_d&\defeq \rbr{ \frac{\varsigma^2}{3d}}^d - d^2 \rbr{1+\frac{4 }{d\varsigma}}\cdot (c_d+\varepsilon) > 0.
    \end{align*}
    After $T$ samples and no more than $\cO(TNd^2\log(1/\varepsilon))$ rounds, with probability at least $1-Ne^{-T\sigma_d(\tau_d)}$,  we have for any $(p_1,\dots, p_N)\in \cS$ that
    \begin{align*}
        \nbr{(p_i-p_j) - (p_i^*-p_j^*)}_2\le  \frac{2(N-1)\sqrt{d}\cdot \varepsilon }{\xi_d^{d-1} \theta}. 
    \end{align*}
\end{lemma}
To construct an efficient learning procedure, 
we need to determine the optimal value for $\varepsilon, c_d$ such that the total round number is minimized while the learning error is controlled by $\varepsilon$.
\begin{corollary}
\label{cor:hyperplane search}
By properly setting $c_d$ and $\varepsilon$ and running the simplex searching algorithm for $t$ rounds,
we guarantee the learning error less than $\varepsilon$ with probability at least 
$$1-N\exp\rbr{-\frac{t\cdot\sigma_d(\tau_d)}{Nd^4\log(N\varsigma^{-2}\varepsilon^{-1}\theta^{-1})}},$$ where $\tau_d=(\varsigma^2/6d)^2$ is a constant. 
\begin{proof}
    Define constant
    \begin{align*}
        \Upsilon \defeq  \frac{\varsigma^{2d+1}}{2\cdot 3^d d^{d+1}(d\varsigma + 4)} <1.
    \end{align*}
    We let $c_d = \Upsilon/2$. Then it suffices for the first condition to hold if $\varepsilon \le \Upsilon/d^{3/2}$. Moreover, we have $\xi_d\ge \Upsilon/\sqrt 8 d$ and the second condition holds automatically with $\tau_d\ge (\varsigma^2/6d)^2$.
    Therefore, the constraint for $\varepsilon$ becomes, 
    \begin{align*}
        \varepsilon \le \min\cbr{\frac{\Upsilon}{d^{3/2}}, 
        \frac{\varepsilon\theta}{2N\sqrt d} \rbr{\frac{\Upsilon}{\sqrt 8 d}}^{d-1} }. 
    \end{align*}
    We can take equality for the optimal $\varepsilon$.
    Obviously, the second term dominates, and we thus have the total rounds bounded by 
    \begin{align*}
        \text{total round}&\le \cO\rbr{ TNd^2 \rbr{\log\rbr{\frac{2N\sqrt d}{\varepsilon \theta}} + d\log \rbr{\frac{\sqrt 8 d}{\Upsilon}} }}\nend
        &=\cO\rbr{ TNd^2 \rbr{\log\rbr{2N\sqrt d(\varepsilon \theta)^{-1}} + d^2\log \rbr{d \varsigma^{-2}} }}, 
    \end{align*}
    where the failure probability is bounded by $Ne^{-T\sigma_d(\tau_d)}$ with $\tau_d=(\varsigma^2/6d)^2$ being a constant. 
\end{proof}
\end{corollary}

\begin{proof}[Proof of \Cref{lem:hyperplane search}]
We consider an undirected graph $\cG=(\cA, E)$ with the node set $\cA$ and the edge set $E=\cbr{e_{ij}\given e_{ij}=\cV_i\cap\cV_j, \forall i\neq j}$. Define event $\cE_{ij}^d$ as follows.

\begin{definition}[Surface Detection Event]
We say that event $\cE_{ij}^d$ happens if there exists $t\in[T]$ and we have successfully searched for a $k_t$ such that for the $d$-dimensional simplex $\SSS^d$ placed around $(x_{k_t}+y_{k_t})/2$ in Algorithm \ref{algo:hyperplane search} and any $x\in\SSS^d$, the best response at $x$ satisfies $a^*(x)\in\{i,j\}$. 
\end{definition}

Simply put, the event $\cE_{ij}^d$ guarantees that $\SSS^d$ only contains two possible actions $\{i,j\}$ and $e_{ij}$ can therefore be successfully learned via binary searching for the intersects of the edges of the simplex with $e_{ij}$.
The following proposition backs up our statement.
\begin{proposition}[Intersection Geometry]\label{prop:intersection}
Under event $\cE_{ij}^d$ and condition $\varepsilon/2 < c_d/(d\sqrt{d+1})$, let $\SSS$ denote the simplex corresponding to $\cE_{ij}^d$. Then $\SSS^d\cap e_{ij}$ contains a $(d-1)$-dimensional ball with radius at least 
\begin{align*}
     \xi_{d}\defeq \sqrt{\frac{c_d^2}{d(d+1)}-\frac{d\varepsilon^2}{4}}.
\end{align*}
\end{proposition}
\begin{proof}
Under event $\cE_{ij}^d$, we claim that the hyperplane $e_{ij}$ must intersect with the $\sqrt d\varepsilon/2$-ball $\BB_1=\BB(\sqrt d\varepsilon/2)$ centered at $(x_{k_t}+y_{k_t})/2$, since $e_{ij}$ passes through $\overbar{x_{k_t} y_{k_t}}$, which lies inside $\BB_1$. 
Moreover, we consider a ball $\BB_2=\BB(c_d/\sqrt{d(d+1)})$ also centered at $(x_{k_t}+y_{k_t})/2$.
Since the smallest distance from the center of a $d$-dimensional simplex with length $\sqrt{2} c_d$ to any of its surface is $c_d/\sqrt{d(d+1)}$, we have $\BB_1\subset \BB_2\subset\SSS^d$ under the condition $\varepsilon/2 < c_d/(d\sqrt{d+1})$. In addition, $\BB_2$ is also the largest ball contained in $\SSS^d$. 
We therefore conclude that the intersection area satisfies,
$$\BB_2\cap e_{ij}\subseteq \SSS^d\cap e_{ij}.$$ 
Since $\SSS^d$ only intersects with hyperplane $e_{ij}$, we have $\BB_2\cap e_{ij}\subseteq \SSS^d\cap e_{ij}\subseteq e_{ij}$.
Thus, the left-hand side corresponds to the intersection area of a $d$-dimensional ball and a $(d-1)$-dimensional hyperplane. Since $e_ij$ intersects with $\BB_1$, thus a $(d-1)$-dimensional ball with radius at least 
\begin{align*}
     \xi_{d}=\sqrt{\frac{c_d^2}{d(d+1)}-\frac{d\varepsilon^2}{4}}, 
\end{align*}
should be contained in $\SSS^d\cap e_{ij}\subseteq e_{ij}$.
\end{proof}
Proposition \ref{prop:intersection} characterizes the geometry of the intersection area $\SSS^d\cap e_{ij}$ under event $\cE_{ij}^d$. Specifically, the intersection area should contain a $(d-1)$-dimensional ball with radius lower bounded by $\xi_d$.
Such a geometry is critical for solving the hyperplanes to small errors.
We consider the following constraints,
\begin{gather}
    \inp[]{u_k}{p_i-p_j}\ge c_i-c_j, \quad \inp[]{w_k}{p_i-p_j}\le c_i-c_j, \quad\forall k\in[K], \nend
    \inp[]{\ind}{p_i-p_j}=0, \label{cond:local}
\end{gather}
where $(u_k, w_k)$ is the binary searching result on the edges of simplex $\SSS^d$. Under event $\cE_{ij}^d$, let $\cS_{ij}$ denote the set of $(p_i, p_j)$ satisfying these constraints.
Apparently, $\cS_{ij}$ is a relaxation of $\cS$. The following proposition characterizes the searching errors in $\cS_{ij}$ under the event $\cE_{ij}^d$ in terms of $p_i-p_j$.
\begin{proposition}[Local Hyperplane Learnability]\label{prop:learnability}
    Under Assumption \ref{assum:cost} and event $\cE_{ij}^d$, for any $(p_i, p_j)\in\cS_{ij}$, we have
    \begin{align*}
        \nbr{(p_i-p_j)-(p_i^*-p_j^*)}_2\le \varphi_d^{-1} \sqrt d \varepsilon,  
    \end{align*}
    where 
    \begin{align*}
        \varphi_d\defeq\rbr{\frac{\xi_d}{\sqrt{1-d^{-1}}}}^{d-1} \cdot \frac{\theta}{\sqrt 2}\cdot \sqrt{\frac{d}{d+1}}.
    \end{align*}
\end{proposition}
\begin{proof}
    We first relax $\cS_{ij}$ to 
    \begin{gather*}
        \check\cS_{ij}: (p_i, p_j)\quad\st \quad \inp[]{s_k}{p_i-p_j}=c_i-c_j + \epsilon_k, \quad \inp[]{\ind}{p_i-p_j}=0, \quad, |\epsilon_k|\le \varepsilon, \quad\forall k\in[K].
    \end{gather*}
    Note that $s_1,\dots,s_K$ corresponds to the vertices of $\SSS^d\cap e_{ij}$. Using Proposition \ref{prop:intersection}, we pick $\tilde s_1,\dots,\tilde s_{d}$ on the ball $\BB_2\subset \SSS^d\cap e_{ij}$ that form a $(d-1)$-dimensional simplex.
    Note that $\SSS^d\cap e_{ij}$ must be convex. Thus, we can express $\tilde s_1,\dots,\tilde s_{d}$ as convex combinations of $s_1,\dots,s_K$. 
    In the matrix form, we suppose
    \begin{align*}
        \begin{bmatrix} \tilde s_1^\top\\ \cdots\\ \tilde s_d^\top \end{bmatrix} = Q\begin{bmatrix} s_1^\top\\ \cdots\\  s_K^\top \end{bmatrix}, 
    \end{align*}
    where $Q$ has row sums equal to $1$. Using $\tilde s_1,\dots,\tilde s_d$, we can further relax $\check \cS_{ij}$ by multiplying $\mathrm{diag}(Q, 1)$ to the constraints,
    \begin{align*}
        \tilde \cS_{ij}:(p_i,p_j) \quad \st\quad \underbrace{\begin{bmatrix} \tilde s_1^\top\\ \cdots\\ \tilde s_d^\top\\ \ind_{d+1}^\top /(d+1)\end{bmatrix}}_{\tilde S} (p_i-p_j)=\begin{bmatrix}
        (c_i-c_j)\ind_d\\ 0
        \end{bmatrix} + \underbrace{\begin{bmatrix}
        \epsilon\\ 0
        \end{bmatrix}}_{\tilde\epsilon} , \quad \nbr{\epsilon}_\infty \le \varepsilon.
    \end{align*}
    Note that the first $d$ rows of $S$ form a $(d-1)$-dimensional simplex on $e_{ij}$. Moreover, since $|c_i-c_j|>\theta$, we note that $\ind_{d+1}$ does not lie on $e_{ij}$. To guarantee linear independence of $\tilde S$, we aim to lower bound the distance from $\ind_{d+1}/(d+1)$ to $e_{ij}$. For any $x\in e_{ij}$, we have 
    $
        \inp[]{x}{p_i^*-p_j^*} = c_i-c_j
    $,
    and for $\ind_{d+1}/(d+1)$ we have $\inp[]{\ind_{d+1}/(d+1)}{p_i^*-p_j^*}=0$. Thus, we use the Cauchy-Schwartz inequality and obtain,
    \begin{align*}
        \nbr{x-\frac{\ind_{d+1}}{d+1}}_2\cdot \nbr{p_i^*-p_j^*}_2\ge\abr{\inp[\Big]{x-\frac{\ind_{d+1}}{d+1}}{p_i^*-p_j^*}} = \abr{c_i-c_j}\ge\theta.
    \end{align*}
    Since $\norm{p_i^*-p_j^*}_2\le \sqrt 2$, we conclude that for any $x\in e_{ij}$, 
    \begin{align*}
        \nbr{x-\frac{\ind_{d+1}}{d+1}}_2\ge \frac{\theta}{\sqrt 2}.
    \end{align*}
    Thus, the distance from $\ind_{d+1}/(d+1)$ to $e_{ij}$ is at least $\theta/\sqrt 2$, which indicates that the volume of the cube spammed by rows of $\tilde S$ is at least 
    \begin{align*}
        \Vol^{d}(\tilde S)&\ge \Vol^{d-1}(\cP^{d-1}) \cdot \rbr{\frac{\xi_d}{\sqrt{1-d^{-1}}}}^{d-1} \cdot \frac{\theta}{\sqrt 2 d}\nend
        &=\Vol^{d}(\cP^{d}) \cdot \rbr{\frac{\xi_d}{\sqrt{1-d^{-1}}}}^{d-1} \cdot \frac{\theta}{\sqrt 2}\cdot \sqrt{\frac{d}{d+1}}.
    \end{align*}
    Therefore, the determinant of $\tilde S$ satisfies
    \begin{align*}
        \det(\tilde S)\ge \rbr{\frac{\xi_d}{\sqrt{1-d^{-1}}}}^{d-1} \cdot \frac{\theta}{\sqrt 2}\cdot \sqrt{\frac{d}{d+1}} \eqdef \varphi_d.
    \end{align*}
    Note that $\tilde S$ has row sums equal to $1$, which implies that all the eigenvalues of $\tilde S$ should be no larger than $1$. Hence, the smallest eigenvalue of $\tilde S$ should be no less than $\varphi_d$ and consequently, the largest eigenvalue of $\tilde S^{-1}$ should be no more than $\varphi_d^{-1}$.
    Following the definition of $\tilde \cS_{ij}$, we have for any $(p_i, p_j)\in\tilde \cS_{ij}$ that
    \begin{align*}
        \nbr{(p_i-p_j) - (p_i^*-p_j^*)}_2 = \nbr{\tilde S^{-1}\tilde \epsilon}_2 \le \varphi_d^{-1} \nbr{\tilde\epsilon}_2 \le \varphi_d^{-1} \sqrt d \varepsilon.
    \end{align*}
\end{proof}

Proposition \ref{prop:learnability} bridges the geometrical argument to the learning errors in terms of $p_i-p_j$.
To obtain such benefits, we need to characterize under what conditions and with what probability the event $\cE_{ij}^d$ will occur. To start with, we study the volume of a surface that separates the probability simplex into two disjoint parts.
\begin{proposition}[Surface Volume]\label{prop:surface}
Suppose $A\subset\cP^{d}$ is a compact subset such that $\Vol^{d}(A)>0$ and $\Vol^{d}(A^c)>0$.
Let $E=A\cap A^c$ be the surfaces that separates $A$ and $A^c$ and we assume that $E$ comprises $J$ hyperplanes. It then follows that
\begin{align*}
    \frac{\Vol^{d-1}(E)}{\Vol^{d-1}(\cP^{d-1})}\ge J^{-(d-1)}\cdot \rbr{\frac{\Vol^d(A^c) \Vol^d(A)}{ \rbr{\Vol^d(\cP^d)}^2}\cdot \frac{2\sqrt{d+1}}{3d\sqrt{2d}}}^d.
\end{align*}
\end{proposition}
\begin{proof}
    For any $x\in\cP^{d}$, if $x\in A$, consider the following set
    \begin{align*}
        B(x)=\cbr{y\in\cP^{d}\biggiven \exists z\in A^c \text{ s.t. $x, y, z$ are on the same line $\ell$} }.
    \end{align*}
    Apparently, $A^c\subseteq B(x)$. Hence, we have $\Vol^d(B(x))\ge \Vol^d(A^c)$. Similarly, for $x\in A^c$, we can define $B(x)$ as 
    \begin{align*}
        B(x)=\cbr{y\in\cP^{d}\biggiven \exists z\in A \text{ s.t. $x, y, z$ are on the same line $\ell$}},
    \end{align*}
    and it follows that $\Vol^d(B(x))\ge \Vol^d(A)$.
    Following these observations, we have
    \begin{align}\label{eq:cL_B}
        \cL_B\defeq\int_{x\in \cP}\Vol^d(B(x)) \rd x \ge 2\Vol^d(A^c) \Vol^d(A).
    \end{align}
    Another important fact is that $\ell$ that connects $x, y, z$ must go through $E$ at least once since $z$ and $x$ belong to different areas.
    For any $(d-1)$-dimensional surface $S\subset \cP^d$ and any point $x\in\cP^d$, We define 
    \begin{align*}
        C(S, x) = \cbr{y\in\cP^d\biggiven \exists z\in S \text{ s.t. $x, y, z$ are on the same line}}.
    \end{align*}
    Apparently, we have $B(x)\subseteq C(E, x)$ and thus $\cL_B\le \cL_{C(E,\cdot)}$. Consider $S=s_1\cup s_2$, and it follows from the definition that $C(S, x)=C(s_1, x)\cup C(s_2, x)$, which suggests that $\Vol^d(C(S, x))\le \Vol^d(C(s_1, x)) + \Vol^d(C(s_2, x))$.
    Suppose that $E=E_1\cup E_2\cup\dots\cup E_J$, it then holds that
    \begin{align*}
        \cL_{C(E,\cdot)} = \int_{x\in\cP} \Vol^d(C(E, x)) \rd x \le \sum_{E_i}\underbrace{\int_{x\in\cP} \Vol^d(C(E_i, x)) \rd x}_{\eqdef R(E_i)}. 
    \end{align*}
    Consider $E_i$ is small enough such that $E_i$ lies on a hyperplane \footnote{Or approximating the surface $E$ with countably infinite hyperplanes.}. We draw hyperplanes $S_1^i\parallelsum S_2^i\parallelsum E_i$ such that the farthest two vertices of the simplex $\cP^d$ lie on $S_1^i$ and $S_2^i$. We denote the area between $S_1^i$ and $S_2^i$ by $\cQ^i$. We separate $\cQ^i$ into $\cQ^i_1$ and $\cQ^i_2$, where $\cQ^i_1$ is defined as the points that have a distance  larger than $\gamma_i\le \sqrt{2}/2$ to the hyperplane containing $E_i$ and $\cQ^i_2=\cQ^i\backslash\cQ^i_1$. 
    Consider the following definition for $x\in \cQ^i_1$,
    \begin{align*}
        \tilde C(E_i, x) =\cbr{y\in \cQ^i\biggiven \exists z\in E_i \text{ s.t. $x, y, z$ are on the same line}}, 
    \end{align*}
    and for $x\in\cQ^i_2$, 
    \begin{align*}
        \tilde C(E_i, x) =\cbr{y\in \cP^d\biggiven \exists z\in E_i \text{ s.t. $x, y, z$ are on the same line}}.
    \end{align*}
    Apparently, $C(E_i, x)\subseteq\tilde C(E_i, x)$ and we thus have
    \begin{align*}
        R(E_i) &\le \int_{x\in\cQ^i_1}\Vol^d(\tilde C(E_i, x))\rd x + \int_{x\in\cQ^i_2}\Vol^d(\tilde C(E_i, x))\rd x\nend
        &\le \Vol^d(\cQ^i_1) \cdot \Vol^{d-1} (E_i)\cdot \rbr{\frac{\sqrt 2}{\gamma_i}}^{d-1}\cdot \sqrt 2 + \Vol^d(\cQ^i_2)\cdot \Vol^d(\cP^d)\nend
        &\le \Vol^d(\cP^d)\cdot \rbr{
        (\sqrt 2)^d \Vol^{d-1} (E_i) \gamma_i^{-(d-1)}+ 2\Vol^{d-1}(\cP^{d-1} ) \cdot \gamma_i
        }, 
    \end{align*}
    where the second inequality holds from the definition of $\cQ_1^i$ and the fact that the largest hyperplane contained in $\tilde C(E_i, x)$ is no more than $(\sqrt{2}/\gamma_i)^{d-1}\Vol^{d-1}(E_i)$.
    Here, since $\gamma_i$ is adjustable, we plug in 
    \begin{align*}
        \gamma_i = \sqrt 2\cdot \rbr{\frac{\Vol^{d-1}(E_i)}{\Vol^{d-1}(\cP^{d-1})}}^{1/d}, 
    \end{align*}
    and obtain
    \begin{align*}
        R(E_i)\le 3\sqrt 2 \Vol^d(\cP^d)\cdot \rbr{\Vol^{d-1}(\cP^{d-1} )}^{(d-1)/d} \cdot\rbr{\Vol^{d-1}(E_i)}^{1/d}.
    \end{align*}
    By summing up $i\in[J]$ and using the Jensen's inequality, we obtain
    \begin{align}\label{eq:cL_C}
        \cL_{C(E,\cdot)}\le 3\sqrt 2 \Vol^d(\cP^d)\cdot \rbr{J\Vol^{d-1}(\cP^{d-1} )}^{(d-1)/d} \cdot\rbr{\Vol^{d-1}(E)}^{1/d}.
    \end{align}
    Combining \eqref{eq:cL_C} with \eqref{eq:cL_B} and using the inequality $\cL_{C(E,\cdot)}\ge \cL_B$, we obtain
    \begin{align*}
        2\Vol^d(A^c) \Vol^d(A)\le 3\sqrt 2 \Vol^d(\cP^d)\cdot \rbr{J\Vol^{d-1}(\cP^{d-1} )}^{(d-1)/d} \cdot\rbr{\Vol^{d-1}(E)}^{1/d}, 
    \end{align*}
    which further implies that
    \begin{align*}
        \Vol^{d-1}(E)&\ge \rbr{\frac{2\Vol^d(A^c) \Vol^d(A)}{3\sqrt 2 \Vol^d(\cP^d)\cdot \rbr{J\Vol^{d-1}(\cP^{d-1} )}^{(d-1)/d}}}^d\nend
        &= J^{-(d-1)}\cdot\Vol^{d-1}(\cP^{d-1} )\cdot
        \rbr{\frac{2\Vol^d(A^c) \Vol^d(A)}{3\sqrt 2 \rbr{\Vol^d(\cP^d)}^2}}^d\cdot \rbr{\frac{\sqrt{d+1}}{d\sqrt{d}}}^d.
    \end{align*}
\end{proof}
Following Proposition \ref{prop:surface}, a direct conclusion is that if $A$ contains $N_1$ action sections and $A^c$ contains $N_2$ action sections ($N=N_1+N_2$), there exists a surface $E_k$ such that 
\begin{align*}
    \frac{\Vol^{d-1}(E_k)}{\Vol^{d-1}(\cP^{d-1})}\ge J^{-d}\cdot \rbr{\frac{\Vol^d(A^c) \Vol^d(A)}{ \rbr{\Vol^d(\cP^d)}^2}\cdot \frac{2\sqrt{d+1}}{3d\sqrt{2d}}}^d \ge \rbr{ \frac{2\sqrt{d+1}\varsigma^2}{3d\sqrt{2d}}}^d,
\end{align*}
where the last inequality holds from Assumption \ref{assum:volume} and the fact that $J\le N_1N_2$. However, a sampled line $\ell$ passing through this $e_{ij}$ does not guarantee event $\cE_{ij}^{d}$ even though $\Vol^{d-1}(e_{ij})$ is bounded below. The following proposition states a sufficient condition for $\cE_{ij}^d$ to hold.
\begin{proposition}[Effective Surface]\label{prop:effective surface}
    Define the effective surface $\tilde e_{ij}$ as 
    \begin{align*}
        \tilde e_{ij} = \cbr{x\in e_{ij}\given \nbr{x-y}_2\ge \iota+h, \quad \forall y\in\partial e_{ij}}, 
    \end{align*}
    where $\iota = 2{\sqrt{2(d+1)} h}/{(d\sqrt{d}\varsigma)}$, $h=c_d+\varepsilon>c_d\sqrt{1-(d+1)^{-1}}+\varepsilon$.
    Event $\cE_{ij}^d$ holds for any searching line $\ell$ crossing $\tilde e_{ij}$ under Assumption \ref{assum:volume}. 
\end{proposition}
\begin{proof}
Under condition $\iota> 0$, we guarantee that the simplex does not intersect with $\partial e_{ij}$. However, such a condition is not sufficient if we want to guarantee that the simplex $\SSS^d$ only contains actions $i,j$ since a third action may occur outside $e_{ij}$.
In the sequel, we will show that it is impossible for a third action to appear if $\ell$ goes through $\tilde e_{ij}$.
Define $ e_{ij}'$ as,
\begin{align*}
    e_{ij}'=\cbr{x\in e_{ij}\given \exists y\in \tilde e_{ij}, \text{ s.t. } \nbr{y-x}_2\le h}.
\end{align*}
Apparently, for any $x\in e'_{ij}$ and $y\in \partial e_{ij}$, we have $\nbr{x-y}_2\ge \iota$.
Now, we aim to prove that the prism $A$ with base $e_{ij}'$ and height $h$ belongs to either $\cV_i$ or $\cV_j$. Suppose that $A$ lies on action $i$'s side and there exists a point in $A$ whose best response is not $i$. Then, there must be a hyperplane $e_{ik}$ passing through $A$ for some $k\notin\{i,j\}$. Suppose $z\in e_{ik}\cap A$.  Consider the set $E_{ik}\cap E_{ij}$, where $E_{ij}$ is the whole hyperplane containing $e_{ij}$. Obviously, for any $x\in e_{ij}'$ and $y\in E_{ik}\cap E_{ij}$, $\nbr{x-y}_2\ge \iota$ since $E_{ik}\cap E_{ij}$ should lie outside $e_{ij}$. In addition, since $\cV_i$ is a convex set, we note that $\cV_i$ must lie between $E_{ik}$ and $E_{ij}$. Therefore, the volume of $\cV_k$ should be no more than the volume of the area that lies between $E_{ik}$ and $E_{ij}$, 
\begin{align*}
    \Vol^{d}(\cV_i)\le \Vol^{d-1}(\cP^{d-1})\cdot \frac{\sqrt 2 h}{\iota} = \Vol^d(\cP^d) \cdot \frac{\sqrt{2(d+1)} h}{d\sqrt{d}\iota}. 
\end{align*}
By Assumption \ref{assum:volume}, we conclude that
\begin{align*}
    \iota \le \frac{\sqrt{2(d+1)} h}{d\sqrt{d}\varsigma},
\end{align*}
which conflicts with the condition $\iota > {\sqrt{2(d+1)} h}/{(d\sqrt{d}\varsigma)}$. Thus, we conclude that $A\subseteq \cV_i$. The same argument also applies to the conclusion $B\subseteq\cV_j$ if $B$ lies on action $j$'s side.
For any line $\ell$ passing through $\tilde e_{ij}$, let $x=\ell\cap \tilde e_{ij}$. Then we have $x\in A\cup B$ and the minimal distance from $x$ to $\partial(A\cup B)$ is at least $h>\varepsilon$. Thus, $x$ is guaranteed to be detected when doing the binary search on $\ell$. On the other hand, $h-\varepsilon > c_d\sqrt{1-(d+1)^{-1}}$, meaning that the simplex $\SSS^d$ placed at $(x_k+y_k)/2$ also lies within $A\cup B$. Thus, the simplex only contains two actions and event $\cE_{ij}^d$ follows.
\end{proof}
Proposition \ref{prop:effective surface} characterizes a sufficient condition for $\cE_{ij}^d$ to hold, i.e., the searching line $\ell$ passing through the effective surface $\tilde e_{ij}$. The next question is whether we can find enough effective surfaces via random sampling. To answer this question, we need to argue that there are sufficiently many effective surfaces with surface volume bounded below. It is also of the same importance to bound below the probability of $\ell$ passing through an effective surface with positive surface volume. Recall the definition of $\sigma_d$.
\begin{align*}
    \PP(\ell\cap e\neq \emptyset)\ge \sigma_d\rbr{\frac{\Vol^{d-1}(e)}{\Vol^{d-1}(\cP^{d-1})}}, \quad \forall e\in\Conv^{d-1}.
\end{align*}
Now, we study the problem that how many effective surfaces needs to be detected. From Proposition \ref{prop:learnability}, we see that if event $\cE_{ij}^d$ holds, we can estimate $p_i-p_j$ up to small errors. To learn all the pairwise outcome distribution difference, we just need to construct a tree in graph $\cG$ where an edge $e_{ij}$ is selected if $\cE_{ij}^d$ happens. To see this point, we invoke the following definition.
\begin{definition}[Connected Components]
    We say that action $i$ and $j$ belongs to the same connected component if there exists a path $i, k_1,\dots,k_n, j$ such that events $\cE_{i k_1}^d, \cE_{k_1 k_2}^d, \dots, \cE_{k_n j}^d$ happen.
\end{definition}
In the sequel, we use $C$ to denote a connected component. With a little abuse of notation, we also denote by $C$ the union of sections $\cV_k$ such that $k\in C$. Using Proposition \ref{prop:surface} and the following discussion, we take $C$ as $A$ and $C^c$ as $A^c$, and it follows that there exists a surface $e_{ij}$ on the boundary of $C$ and $C^c$ such that 
\begin{align*}
    \frac{\Vol^{d-1}(e_{ij})}{\Vol^{d-1}(\cP^{d-1})}\ge  \rbr{ \frac{2\sqrt{d+1}\varsigma^2}{3d\sqrt{2d}}}^d.
\end{align*}
Here, we assume that $i\in C$ and $j\in C^c$.
Note that we have $\Vol^{d-2}(\partial e)\le \Vol^{d-2}(\partial \cP^{d-1})$ for the sake that $e$ is convex. Recall the definition of the effective surface $\tilde e_{ij}$, which corresponds to shrinking $e_{ij}$ up to distance $\iota+h$. Therefore, the volume of $\tilde e_{ij}$ is at least
\begin{align*}
    \frac{\Vol^{d-1}(\tilde e_{ij})}{\Vol^{d-1}(\cP^{d-1})}
    &\ge \frac{\Vol^{d-1}(e_{ij}) - \Vol^{d-2}(\partial e) (h+\iota)}{\Vol^{d-1}(\cP^{d-1})}\nend
    &\ge \frac{\rbr{ \frac{2\sqrt{d+1}\varsigma^2}{3d\sqrt{2d}}}^d \Vol^{d-1}(\cP^{d-1}) - d \Vol^{d-2}(\cP^{d-2}) (h+\iota)}{\Vol^{d-1}(\cP^{d-1})}\nend
    &\ge \rbr{ \frac{2\sqrt{d+1}\varsigma^2}{3d\sqrt{2d}}}^d - {(d-1)\sqrt{(d-1)d}} \rbr{1+2\frac{\sqrt{2(d+1)} }{(d\sqrt{d}\varsigma)}}\cdot h\nend
    &\eqdef \tau_d.
\end{align*}
Therefore, with probability $\sigma_d(\tau_d)$ event $\cE_{ij}^d$ will happen in the next sample following Proposition \ref{prop:effective surface}, which also means that $C$ will expand. Since $C$ will expand for at most $N-1$ times, we have after $T=b\log(N-1)/\sigma_d(\tau_d)$ samples that $C=\cA$ with probability at least $1- (1/N)^{b-1}$.
Moreover, when $C=\cA$, the error for estimating $p_i-p_j$ is bounded by
\begin{align*}
    \nbr{(p_i-p_j)-(p_i^*-p_j^*)}_2\le (N-2)\varphi_d^{-1} \sqrt d \varepsilon.
\end{align*}
\end{proof}

%% file: proofs/mdp-proofs.tex
\paragraph{Visitation Measure.}
We define the state visitation measure $\rho^{\bpi}_h \in \Delta(\cS)$ at step $h$ induced by action policy $\pi$ as
$$ \rho^{\bpi}_h(s) := \Ex_{\bpi, \{P_h\}_{h=0}^{H}}\sbr{\ind(s_h=s)}$$
Given that $\rho_1^{\pi} = P_0 $, the visitation measure can be computed iteratively from $h=1$ to $H+1$ as,
$$ \rho^{\bpi}_{h+1}(s) = \inp[]{\rho_{h}^{\bpi}\otimes\pi_{h}}{P_{h}}_{\cS\times\cA} = \sum_{s'\in\cS}\sum_{a\in \cA}  \rho_h^{\bpi}(s') \pi_h(a | s')  P_{h}(s | s', a ).  $$
With this definition, we have for any $\bx \in \cX, \bpi \in \Pi$,
$$
U^{\bx,\bpi} 
= \sum_{h=1}^H \inp[]{\rho_{h}^{\bpi}\otimes\pi_{h}}{P_{h}\cdot x_{h} - c_{h}}_{\cS\times\cA} = \sum_{h=1}^H \sum_{s\in \cS}\sum_{a\in \cA} \rho_{h}^{\bpi}(s)\pi_{h}(a|s) [P_{h}(s,a)\cdot x_{h}(s) - c_{h}(s,a)],
$$
$$
V^{\bx,\bpi} 
= \sum_{h=1}^H \inp[]{\rho_{h}^{\bpi}\otimes\pi_{h}}{r_{h} - P_{h}\cdot x_{h}}_{\cS\times\cA} = \sum_{h=1}^H \sum_{s\in \cS}\sum_{a\in \cA} \rho_{h}^{\bpi}(s)\pi_{h}(a|s) [r_{h}(s,a) - P_{h}(s,a)\cdot x_{h}(s)].
$$

\paragraph{Parameter Estimation.}

Algorithm~\ref{algo:ucb-mdp} omits the details of how it estimates the parameters based on the observed trajectory $\{ (s_h^t, a_h^t, r_h^t) \}_{h\in [H]} $ in each episode. Specifically, at the end of $t$-th episode, given the trajectory $\{ (s_h^t, a_h^t, r_h^t) \}_{h\in [H]} $, it updates the counting variables for all $ h\in [H]$,
                \begin{align*}
                    N_h^{t+1}(s,a) &\gets N_h^{t}(s,a) + \one[ s^t_h=s, a^t_h=a], 
                    \\
                    N_h^{t+1}(s,a;s') & \gets N_h^{t}(s,a;s') + \one[ s^t_h=s, a^t_h=a, s^t_{h+1} =s'], \\
                    N_h^t(s) & = \sum_{a\in \cA} N_h^t(s, a),
                \end{align*} \\
and the empirical mean reward and bonus for all $h\in [H], s,s'\in\cS, a\in \cA$,
                \begin{align*}
                \hat{r}_{h}^t(s, a)  &\gets \frac{\sum_{\tau\in [t-1]} \one[ s^\tau_h=s, a^\tau_h=a] r^\tau_h(s,a) }{N^t_h(s,a)}, \\
                \hat{P}^t_h(s'|s,a) & \gets \frac{N^t_h(s,a;s')}{N^t_h(s,a)},\\
                b_h^t(s, a) &\gets (2H+2)\sqrt{ \frac{\ln(S A H T / \delta)}{N_h^t(s,a) } }.
                \end{align*}

For notational convenience, we will refer to the $t$-th episode as the $T_1 + t$-th episode after running the \oracle{} for $T_1$ rounds in the sequel.

\begin{lemma}[\citet{agarwal2019reinforcement}]
\label{lm:confidence-bound}
Fix $\delta \in (0,1)$. In any episode $t$, $\forall s \in \cS, a \in \cA, h \in [H]$, with probability at least $1 - \delta$, we have 
   $$ \abs{ \hat{r}^t_h(s, a ) - {r}(s, a ) } \leq 2\sqrt{\frac{\ln (S A H T / \delta)}{N_h^t(s, a)}}, \quad \norm{\hat{P}_h^t(s, a)-P_h(s, a) }_1 \leq 2\sqrt{\frac{\ln (S A H T / \delta)}{N_h^t(s, a)}}, $$
   $$ \abs{\left(\hat{P}_h^t(s, a)-P_h(s, a)\right)^{\top} f} \leq 8 H \sqrt{\frac{S \ln (S A H T / \delta)}{N_h^t(s, a)}}, \quad \forall f: \cS \to [0, H]. $$
\end{lemma}

We denote $\mathcal{E}_{\text{model}}$ as the event the inequalities in Lemma \ref{lm:confidence-bound} and \ref{lm:mdp-least-payment-oracle} holds.  To make our notation consistent, we let $R^{\bpi} := \Ex_{s \sim P_0} R_1^{\bpi}(s)$ and $C^{\bpi} := \Ex_{s \sim P_0} C_1^{\bpi}(s)$ such that $ V^{\bpi} = R^{\bpi} - C^{\bpi} - U^{\bpi} =  R^{\bpi} - \zeta^{\bpi}$.

Following the optimism principle, we determine the optimal action policy to explore from $\argmax_{\bpi} \hat{R}_h^{\bpi}(s) - \hat{\zeta}^{\bpi}_h(s) $ --- the optimism is preserved as long as the difference of $\hat{\zeta}_h^{\bpi}( s) - {\zeta}^{\bpi}_h( s)$ is relatively small. Here, with some carefully chosen amount of reward bonus $b_h$, $\hat{R}_h^{\bpi}(s)$ is the principal's expected reward under optimism for a given action policy $\bpi$ at $h$-th step with state $s$, i.e.,
$$  
\hat{R}_h^{\bpi}(s) = \min \{ H,\ \hat{r}_h(s, \pi_h(s) ) + b_h(s, \pi_h(s)) + \hat{P}_h(s,  \pi_h(s) ) \cdot   \hat{R}^{\bpi}_{h+1} \}.
$$
\begin{lemma}[Optimism]\label{lm:optimism}
If $\mathcal{E}_{\text{model}}$ is true, 
for any $\bpi$, in any episode $t$, $\hat{R}^{t,\bpi}(s) - R^{\bpi}(s) \geq 0  $.
\end{lemma}
\begin{proof}
    We prove by induction that $\hat{R}^{t,\bpi}(s) - R^{\bpi}(s) \geq 0 $. Start from the $(H+1)$-step, since $\hat{R}_{H+1}^{t,\bpi}(s) =  R_{H+1}^{\bpi}(s)=0,  \forall s\in \cS$, the base case holds. For the inductive case, given that $\hat{R}_{h+1}^{t,\bpi}(s) \geq  R_{h+1}^{\bpi}(s),  \forall s\in \cS$, we can derive the following inequality for any $s\in \cS$, let $a = \pi_h(s)) $
\begin{align*}
\hat{R}_{h}^{t,\bpi}(s) - R_h^{\bpi}(s) 
& = \hat{r}^t_h(s, \pi_h(s) )  + b^t_h(s, a) - r_h(s, a )  + \hat{P}^t_h(s,  a ) \cdot   \hat{R}^{t,\bpi}_{h+1} - P_h(s,  a ) \cdot  R^{\bpi}_{h+1} \\
& \geq \hat{r}^t_h(s, a ) + b^t_h(s, a) - r_h(s, a )  + \hat{P}^t_h(s,  a ) \cdot  R^{\bpi}_{h+1} - P_h(s,  a ) \cdot  R^{\bpi}_{h+1} \\
& = \hat{r}^t_h(s, a ) + b^t_h(s, a) - r_h(s, a )   +\left(\hat{P}_h^t(s, a)-P_h(s, a)\right) \cdot R^{\bpi}_{h+1} \\
& \geq b_h^t(s, a ) - (2 + 2 H) \sqrt{\frac{\ln (S A H T / \delta)}{N_h^t(s, a)}} \geq 0,
\end{align*}
where the first inequality is due to the given induction condition $\hat{R}_{h+1,t}^{\bpi}(s) \geq  R_{h+1}^{\bpi}(s),  \forall s\in \cS$ and the last inequality is due to Lemma \ref{lm:confidence-bound}. 
Therefore, the induction holds which concludes the proof.
\end{proof}

\begin{lemma}\label{lm:simulation}
If $\mathcal{E}_{\text{model}}$ is true, 
$\sum_{t=1}^{T} \hat{R}^{t,\bpi^t}(s) - R^{\bpi^t}(s) = O(H^2S \sqrt{AT \ln (S A H T / \delta) })  $ .
\end{lemma}
\begin{proof}
We apply the simulation lemma~\cite{agarwal2019reinforcement} to bound the difference term for each episode $t$ in $E_2$, 
\begin{align*}
 \hat{R}^{t,\bpi^t} - R^{\bpi^t} 
 & \leq  \sum_{h=1}^{H} \inp{{\rho}^{\bpi^t}_h\otimes \pi_h^t }{ b^t_h + \left(\hat{P}_h^t-P_h\right) \cdot \hat{R}^{t,\bpi^t}_{h+1}  }_{\cS\times \cA}     \\
& \leq  \sum_{h=1}^{H} \Ex_{s,a\sim \pi^t_h\otimes \pi_h^t } \big[ 10H \sqrt{\frac{S\ln (S A H  T / \delta)}{N_h^t(s, a )}} \big],
\end{align*}
where the last inequality follows from Lemma \ref{lm:confidence-bound}. 

The cumulative loss for all $T$ episodes can be bounded as,
\begin{align*}
 \sum_{t=1}^{T} \hat{R}^{t,\bpi^t} - R^{\bpi^t} 
 & =  10H \sqrt{S\ln (S A H  T / \delta)} \Ex \bigg[ \sum_{t=1}^{T} \sum_{h=1}^{H}  \sqrt{1/N_h^t(s, a ) } \bigg] \\
& \leq 20H^2 \sqrt{S\ln (S A H  T / \delta)} \sqrt{SAT} \\
& = O(H^2S \sqrt{AT \ln (S A H  T / \delta) }),
\end{align*}
where the first equality is by the linearity of expectation, and the inequality is by the fact $ \sum_{h=1}^N 1 / \sqrt{i} \leq 2 \sqrt{N} $.

\end{proof}

\paragraph{Regret Decomposition.} 
For the ease of analysis, we assume a common regularity condition in MDP, which ensures that the trajectory induced under any policy has enough randomness. This lower bound constant $\kappa$ is no more than $1/|\cS|$, and we expect this assumption to be relaxed in future work.

\begin{assumption}[MDP Mixing Condition]\label{assum:mixing-mdp}
There exists $\kappa > 0$ such that $P_h(s' | s,a) \geq \kappa, \forall s',s, a.$
\end{assumption}

We now present the proof of Theorem~\ref{thm:contractual-rl} via a decomposition of its regret into different components.

\begin{restatethm}[Full Statement of Theorem~\ref{thm:contractual-rl}]
In a contractual RL problem, with probability at least $1-\delta$,
Algorithm~\ref{algo:ucb-mdp} has $\tilde{O}\left((H^2SA^{-1/2} +  H^2\kappa^{-1/2} ) \sqrt{T \ln (S A H T / \delta) } \right)$ regret using the solver in Algorithm~\ref{algo:value-iteration}, and   $\tilde{O}\left((H^2SA^{-1/2} +  \eta \lambda_w \kappa^{-1/2} ) \sqrt{T \ln (S A H T / \delta) } \right)$ regret using the solver in Algorithm~\ref{algo:lp-solver}.
\end{restatethm}

\begin{proof}[Proof of Theorem~\ref{thm:contractual-rl}]

Consider the following decomposition of regret,
\begin{align*}
\Reg(T) & = \sum_{t=1}^{T} V^* - V^{\bx^t},  \\
& \leq O(T_1) + \sum_{t=1}^{T - T_1} V^* - V^{\bx^t} \\
& = O(T_1) + \sum_{t=1}^{T - T_1} 
[ V^* - \hat{V}^{t,\bpi^*} ]
+ [\hat{V}^{t,\bpi^*} - \hat{V}^{t,\bpi^t} ] 
+ [\hat{V}^{t,\bpi^t} - V^{\bx^t} ] \\
& \leq O(T_1) + \sum_{t=1}^{T - T_1} \hat{\zeta}^{t,\bpi^*} - \zeta^{\bpi^*} + \hat{R}^{t,\bpi^t} - R^{\bpi^t} + \abs{ {\zeta}^{\bpi^t}  - \hat{\zeta}^{t,\bpi^t} },
\end{align*}
where $\hat{V}^{t,\bpi^*} = \hat{R}^{t,\bpi^*} - \hat{\zeta}^{t,\bpi^*}$, $\hat{V}^{t,\bpi^t} = \hat{R}^{t,\bpi^t} - \hat{\zeta}^{t,\bpi^t} $ are the principal's estimated value under the least payment contract policy determined from Equation \eqref{eq:least-payment-policy} with $\{ \hat{P}^t_h, \hat{r}^t_h, b^t_h, \epsilon^t_h\}_{h\in [H]}$
and $V^{\bpi^t} = R^{\bpi^t} - \zeta^{\bpi^t} $ is the principal's exact value under the least payment contract policy determined from Equation \eqref{eq:least-payment-policy}. We now consider each of the three difference terms in the last inequality:
\begin{itemize}[leftmargin=*]
    \item The bound of first term is derived from the value decomposition, $V^* - \hat{V}^{t,\bpi^*} = R^{\bpi^*} - \zeta^{\bpi^*} - \hat{R}^{t,\bpi^*} + \hat{\zeta}^{t,\bpi^*} \leq  \hat{\zeta}^{t,\bpi^*} - \zeta^{\bpi^*} $, since the difference $R^{\bpi^*} - \hat{R}^{t,\bpi^*} \leq 0$, by Lemma \ref{lm:optimism}.
    \item The second term is non-positive, $\hat{V}^{t,\bpi^*} - \hat{V}^{t,\bpi^t}\leq 0$, since $\bpi^t = \argmax_{\bpi} \hat{R}^{t,\bpi} - \hat{\zeta}^{t,\bpi} $ is the optimal action policy to induce under the optimistic planning.
    \item The third term can be decomposed as, $\hat{V}^{t,\bpi^t} - V^{\bx^t} = \hat{R}^{t,\bpi^t} - \hat{U}^{t,\bpi^t} - C^{\bpi^t} - R^{\bpi^t} + U^{\bx^t} + C^{\bpi^t} \leq \hat{R}^{t,\bpi^t} - R^{\bpi^t} + \abs{ {\zeta}^{\bpi^t}  - \hat{\zeta}^{t,\bpi^t} }  $, since $U^{\bx^t} - \hat{U}^{t,\bpi^t} \eqsim \abs{ {\zeta}^{\bpi^t}  - \hat{\zeta}^{t,\bpi^t} }$ by either Lemma~\ref{lm:lp-solution-guarantee} or~\ref{lm:pess-contract-policy}.
\end{itemize}

Finally, if we adopt the solver in Algorithm~\ref{algo:lp-solver},
by Lemma \ref{lm:simulation} and \ref{lm:lp-solution-cost-bound}, the total regret is, 
\begin{align*} 
 \Reg(T) 
 &= O(T_1) + O(H^2S \sqrt{A(T-T_1) \ln (S A H T / \delta) }) + \sum_{\tau=1}^{T-T_1} O\left( H^2\kappa^{-1/2}\sqrt{\ln(S A H T / \delta)/t } \big\} \right)\\
 &= {O}\left( HAS^5\kappa_0^{-1}\log( \eta T \lambda_w^{-1} )+ H^2S \sqrt{AT \ln (S A H T / \delta) } + H^2\kappa^{-1/2} \sqrt{ T \ln (S A H  / \delta)} \right) \\
 &= \tilde{O}\left((H^2SA^{-1/2} + H^2\kappa^{-1/2} ) \sqrt{T \ln (S A H T / \delta) } \right),
 \end{align*}
where the first term $T_1 = \tilde{O}(\log T)$ can be dropped as being dominated by the second term on the order of $\sqrt{T-T_1}$. 

Finally, if we adopt the solver in Algorithm~\ref{algo:value-iteration}, 
by Lemma \ref{lm:simulation} and \ref{lm:cost-bound}, the total regret is, 
\begin{align*} 
 \Reg(T) 
 &= O(T_1) + O(H^2S \sqrt{A(T-T_1) \ln (S A H T / \delta) }) + \sum_{\tau=1}^{T-T_1} O\left( \min\big\{ H, \eta \lambda_s^{1-H}\kappa^{-1/2} \sqrt{\ln(S A H T / \delta)/T } \big\}  \right) \\
 &= {O}\left( HAS^5\kappa_0^{-1}\log( \eta T\lambda_s^{-1} )\log (1/ \delta) + H^2S \sqrt{AT \ln (S A H T / \delta) } + \eta \lambda_s^{1-H}\kappa^{-1/2} \sqrt{ T \ln (S A H  / \delta)} \right)\\
 &= \tilde{O}\left((H^2SA^{-1/2} + \eta \lambda_s^{1-H}\kappa^{-1/2} ) \sqrt{T \ln (S A H T / \delta) } \right),
 \end{align*}
where the first term $T_1 = \tilde{O}(\log T)$ can be dropped as being dominated by the second term on the order of $\sqrt{T-T_1}$.

\end{proof}

%% file: proofs/mdp-search.tex
\begin{algorithm}[tbh]
    \caption{\Oracle{} in Contractual RL with Far-sighted Agent}
     \label{algo:farsighted mdp learning oracle}
        \KwIn{State, action set $\cS, \cA$, number of steps $H$, cost function $\{ c_h \}_{h=1}^{H}$, search precision $\varepsilon$. }
        For each $s\in \cS, h\in [H]$, initialize a subroutine $\sA(s, h)$ described in Algorithm \ref{algo:hyperplane search}.\\
        Set $N_{s, h} \gets 0$ for all $s, h$.\\
        \For{$t=1,\dots, \chi(\varepsilon)$}{
                $(s^t, h^t)=\argmin_{s, h} N_{s, h}$.\\
                Set $x_{h}(s, \cdot)$ to $0$ for all $(h, s)\neq (h^t, s^t)$, 
                and set $x_{h^t}(s^t, \cdot)$ to the contract that $\sA(s^t, h^t)$ is going to use next.\\
                Set $x_{h^t}(s^t, \cdot) \gets x_{h^t}(s^t, \cdot)+ H \kappa_0^{-1} \max_{(s, a, h)\in\cS\times\cA\times[H]} c_h(s, a)$.\\
                Execute the contract policy $\bx$ and collect trajectory $\cT=\{ (s_h, a_h, r_h) \}_{h\in [H]}$.\\
            \If{$s_{h^t} = s^t$ in the trajectory $\cT$} 
            { 
                Let $\sA(s^t, h^t)$ step with $(s_{h^t}, a_{h^t}, s_{h^t+1})\in\cT$ and $N_{s^t, h^t}\leftarrow N_{s^t, h^t}+1$.\\
            }
            \Return $\mu_h(s,\cdot, \cdot)$ from $\sA(s^t, h^t)$ for each $s\in \cS, h\in [H]$.
        }
\end{algorithm}

The algorithm for constructing the \oracle{} is summarized in Algorithm \ref{algo:farsighted mdp learning oracle}.

\begin{assumption}[Weakly ergodic MDP]\label{assum:ergodic mdp}
There exists $\kappa_0 > 0$ such that $\max_{\bpi} \rho_h^{\bpi}(s) \ge 2\kappa_0, \forall s, h. $
\end{assumption}
Note that this assumption is weaker than the mixing MDP assumption where each state's visitation measure at each step is lower bounded under all the agent's policy, since the agent's policy $\pi$ chosen here is in favor of visiting state $s$ at step $h$. 

Also, recall that throughout this section, we assume that principal already knows the agent's cost function $\{ c_h \}_{h=1}^{H}$ and the contract design space is restricted to $\{ x_h: \cS\times \cS \to [0, \eta] \}$. Moreover, we inherit all the technical assumptions from the analysis in Appendix~\ref{sec:search-distribution-diff}.
With these assumptions, we show the following Lemma \ref{lm:mdp-least-payment-oracle} holds.

\begin{lemma} \label{lm:mdp-least-payment-oracle}
For any $\delta \in (0,1)$, Algorithm~\ref{algo:farsighted mdp learning oracle} guarantees an estimation of 
$\norm{\hat{\mu}_h - {\mu}_h }_{1,\infty}\le \varepsilon $ 
with probability at least $1-\delta$ in $\tilde{O}( HAS^5\kappa_0^{-1}\log(1/\varepsilon)\log (1/ \delta)  ) $ episodes.
\end{lemma}

\begin{proof}[Proof for \Cref{lm:mdp-least-payment-oracle}]

Since each time, only $ x_h(s, h)$ gives nonzero payment, and we add a constant $H \kappa_0^{-1}\max_{(s, a, h)\in\cS\times\cA\times[H]} c_h(s, a)$ that dominates the potential costs, the agent's optimal policy generates a visitation measure of $s$ at step $h$ larger than $\kappa_0$.
To illustrate this point, we take $\hat\bpi$ that maximizes $\rho_h^{\bpi}(s)$ while $\hat\pi_h=\argmax_{\pi_h}\inp[]{\pi_h}{P_h\cdot x_h}_{\cS\times\cA}$ at step $h$ (the choice of $\pi_h$ does not influence $\rho_h^{\bpi}(s)$).
For any $\bpi$ such that $\rho_h^{\bpi}(s)<
\kappa_0$, we have for the agent's expected profit difference expressed as
\begin{align*}
&\qquad U^{\hat{\bpi}} - U^{{\bpi}} \\
    &\quad = \sum_{h'=1}^H \inp[]{\rho_{h'}^{{\hat\bpi}}\otimes\hat\pi_{h'} - \rho_{h'}^{\bpi}\otimes\pi_{h'}}{P_{h'}\cdot x_{h'} - c_{h'}}_{\cS\times\cA} \nend
    &\quad \ge \inp[]{\rho_h^{{\tilde\bpi}}\otimes\pi_h(s, \cdot) - \rho_h^{\bpi}\otimes\pi_h(s, \cdot)}{P_h\cdot x_h(s, \cdot)}_{\cA} - \sum_{h'=1}^H \inp[]{\rho_{h'}^{\tilde\bpi}\otimes\tilde\pi_{h'} - \rho_{h'}^{\bpi}\otimes\pi_{h'}}{c_{h'}}_{\cS\times\cA}\nend
    & \quad > \kappa_0 H\kappa_0^{-1}\max_{(s, a, h)\in\cS\times\cA\times[H]} c_h(s, a) - H\max_{(s, a, h)\in\cS\times\cA\times[H]} c_h(s, a) =0, 
\end{align*}
where $\tilde \bpi$ is constructed by substituting $\hat\pi_h$ with $\pi_h$ in $\hat\bpi$.
Here, the first inequality holds since $\hat\pi_h$ is the optimal one-step policy concerning the payment. The last inequality holds by noting that $\tilde\pi$ still maximizes $\rho_h^{\bpi}(s)$ which gives $\rho_h^{\tilde\bpi}(s)\ge 2\kappa_0$ by Assumption \ref{assum:ergodic mdp}, and that $ x$ has a constant shift $H\kappa_0^{-1}\max_{(s, a, h)\in\cS\times\cA\times[H]} c_h(s, a)$. 
Note that the best response can only have higher profit for the agent. Thus, we conclude that any $\pi$ such that $\rho_h^{\bpi}(s)<\kappa_0$ cannot be the agent's best response. In other words, for the best response $\bpi^*$, we have $\rho_h^{\bpi^*}(s)\ge \kappa_0$.

Let $X_{s, h}$ be the random variable representing the total steps that $\sA(s, h)$ successfully takes and $Y$ be a random variable such that $Y\sim \text{Binomial}(n, \kappa_0)$, where $n=\chi(\varepsilon)$. It follows from our previous discussion that for any $\lambda\in(0, \kappa_0)$, 
\begin{align*}
    \PP\rbr{\sum_{s, h}X_{s, h}\le \lambda n}&\le \PP(Y\le \lambda n) \nend
    &\le \inf_{\theta\ge 0}\EE\sbr{\exp(-\theta (Y-\lambda n))}\nend
    &\le \exp\rbr{-n\kappa_0 f(\lambda/\kappa_0)}, 
\end{align*}
where $f(x)=1-x+x\log(x)$. By the algorithm procedure, we conclude that
\begin{align*}
    \PP\rbr{\min_{s, h} X_{s, h}\le \frac{\lambda n}{H|\cS|}-1} \le \exp(-n\kappa_0 f(\lambda/\kappa_0)).
\end{align*}
Combined with Corollary \ref{cor:hyperplane search}, we conclude that Algorithm \ref{algo:farsighted mdp learning oracle} has error 
$$\sup_{s, a, h}\norm{(\hat P_h(s, a)-\hat P_h(s, a') - (P_h(s, a)- P_h(s, a'))}_2\ge  \varepsilon$$ 
with probability no more than 
\begin{align*}
    H|\cS||\cA|\exp\rbr{-\frac{(\lambda n-H|\cS|) \cdot \sigma_d(\tau_d)}{H|\cA||\cS|^5\log(|\cA|\varsigma^{-2}\varepsilon^{-1}\theta^{-1})}} + \exp\rbr{-n\kappa_0 f(\lambda/\kappa_0)}, 
\end{align*}
for any $\lambda\in(0, \kappa_0)$. The first term dominates and we just plug in $\lambda = \kappa_0/2$ and conclude that the error is no more than $\varepsilon$ with probability at least
\begin{align*}
1- C_0H|\cS||\cA|\exp\rbr{-\frac{\kappa_0 \chi(\varepsilon) \cdot \sigma_\cS(\tau_\cS)}{H|\cA||\cS|^5\log(|\cA|\varsigma^{-2}\varepsilon^{-1}\theta^{-1})}}, 
\end{align*}
where $C_0$ is a constant and $\tau_\cS=(\varsigma^2/6|\cS|)^2$. 

Fix $\delta \in (0,1)$, and convert the error bound to $\ell_1$ norm, 
this algorithm guarantees that, with probability at least $1-\delta$,
$$\sup_{s, a, h}\norm{\hat{\mu}_h(s, a, a') - {\mu}_h(s, a, a') }_2 \leq \varepsilon,$$   
in ${O}( \frac{HAS^5\log(\varepsilon^{-1}S^{1/2})}{\kappa_0 \sigma_\cS( \tau_\cS )) }\log (C_0HSA/ \delta)  ) = \tilde{O}( HAS^5\kappa_0^{-1}\log(1/\varepsilon)\log (1/ \delta)  ) $ rounds, where we ignore the constant $\sigma_\cS( \tau_\cS )$ from Assumption~\ref{assum:volume}.

\end{proof}

\begin{lemma} \label{lm:improved-distribution-est}
If $\mathcal{E}_{\text{model}}$ is true, 
with 
$\norm{\hat{\mu}_h - {\mu}_h }_{1,\infty} \leq \varepsilon $, at any episode $t \in [T]$,
we have 
\begin{equation}\label{eq:confidence-bound-mixing}
 \norm{\hat{P}_h^t(s, a)-P_h(s, a) }_{1} \leq  2 \sqrt{\frac{\ln (S A H T / \delta)}{ N_h^t(s) } } + \varepsilon,   
\end{equation}
\end{lemma}
\begin{proof}
To improve the estimate of $\hat{P}$ with $\hat{\mu}$,
we can construct the estimator of $\hat{P}_h^t$ as,
\begin{align*}
 \hat{P}^t_h(s,a) 
& = \sum_{a'\in \cA} \frac{N^t_h(s,a')}{N_h^t(s)} [\hat{P}^t_h(s, a') + \hat{\mu}_h(s, a, a') ] \\
& = \frac{ \sum_{a'\in \cA} [ N^t_h(s,a';\cdot) +  {\mu}_h(s, a, a') N^t_h(s,a') ] }{N_h^t(s)} + \sum_{a'\in \cA} \frac{  N^t_h(s,a')   }{N_h^t(s)}  [ \hat{\mu}_h(s, a, a')  - {\mu}_h(s, a, a') ]
\end{align*}
The first part is an unbiased estimator for ${P}^t_h(s,a)$ and thus by Hoeffding's inequality, with probability at least $1-\delta$, for all $s,a,h,t$, with $N_h^t(s)$ samples,
$$
\bignorm{ \frac{ \sum_{a'\in \cA} [ N^t_h(s,a';\cdot) +  \hat{\mu}_h(s, a, a') N^t_h(s,a') ] }{N_h^t(s)} - {P}^t_h(s,a) }_1 \leq 2 \sqrt{\frac{\ln (S A H T / \delta)}{ N_h^t(s) } }.
$$
Using the triangle inequality and the fact that $\frac{  N^t_h(s,a')   }{N_h^t(s)} \leq 1$, we get the inequality in Eq \eqref{eq:confidence-bound-mixing}. 
\end{proof}

%% file: proofs/solver-LP.tex
\begin{algorithm}[tbh]
    \caption{Linear Programming under Uncertainty}
    \label{algo:lp-solver}
        \KwIn{ Parameters $\{ \hat{P}_h, \hat{r}_h, b_h, \hat{\mu}_h\}_{h\in [H]}$, robustness parameter $\epsilon$}
        \Output{ Contract policy $\bx$, action policy $\bpi$ }
        \For{$\bpi \in \Pi$}{ 
\begin{equation}\label{eq:robust-contract-policy}
\hat{U}^{\bpi},\hat{\bx}^{\bpi} = \minarg_{\bx} \hat{U}^{\bx, \bpi} \quad\text{s.t.}\quad   \hat{U}^{\bx, \bpi} - \hat{U}^{\bx, \bpi'} \geq  \epsilon, \forall \bpi' \neq \bpi,
\end{equation}
$$
\hat{V}^{\bpi} \gets \hat{R}^{\bpi} - {C}^{\bpi} - \hat{U}^{\bpi} 
$$
}
Set $\hat{\bpi}^* \gets \max_{\bpi\in \Pi} \hat{V}^{\bpi}, \hat{\bx}^* \gets  \hat{\bx}^{\hat{\bpi}^*} $ \\
\Return $\hat{\bx}^*, \hat{\bpi}^* $
\end{algorithm}

\begin{assumption}[Weak $\lambda$-Inducibility] \label{assum:inducibility-weak} 
For any stationary action policy $\bpi \in \Pi$,
there exists an event $e_h \in \{0,1\}^S$ for each $h$ such that 
$\sum_{h=2}^{H+1} (\rho_h^{\bpi} - \rho_h^{\bpi'}) \cdot e_h \geq \lambda_w, \forall \bpi\neq \bpi'.$ 
\end{assumption}

This assumption ensures that for any stationary action policy $\bpi \in \Pi$, there exists a set of bounded reward functions $\{r_h\}_{h=1}^{H}$ with $\max_{s\in \cS,a\in \cA, h\in [H]} \norm{r_h(s,a)}\leq 1$ such that $\bpi$ dominates any other stationary action policy $\bpi'$ with additional expected utility at least $\lambda_w$ in the MDP environment $(\cS, \cA, \{P_h, r_h\}_{h=1}^{H})$. For example, one can set $r_h(s,a) = \sum_{s\in\cS} e_{h+1}(s') P_{h}(s' | s,a)$. For similar reason, under this assumption, for any stationary action policy $\bpi \in \Pi$, there exists a contract policy $\bx$ with $\norm{x_h}_{\infty}\leq 1$ under which $U^{\bx,\bpi} - U^{\bx,\bpi'} \geq \lambda_w, \forall \bpi\neq \bpi'$, where $\pi_{h}(s,s')=e_{h+1}(s')$.

We describe the statistically efficient solver in Algorithm~\ref{algo:lp-solver}. Below we show that the contract policy it solves is robust and optimistic in a formal sense.

\begin{lemma}\label{lm:lp-solution-guarantee}
Under Assumption \ref{assum:inducibility-weak}, 
with estimation of $\hat{\mu}, \hat{P}$ such that
$\norm{\hat{\mu}_h - {\mu}_h }_{1,\infty} \leq \frac{\epsilon}{SH^2} $ and 
$\norm{\hat{P}_h - {P}_h }_{1,\infty} \leq \epsilon' $, the policy $\bpi, \hat{\bx}^{\bpi}$ solved from LP~\ref{eq:robust-contract-policy} satisfies the following condition,
$$
U^{ \hat{\bx}^{\bpi} } - U^{\bpi} \eqsim \abs{ \hat{\zeta}^{\bpi} - \zeta^{\bpi} } = O\left( \lambda_w^{-1}H\eta\epsilon + H^2\epsilon' \right).
$$
\end{lemma}
\begin{proof}

Pick any action policy $\bpi\in \Pi$, we first show that $\bpi$ is the agent's best response under $\hat{\bx}^{\bpi}$ solved from Equation~\ref{eq:robust-contract-policy}. That is, due to Lemma \ref{lm:robustness-error}, the following inequality holds, $\forall \bpi' \neq \bpi$,
$${U}^{\hat{\bx}^{\bpi}, \bpi} - {U}^{\hat{\bx}^{\bpi}, \bpi'} \geq \hat{U}^{\hat{\bx}^{\bpi}, \bpi} - \hat{U}^{\hat{\bx}^{\bpi}, \bpi'} - \epsilon \geq 0.$$

Let $\norm{\hat{\mu}_h - {\mu}_h }_{1,\infty} \leq \frac{\epsilon}{SH^2} = \varepsilon $.
We show that the payment under $\bx$ has bounded suboptimality,
\begin{equation}\label{eq:bounded-payment}
\hat{\zeta}^{\bpi} - {\zeta}^{\bpi}
= \hat{U}^{\hat{\bx}^{\bpi}, \bpi} - {U}^{\bpi}
\leq \hat{U}^{{\bx}^{\bpi}, \bpi} - {U}^{{\bx}^{\bpi}, \bpi} + \frac{2SH^3\eta \varepsilon}{\lambda_w} \leq \frac{2SH^3\eta \varepsilon}{\lambda_w} + H^2\epsilon' = \frac{H\eta\epsilon}{\lambda_w} + H^2\epsilon',
\end{equation}
where $\bx^{\bpi} = \argmin_{\bx} {U}^{\bx, \bpi} \quad\text{s.t.}\quad   {U}^{\bx, \bpi} - {U}^{\bx, \bpi'} \geq  0, \forall \bpi' \neq \bpi$ and thus $ {U}^{\bpi} = {U}^{\bx^{\bpi}, \bpi}$ by definition.

We first prove the first inequality in Equation~\eqref{eq:bounded-payment}. Let $\bx^0$ be the contract policy such that $ {U}^{\bx^0, \bpi} - {U}^{\bx^0, \bpi'} \geq \lambda_w,  \forall \bpi' \neq \bpi$ with ${U}^{{\bx}^0, \bpi} \leq H$ and $\hat{U}^{{\bx}^0, \bpi} \leq  H\eta$. Such $\bx^0$ exists by Assumption~\ref{assum:inducibility}. It is easy to verify the linearity here that for any $\bx', \bx''\in \bpi, c\in \RR$, let $\bx = \bx' + c\bx''$ such that $\bx_h(s,s')=\bx'_h(s,s')+c\bx''_h(s,s'), \forall h,s,s'$, then $U^{\bx,\bpi} = U^{\bx',\bpi} + cU^{\bx'',\bpi}$.
We claim that $\tilde{\bx} = \bx^{\bpi} + \frac{2SH^2 \varepsilon }{\lambda_w}\bx^0$ satisfies the following inequalities:
\begin{equation}\label{eq:induced-contract-policy}
 \hat{U}^{\tilde{\bx}, \bpi} - \hat{U}^{\tilde{\bx}, \bpi'} \geq {U}^{\tilde{\bx}, \bpi} - {U}^{\tilde{\bx}, \bpi'} - SH^2 \varepsilon \geq  \epsilon,    
\end{equation} 
where the first inequality is due to Lemma \ref{lm:robustness-error} and the second inequality is by the construction of $\tilde{\bx}$. 
$$
\hat{U}^{\hat{\bx}^{\bpi}, \bpi} 
\leq \hat{U}^{\tilde{\bx}, \bpi} 
= \hat{U}^{{\bx}^{\bpi}, \bpi} + \frac{2SH^2 \varepsilon}{\lambda_w} \hat{U}^{{\bx}^0, \bpi} 
\leq \hat{U}^{{\bx}^{\bpi}, \bpi} + \frac{2SH^3\eta \varepsilon}{\lambda_w}.
$$
For the first inequality, notice that $\tilde{\bx} $ is feasible in the LP~\eqref{eq:robust-contract-policy} according to Equation~\eqref{eq:induced-contract-policy} and its value should be no larger than the optimal value of LP~\eqref{eq:robust-contract-policy}. The equation is due to the linearity of $U^{\bx, \bpi}$ w.r.t. $\bx$. The last inequality uses the fact that $\hat{U}^{{\bx}^0, \bpi} \leq H$ by construction.

For the last inequality in Equation~\eqref{eq:bounded-payment},
we apply the simulation lemma~\cite{agarwal2019reinforcement},
\begin{align*}
    \hat{U}^{{\bx}^{\bpi}, \bpi} - {U}^{{\bx}^{\bpi}, \bpi} 
    & = \sum_{h=1}^H \inp[]{\rho_{h}^{{\bpi}}\otimes\pi_{h}}{ (\hat{P}_{h}-{P}_{h})\cdot \hat{U}_{h}^{{\bx}^{\bpi}, \bpi}}_{\cS\times\cA} \\
    & \leq H\sum_{h=1}^H \norm{\hat{P}_{h}(s,a)-{P}_{h}(s,a)}_1 \\
    & \leq H^2\epsilon'
\end{align*}

\end{proof}

\begin{lemma}\label{lm:robustness-error}
Given that $\norm{\hat{\mu}_h - {\mu}_h }_{1,\infty} \leq \varepsilon $,
    we have 
    $$
    \abs{ U^{\hat{\bx}^{\bpi}, \bpi} - U^{\hat{\bx}^{\bpi}, \bpi'}  - \hat U^{\hat{\bx}^{\bpi}, \bpi} + \hat U^{\hat{\bx}^{\bpi}, \bpi'}} \le SH^2 \varepsilon 
    $$
\end{lemma}
\begin{proof}
Notice that we can apply the performance difference lemma~\cite{agarwal2019reinforcement} on any ${U}^{\bx, \bpi}, {U}^{\bx, \bpi'}$ as they share the same $\bx$ and thus the same environment except for different policy,
\begin{align*}
    {U}^{\bx, \bpi} - {U}^{\bx, \bpi'} 
     = \sum_{h=1}^H \inp[]{\rho_{h}^{{\bpi}}\otimes(\pi_{h} -\pi'_{h}) }{ {P}_{h}\cdot U_{h}}_{\cS\times\cA} 
     = \sum_{h=1}^H  \sum_{s\in \cS} \rho_{h}^{{\bpi}}(s)\mu_{h}(s, \pi_h(s), \pi_h'(s))\cdot U_{h}
\end{align*}
Hence, we have 
\begin{align*}
&\quad \abs{ U^{\hat{\bx}^{\bpi}, \bpi} - U^{\hat{\bx}^{\bpi}, \bpi'}  - \hat U^{\hat{\bx}^{\bpi}, \bpi} + \hat U^{\hat{\bx}^{\bpi}, \bpi'} } \\
    & = \abs{ \sum_{h=1}^H  \sum_{s\in \cS} \left[ \rho_{h}^{{\bpi}}(s)\mu_{h}(s, \pi_h(s), \pi_h'(s))\cdot U_{h} -  \hat{\rho}_{h}^{{\bpi}}(s)\hat{\mu}_{h}(s, \pi_h(s), \pi_h'(s))\cdot \hat{U}_{h} \right] } \\
    & \leq  H \sum_{h=1}^H  \sum_{s\in \cS} \norm{ \mu_{h}(s, \pi_h(s), \pi_h'(s)) - \hat{\mu}_{h}(s, \pi_h(s), \pi_h'(s)) }_1 \\
    & \leq SH^2 \max_{s, a, h}\norm{\hat{\mu}_h(s, a, a') - {\mu}_h(s, a, a') }_1 \\
    & \le SH^2 \varepsilon  
\end{align*}
\end{proof}

\begin{lemma}  \label{lm:lp-solution-cost-bound}
If $\mathcal{E}_{\text{model}}$ is true, 
with $T_1 = O( HAS^5\kappa_0^{-1}\log(\eta T \lambda_w^{-1} )) $, under Assumption \ref{assum:mixing-mdp} and \ref{assum:inducibility-weak}, at any episode $t \in [T]$, we have
$$
\abs{ \hat{\zeta}^{t,\bpi}- \zeta^{\bpi} } = O\left( H^2\kappa^{-1/2}\sqrt{\ln(S A H T / \delta)/t } \big\} \right).
$$
\end{lemma}
\begin{proof}
This proof is to find a good trade-off between the error bound of $P$ and $\mu$ according to the ratio given in Lemma \ref{lm:lp-solution-guarantee}. 
For this analysis, let $\sup_{s, a,a',h}\norm{ \hat{\mu}^t_h(s, a, a') - {\mu}_h(s, a, a') }_{1} \leq \varepsilon$ such that we can apply Lemma~\ref{lm:improved-distribution-est} and get
$$
\norm{\hat{P}_h^t(s, a)-P_h(s, a) }_{1} \leq  2 \sqrt{\frac{\ln (S A H T / \delta)}{ 2\kappa t } } +  \varepsilon,
$$ 
where we use the fact $N_h^t(s) \geq 2\kappa t$ under Assumption \ref{assum:mixing-mdp}.
By Lemma \ref{lm:lp-solution-guarantee}, we can construct contracts to induce any action policy $\bpi$ such that 
$$
\abs{ \hat{\zeta}^{t,\bpi} - \zeta^{\bpi} } = O\left( \lambda_w^{-1}SH^3\eta\varepsilon + H^2(2 \sqrt{\frac{\ln (S A H T / \delta)}{ 2\kappa t } } +  \varepsilon) \right) = O\left( \lambda_w^{-1}SH^3\eta\varepsilon + H^2\sqrt{\frac{\ln (S A H T / \delta)}{ 2\kappa t } }  \right).
$$ 
Then, we can set $\varepsilon= \frac{\lambda_w}{SH\eta\sqrt{t} } $ such that the second term dominates the regret bound, and we have
\begin{align*}
\abs{ \hat{\zeta}^{t,\bpi} - \zeta^{\bpi} } 
&= O\left( H^2\kappa^{-1/2}\sqrt{\ln(S A H T / \delta)/t } \big\} \right) 
\end{align*}
It remains to ensure that in $T_1$ rounds, we can ensure $\varepsilon=\frac{\lambda_w}{SH\eta\sqrt{T-T_1} }$.
By Lemma \ref{lm:mdp-least-payment-oracle}, we know the sample complexity of $\varepsilon$ is on the order of $\tilde{O}( HAS^5\kappa_0^{-1}\log( 1/\varepsilon )  )$. Hence, $T_1 = O( HAS^5\kappa_0^{-1}\log( SH\eta T \lambda_w^{-1} ) ) = O( HAS^5\kappa_0^{-1}\log( \eta T \lambda_w^{-1} ) )$
\end{proof}

%% file: proofs/solver-VI.tex
\begin{algorithm}[tbh]
    \caption{Value-Iteration under Uncertainty}
    \label{algo:value-iteration}
        \KwIn{ Parameters $\{ \hat{P}_h, \hat{r}_h, b_h, \hat{\mu}_h,\epsilon_h \}_{h\in [H]}$, robustness parameter $\{ \epsilon_h \}_{h\in [H]}$ }
        \Output{ Contract policy $\bx$, action policy $\bpi$ }
        Set $ \hat{V}_{H+1}(s), \hat{U}_{H+1}(s) \gets 0, \forall a\in \cA, s\in \cS$. \\
        \For{$h = H\dots 1$}{ 
\begin{equation}\label{eq:least-payment-policy-est}
\begin{aligned}
    \hat{x}_h(s; a) &= \argmin_{x: \cS \to \RR_{+} }  \{ \hat{P}_h(s, a) \cdot x \ |\  \hat{\mu}_h(s, a, a') \cdot [x +  \hat{U}_{h+1}] \geq  {c}_h(s, a) - {c}_h(s, a') + \epsilon_h, \forall a'\neq a \},   \\
    \hat{Q}_h(s, a) &= \min \big\{H,\ \hat{r}_h(s, a) + b_h(s,a) + \hat{P}_h(s, a) \cdot \hat{V}_{h+1} \big\} - \hat{P}_h(s, a) \cdot \hat{x}_h(s; a),  \\
    \hat{V}_h(s), \pi_h(s) &= \maxarg_{a \in \cA} \hat{Q}_h(s, a),  \quad \hat{x}_h(s) =  \hat{x}_h(s; \pi_h(s)) \\
    \hat{U}_h(s) &= \min\big\{H,\ \hat{P}_h(s, a) \cdot [\hat{x}_h(s) 
+ \hat{U}_{h+1} ] - c_h(s,a) ) \big\}.
\end{aligned}
\end{equation}
}
\Return $\hat{\bx} = \{ \hat{x}_h \}_{h\in [H]}, \bpi = \{\pi_h \}_{h\in [H]} $
\end{algorithm}

\begin{assumption}[Strong $\lambda$-Inducibility] \label{assum:inducibility-strong} 
For any $a\in \cA, s\in \cS, h\in [H]$, there exists an event $e \in \{0,1\}^S$ such that 
$[P_h(s, a) - P_h(s, a')] \cdot e \geq \lambda_s, \forall a\neq a'.$ 
\end{assumption}

This assumption ensures that at any state $s$ of any step $h$, for any cost function $c_h$ and any action $a$, there exists a contract $x_h$ to induce $a$. Intuitively, this assumption asks that each action is dominantly capable of inducing a set of outcomes over others. 
It is different from assuming that the outcome distribution of each action is distinguishable from other. In particular, notice that even if we have $\min_{a'\neq a}\norm{P_h(s, a) - P_h(s, a')} \geq \lambda_s$, there is no guarantee that $\exists x \text{ s.t. } x\cdot [P_h(s, a) - P_h(s, a')] \geq 0, \forall a\neq a'$, i.e., there exists contract $\bx$ under which $a$ is the optimal action for the agent at this step regardless of the cost.
Notice that when $H=1$, Assumption \ref{assum:inducibility-weak} and \ref{assum:inducibility-strong} are equivalent. In general, Assumption \ref{assum:inducibility-strong} is stronger than Assumption \ref{assum:inducibility-weak}.

We describe the computationally efficient solver in Algorithm~\ref{algo:value-iteration}. Below we show that the contract policy it solves is robust and optimistic in a formal sense.

\begin{lemma}\label{lm:pess-contract-policy}
Under Assumption \ref{assum:inducibility-strong}, with estimation of $\hat{\mu}, \hat{P}$ such that
$\norm{\hat{\mu}_h - {\mu}_h}_{1,\infty} \leq \frac{\epsilon}{H-h+\eta} $, 
$\norm{\hat{P}_h- {P}_h }_{1,\infty} \leq \epsilon'/\eta $ and choice of $\epsilon_H=\epsilon, \epsilon_h = O(\min\{H, \epsilon \lambda_s^{h-H} + \epsilon' \lambda_s^{h+1-H} \} ), \forall h<H$, the policy $\bpi, \hat{\bx}^{\bpi}$ output from Algorithm~\ref{algo:value-iteration} satisfies the following condition,
$$
U^{ \hat{\bx}^{\bpi} } - U^{\bpi} \eqsim \abs{ \hat{\zeta}^{\bpi} - \zeta^{\bpi} } = O\left(\min\{H,  \epsilon \lambda_s^{-H-1} + \epsilon' \lambda_s^{-H} \} \right).
$$
\end{lemma}
\begin{proof}
Fix any action policy $\bpi$, we compare the contract policy ${\bx}^{\bpi}$ and $\hat{\bx}^{\bpi}$ solved respectively from Equation \eqref{eq:least-payment-policy} and \eqref{eq:least-payment-policy-est}, i.e., using the ground-truth and estimated parameters. In the remainder of the proof, we will save the superscript $\bpi$ in both $\hat{\bx}^{\bpi}$ and ${\bx}^{\bpi}$ for simplicity.

For the simplicity of notations, we prove the following claim. For any $h\in [H]$,
suppose we have $ \norm{ \hat{U}_{h+1}^{\bpi} - U_{h+1}  }_{\infty} \leq \alpha $, $\norm{\hat{P}_h - {P}_h }_{1,\infty} \leq \epsilon'/\eta $, $\norm{\hat{\mu}_h - {\mu}_h }_{1,\infty} \leq \beta $.
Then, with the choice of $\epsilon_h = (H-h+\eta)\beta + 2\alpha $,  the solution of Equation \eqref{eq:least-payment-policy-est} in the $h$-th step satisfies the following conditions: for every $s\in \cS$,
\begin{enumerate}
    \item $\hat{x}_h(s)$ induces agent's action $\pi_h(s)$.
    \item The expected payment of $\hat{x}_h(s)$ is bounded as 
$ {U}^{\hat{x}_h(s)}_h(s) - U_h^{\bpi}(s)  \leq \min\{ H,
\frac{\epsilon_h}{\lambda_s} + \epsilon' \}. 
$
    \item The estimated payment of $\hat{x}_h(s)$ (as well as the estimated agent value) is bounded as $ \abs{ \hat{\zeta}^{\bpi}_{h}(s) - {\zeta}^{\bpi}_{h}(s) } = \abs{ \hat{U}^{\bpi}_h(s) - U^{\bpi}_h(s) } \leq \min\{ H, 
        \frac{\epsilon_h}{\lambda_s} + 2\epsilon' \}.
    $
\end{enumerate}

We pick any state $s$ and let $a = \pi_h(s)$. Equation \eqref{eq:least-payment-policy-est} solves for $\hat{x}_h(s)$ as the solution to the following optimization program,
\begin{lp}\label{lp:pessimistic-contract-inductive}
\mini{ \hat{P}_h(s, a)\cdot x  }
\st 
\qcon{ \hat{\mu}_h(s, a, a') \cdot [x +  \hat{U}_{h+1}^{\bpi}] \geq  {c}_h(s, a) - {c}_h(s, a') + \epsilon_h }{a' \neq a}
\qcon{x(s,s') \geq 0}{s'\in \cS} 
\end{lp} 
We first argue that with $\epsilon_h = (H-h+\eta)\beta + \alpha $,
$\hat{x}_h(s)$ induces the agent to play action $a$ at state $s$. That is because the following inequality must hold, $\forall a' \neq a$,
\begin{align*}
& \quad  {\mu}_h(s, a, a') \cdot [\hat{x}_h(s) +  U_{h+1}] \\
& \geq   \hat{\mu}_h(s, a, a') \cdot [\hat{x}_h(s) + \hat{U}_{h+1}]  -  \beta \norm{\hat{x}_h(s) }_{\infty} - \beta\norm{U_h}_{\infty} - \alpha \norm{{\mu}_h(s, a, a')}_{1} + \beta \alpha \\
& \geq  c(s, a) - c(s, a') + \epsilon_h  - \beta\eta - \beta(H-h) - 2\alpha \\
& \geq  c(s, a) - c(s, a').   
\end{align*} 
To see that LP \eqref{lp:pessimistic-contract-inductive} must have a feasible solution, let us consider the contract ${x}'_h(s) = {x}_h(s) + \frac{\epsilon_h}{\lambda_s}x^{a}$. We choose $x^{a}$ such that $[P_h(s, a) - P_h(s,a')] \cdot x^{a} \geq \lambda_s, \forall a'\neq a$. The existence of $x^{a}$ is implied by Assumption~\ref{assum:inducibility-strong}. ${x}'_h(s)$ satisfies the constraints of LP \eqref{lp:pessimistic-contract-inductive}, i.e., $\forall a' \neq a,$
\begin{align*}
& \quad \hat{\mu}_h(s, a, a') \cdot  [ {x}'_h(s) +  \hat{U}_{h+1}^{\bpi}] \\
& \geq {\mu}_h(s, a, a') \cdot  {x}'_h(s)  -  \beta \norm{x'_h(s) }_{\infty} - \beta\norm{{U}_h}_{\infty} - \alpha\norm{\hat{\mu}_h(s, a, a')}_{1} + \beta \alpha \\
& \geq {\mu}_h(s, a, a') \cdot  {x}_h(s) + {\mu}_h(s, a, a') \cdot \frac{\beta}{\lambda_s} x^{a} - \beta\eta - \beta(H-h) - 2\alpha  \\
& \geq  c(s, a) - c(s, a')  + \epsilon_h  - \beta\eta - \beta(H-h) \\
& \geq  c(s, a) - c(s, a').   
\end{align*}
Moreover, since ${x}_h(s)$ minimizes LP \eqref{lp:pessimistic-contract-inductive}, we have $ \hat{P}(a)\cdot \hat{x}_h(s) \leq  \hat{P}(a)\cdot {x}'_h(s) = \hat{P}(a)\cdot [{x}_h(s) + \frac{\epsilon_h}{\lambda_s}x^{a}].$
Given that $\norm{x^{a}}_{\infty}=1$ and $\norm{P(a)}_{1}=1$, we have $\hat{P}_h(s, a)x^{a} \leq 1$.
Using $\hat{x}_h(s)$, the principal's expected payment is bounded from the least payment as 
\begin{align*}
  {U}^{\hat{x}_h(s)}_h(s) - U_h^{\bpi}(s)  
  & =  P_h(s, a)\cdot [ \hat{x}_h(s) -  {x}_h(s) ] \\
  & =  \hat{P}_h(s, a)\cdot [ \hat{x}_h(s) -  {x}_h(s) ] + [P_h(s, a) - \hat{P}_h(s, a) ] \cdot [ \hat{x}_h(s) -  {x}_h(s) ]  \\
  & \leq \hat{P}_h(s, a)\frac{\epsilon_h}{\lambda_s}x^{a} + \norm{P_h(s, a) - \hat{P}_h(s, a) }_1 \norm{ \hat{x}_h(s) -  {x}_h(s) }_{\infty} \\
    & \leq \frac{\epsilon_h}{\lambda_s} + \epsilon'.
\end{align*}
In addition, $U_h^{\bpi}(s)\leq H, {U}^{\hat{x}_h(s)}_h(s) \geq 0$, this gives the bound  ${U}^{\hat{x}_h(s)}_h(s) - U_h^{\bpi}(s) \leq \min\{ H, \frac{\epsilon_h}{\lambda_s} + \epsilon'\} .$

Since $\abs{ \hat{U}^{\bpi}_h(s) - U_h^{\hat{x}_h(s)}(s)} = \abs{ [\hat{P}_h(s, a) - P_h(s,a)]\cdot \hat{x}_h(s) } \leq \epsilon'  $, we have 
$$  \abs{ \hat{\zeta}^{\bpi}_{h}(s) - {\zeta}^{\bpi}_{h}(s) } = \abs{ \hat{U}^{\bpi}_h(s) - U^{\bpi}_h(s) } \leq \abs{ \hat{U}^{\bpi}_h(s) - U_h^{\hat{x}_h(s)}(s)} + \abs{{U}^{\hat{x}_h(s)}_h(s) - U_h^{\bpi}(s)}  \leq  \min\{ H, \frac{\epsilon_h}{\lambda_s} + 2\epsilon' \}.$$

We now plug in the value of $\alpha, \beta$. For the base case in the $H$-th step, we have $\alpha=0, \beta=\frac{\epsilon}{H+2\eta}$. Then, $\epsilon_{H} =  \epsilon,  \abs{ \hat{\zeta}^{\bpi}_H(s) - \zeta^{\bpi}_H(s) } = \abs{ \hat{U}^{\bpi}_H(s) - U^{\bpi}_H(s) }  = O\left(\min\{ H,  \frac{\epsilon}{\lambda_s} +  \epsilon' \} \right)  $

For the inductive case in the $h$-th step, given that $\norm{ \hat{U}^{\bpi}_{h+1} - U^{\bpi}_{h+1} }_{\infty} = O(\min\{ H, \epsilon \lambda_s^{h-H} + \epsilon' \lambda_s^{h+1-H} \} )$,
we have $\alpha = O(\min\{ H, \epsilon \lambda_s^{h-H} + \epsilon' \lambda_s^{h+1-H}\}), \beta=\frac{\epsilon}{H-h+2\eta}$ and $\epsilon_h = \norm{ \hat{U}^{\bpi}_{h+1} - U^{\bpi}_{h+1} }_{\infty} + \epsilon = O(\min\{ H, \epsilon \lambda_s^{h-H} + \epsilon' \lambda_s^{h+1-H} \} )$. Then, we can plug in $\epsilon_h$ and bound the estimation error as,
${U}^{\hat{x}_h(s)}_h(s) - U_h^{\bpi}(s) \eqsim \abs{ \hat{\zeta}^{\bpi}_{h}(s) - {\zeta}^{\bpi}_{h}(s) } =  \abs{ \hat{U}^{\bpi}_h(s) - U^{\bpi}_h(s) } \leq O\left(\min\{ H, \epsilon \lambda_s^{h-1-H} + \epsilon' \lambda_s^{h-H} \} \right) . 
$
This gives the bound as claimed in the Lemma.
\end{proof}

\begin{lemma} \label{lm:cost-bound}
If $\mathcal{E}_{\text{model}}$ is true, 
under Assumption \ref{assum:mixing-mdp} and \ref{assum:inducibility-strong}, 
with $T_1 = O(HAS^5\kappa_0^{-1}\log( \eta T\lambda_s^{-1} )) $, at any episode $t \in [T]$, we have
$$
\abs{ \hat{\zeta}^{t,\bpi}- \zeta^{\bpi} } = O\left( \min\big\{ H, \eta \lambda_s^{1-H}\kappa^{-1/2} \sqrt{\ln(S A H T / \delta)/t } \big\}  \right).
$$
\end{lemma}
\begin{proof}
This proof is to find a good trade-off between the error bound of $P$ and $\mu$ according to the ratio given in Lemma \ref{lm:lp-solution-guarantee}. 
For this analysis, let $\sup_{s, a,a',h}\norm{ \hat{\mu}^t_h(s, a, a') - {\mu}_h(s, a, a') }_{1} \leq \varepsilon$ such that we can apply Lemma~\ref{lm:improved-distribution-est} and get
$$
\norm{\hat{P}_h^t(s, a)-P_h(s, a) }_{1} \leq  2 \sqrt{\frac{\ln (S A H T / \delta)}{ 2\kappa t } } +  \varepsilon,
$$ 
where we use the fact $N_h^t(s) \geq 2\kappa t$ under Assumption \ref{assum:mixing-mdp}, similar to the analysis in Lemma \ref{lm:mdp-least-payment-oracle}. 
By Lemma \ref{lm:pess-contract-policy}, the Equation \eqref{eq:least-payment-policy-est} can construct contracts to induce any action policy $\bpi$ such that 
$$
\abs{ \hat{\zeta}^{t,\bpi} - \zeta^{\bpi} } = O\left( \min\big\{ H, (H+\eta)\epsilon \lambda_s^{-H} + \eta \lambda_s^{1-H}\sqrt{\frac{\ln (S A H T / \delta)}{ 2\kappa t } } \big\} \right)
$$ 

Then, we can set $\varepsilon= \frac{\lambda_s}{(H+\eta)\sqrt{t} } $ such that the second term dominates the regret bound, and we have
\begin{align*}
\abs{ \hat{\zeta}^{t,\bpi} - \zeta^{\bpi} } 
&= O\left( \min\{ H,  \eta \lambda_s^{1-H}\sqrt{\frac{\ln (S A H T / \delta)}{ \kappa t } }  \right) \\
&= O\left( \min\big\{ H, \eta \lambda_s^{1-H}\kappa^{-1/2} \sqrt{\ln(S A H T / \delta)/t } \big\} \right)
\end{align*}
It remains to ensure that in $T_1$ rounds, we can ensure  $\varepsilon= \frac{\lambda_s}{(H+\eta)\sqrt{t} } $.
By Lemma \ref{lm:mdp-least-payment-oracle}, we know the sample complexity of $\varepsilon$ is on the order of $\tilde{O}( HAS^5\kappa_0^{-1}\log( t ) )$. Hence, we can set $T_1 = O(HAS^5\kappa_0^{-1}\log( (H+\eta)T\lambda_s^{-1} )) = O(HAS^5\kappa_0^{-1}\log( \eta T\lambda_s^{-1} ))$.
\end{proof}